\documentclass{article} 
\usepackage{iclr2023_conference,times}

\usepackage{amsmath,amsfonts,bm}

\def\eqref#1{equation~\ref{#1}}

\def\1{\bm{1}}

\def\eps{{\epsilon}}

\DeclareMathAlphabet{\mathsfit}{\encodingdefault}{\sfdefault}{m}{sl}
\SetMathAlphabet{\mathsfit}{bold}{\encodingdefault}{\sfdefault}{bx}{n}
\newcommand{\tens}[1]{\bm{\mathsfit{#1}}}

\newcommand{\E}{\mathbb{E}}

\newcommand{\lr}{\alpha}

\DeclareMathOperator{\sign}{sign}

\usepackage{color, xcolor}
\definecolor{mydarkblue}{rgb}{0,0.08,0.45}
\usepackage[colorlinks=true,
    linkcolor=mydarkblue,
    citecolor=mydarkblue,
    filecolor=mydarkblue,
    urlcolor=mydarkblue,
    pdfview=FitH]{hyperref}
\usepackage{hyperref}
\usepackage{url}

\usepackage{amsmath}
\usepackage{amssymb}
\usepackage{mathtools}
\usepackage{amsthm}

\theoremstyle{plain}
\newtheorem{theorem}{Theorem}[section]

\newtheorem{lemma}[theorem]{Lemma}

\theoremstyle{definition}

\theoremstyle{remark}

\usepackage{mathmacros}

\newcommand{\x}{\m{x}}
\newcommand{\y}{\m{y}}
\newcommand{\z}{\m{z}}
\newcommand{\zz}{z}

\newcommand{\Q}{\tens{Q}}
\newcommand{\Qm}{\m{Q}}
\newcommand{\J}{\m{J}}
\newcommand{\G}{\m{G}}
\newcommand{\al}{\alpha}
\newcommand{\bt}{\beta}
\newcommand{\dl}{\delta}
\newcommand{\gm}{\gamma}
\renewcommand{\th}{\sm{\theta}}
\newcommand{\Lo}{\mathcal{L}}
\renewcommand{\P}{P}
\newcommand{\D}{D}
\renewcommand{\lr}{\eta}

\renewcommand{\v}{\m{v}}

\newcommand{\U}{\m{U}}
\newcommand{\X}{\m{X}}
\newcommand{\Y}{\m{Y}}
\newcommand{\N}{N}
\newcommand{\K}{K}

\newcommand{\W}{\m{W}}

\newcommand{\tzz}{\tilde{z}}
\newcommand{\tjj}{\tilde{J}}
\newcommand{\qlam}{\omega}
\newcommand{\tjjl}[1][i]{\tilde{J}(\qlam_{#1})}
\newcommand{\tl}{\tilde}
\newcommand{\w}{\m{w}}
\newcommand{\M}{M}
\newcommand{\jj}{J}
\newcommand{\TT}{T}

\renewcommand{\eps}{\epsilon}

\renewcommand{\E}{E}

\newcommand{\lam}{\lambda}
\renewcommand{\aa}{a}
\newcommand{\sgz}{\sigma_{z}}
\newcommand{\sgJ}{\sigma_{J}}
\newcommand{\tsgz}{\tilde{\sigma}_{z}}
\newcommand{\tsgJ}{\tilde{\sigma}_{J}}

\newif\ifcomments
\commentsfalse 

\ifcomments
\newcommand{\jp}[1]{{\color{blue}[JP: #1]}}
\newcommand{\fpd}[1]{{\color{purple}[FP: #1]}}
\newcommand{\aga}[1]{{\color{red}[AA: #1]}}
\else
\newcommand{\jp}[1]{}
\newcommand{\fpd}[1]{}
\newcommand{\aga}[1]{}
\fi

\title{Second-order regression models exhibit progressive sharpening to the edge of stability\aga{Comments on}}

\author{Atish Agarwala, Fabian Pedregosa \& Jeffrey 
Pennington  \\
Google Research, Brain Team\\
\texttt{\{thetish, pedregosa,jpennin\}@google.com}
}

\iclrfinalcopy 
\begin{document}

\maketitle

\begin{abstract}
Recent studies of gradient descent with large step sizes have shown that there is often a regime with an initial increase in the largest eigenvalue of the loss Hessian (progressive sharpening), followed by a stabilization of the eigenvalue near the maximum value which allows convergence (edge of stability). These phenomena are intrinsically non-linear and do not happen for models in the constant Neural Tangent Kernel (NTK) regime, for which the predictive function is approximately linear in the parameters. As such, we consider the next simplest class of predictive models, namely those that are quadratic in the parameters, which we call second-order regression models. For quadratic objectives in two dimensions, we prove that this second-order regression model exhibits progressive sharpening of the NTK eigenvalue towards a value that differs slightly from the edge of stability, which we explicitly compute. In higher dimensions, the model generically shows similar behavior, even without the specific structure of a neural network, suggesting that progressive sharpening and edge-of-stability behavior aren't unique features of neural networks, and could be a more general property of discrete learning algorithms in high-dimensional non-linear models.
\end{abstract}

\section{Introduction}

A recent trend in the theoretical understanding of deep learning has focused on the
\emph{linearized} regime, where the Neural Tangent Kernel (NTK) controls the learning
dynamics \citep{jacot_neural_2018, lee_wide_2019}.
The NTK describes learning dynamics of all networks over short enough time horizons, and can 
describe the dynamics of wide networks over large time horizons.
In the NTK regime, there is a function-space ODE which allows for explicit characterization of the 
network outputs~\citep{jacot_neural_2018, lee_wide_2019, yang_tensor_2021}. This approach has been 
used across the board to gain insights into wide neural networks,
but it suffers a major limitation: the model is linear in the parameters, so it describes a regime
with relatively trivial dynamics that cannot capture feature learning and cannot accurately
represent the types of complex training phenomena often observed in practice.

While other large-width scaling regimes can preserve some non-linearity and allow for certain types
of feature learning~\citep{bordelon_selfconsistent_2022, yang_tensor_2022}, such approaches tend to
focus on the small learning-rate or continuous-time dynamics. In contrast, recent empirical work has
highlighted a number of important phenomena arising from the non-linear discrete dynamics in
training practical networks with large learning
rates~\citep{neyshabur_exploring_2017, gilmer_loss_2022, ghorbani_investigation_2019, foret_sharpnessaware_2022}.
In particular, many experiments have shown the tendency for networks to display
\emph{progressive sharpening} of the curvature
towards the \emph{edge of stability}, in which the maximum 
eigenvalue of the loss Hessian increases over the course of training until it stabilizes at a value 
equal to roughly two divided by the learning rate, corresponding to the largest eigenvalue for
which gradient descent would converge in a quadratic potential~\citep{wu_how_2018, giladi_stability_2020,cohen_adaptive_2022, cohen_gradient_2022}.

In order to build a better understanding of this behavior, we introduce a class of models which 
display all the relevant phenomenology, yet are simple enough
to admit numerical and analytic understanding. In particular, we propose a simple 
\emph{quadratic regression model} and corresponding quartic loss function which fulfills both these
goals. We prove that under the right conditions, this simple model shows both progressive
sharpening \emph{and} edge-of-stability behavior. We then empirically analyze a more general
model which shows these behaviors \emph{generically} in the large datapoint, large model limit.
Finally, we conduct a numerical analysis on the properties of a real neural network and use
tools from our theoretical analysis to show that edge-of-stability behavior ``in the wild" shows
some of the same patterns as the theoretical models.

\section{Basic quartic loss function}


\subsection{Model definition}

We consider the optimization of the quadratic loss function $\Lo(\th) = z^{2}/2$, where $z$ a 
quadratic function on the $\P\times 1$-dimensional
parameter vector $\th$ and $\Qm$ is a $\P\times\P$ symmetric 
matrix:
\begin{equation}
\zz = \frac{1}{2}\left[\th^{\top}\Qm\th-\E\right]\,.
\label{eq:quad_model_one_data}
\end{equation}
This can be interpreted either as a model in which the predictive function is quadratic in the input
parameters, or as a second-order approximation to a more complicated non-linear function such as a 
deep network.
In this objective, the gradient flow (GF) dynamics with scaling factor $\lr$ is given by
\begin{equation}
\dot{\th} = -\lr\nabla_{\th}\Lo = \lr\zz\frac{\partial \zz}{\partial \th} = \frac{\lr}{2}\left[\th^{\top}\Qm\th-\E\right] \Qm\th\,.
\end{equation}
It is useful to re-write the dynamics in terms of $\tzz$ and the $1\times\P$-dimensional Jacobian
$\J= \partial \zz/\partial\th$:
\begin{equation}
\quad ~\dot{\zz} = -\lr(\J\J^\top)\zz,\quad ~\dot{\J} = -2\lr\zz \Qm\J\,.
\end{equation}
We note that in this case the neural tangent kernel (NTK) is a scalar given by the scalar
$\J \J^\top$. In these coordinates, we have $\E = \J\Qm^{+}\J^{\top}-2\zz$,
where $\Qm^{+}$ denotes the Moore-Penrose pseudoinverse.

The GF equations can be simplified by two transformations. 
First, we transform to
$\tzz = \lr\zz$ and $\m{\tjj} = \lr^{1/2}\J$. Next, we rotate $\th$ so that
$\Qm$ is diagonal. This is always possible since $\Qm$ is symmetric.
Since the NTK is given by $\J\J^\top$, this rotation preserves
the dynamics of the curvature. Let $\qlam_1 \geq \ldots \geq \qlam_P$ be the
eigenvalues of $\Qm$, and $\v_{i}$ be the associated eigenvectors (in case of degeneracy,
one can pick any basis). We define $\tjjl = \tl{\J}\v_{i}$, the projection of $\tl{\J}$
onto the $i$th eigenvector. Then the gradient flow
equations can be written as:
\begin{equation}
\frac{d\tzz}{dt} = - \tzz\sum_{i=1}^{\P}\tjjl^{2},\quad ~\frac{d\tjjl^2}{dt} = -2\tzz\qlam_{i}\tjjl^2 \,.
\end{equation}
The first equation implies that $\tzz$ does not change sign under GF dynamics.
Modes with positive $\qlam_{i}\tzz$
decrease the curvature, and those with negative $\qlam_{i}\tzz$ increase the curvature.

In order to study edge-of-stability behavior, we need initializations which allow the
curvature ($\J\J^{\top}$ in this case) to increase over time - a phenomenon known as
\emph{progressive sharpening}. Progressive sharpening has been shown to be ubiquitous in machine learning
models \citep{cohen_gradient_2022}, so any useful phenomenological model should show it
as well. One such initialization for this quadratic regression model is
$\qlam_{1} = -\omega$, $\qlam_{2} = \omega$, $\tjjl[1] = \tjjl[2]$.
This initialization (and others) show progressive
sharpening at all times.

\subsection{Gradient descent}

We are interested in understanding the \emph{edge-of-stability}
(EOS) behavior in this model: gradient descent (GD) trajectories where the maximum
eigenvalue of the NTK, $\J\J^{\top}$, remains close to the critical value $2/\lr$.
(Note: we define edge of stability with respect to the maximum NTK eigenvalue; for
any twice-differentiable model trained with squared loss,
this is equivalent to the maximum eigenvalue of the loss Hessian used in 
\cite{cohen_gradient_2022} as the model converges to a stationary point \citep{jacot_asymptotic_2020}.)

When $\Qm$ has both positive and negative eigenvalues, the loss landscape is the square of a
hyperbolic parabaloid (Figure \ref{fig:two_param_landscape}, left). As suggested by the
gradient flow analysis, this causes some trajectories to increase their curvature before
convergence. This causes the final curvature to depend on both the initialization and
learning rate. One of the challenges in analyzing
the gradient descent (GD) dynamics is that they rapidly and heavily oscillate around minima
for large learning rates. One way 
to mitigate this issue is to consider only every other step
(Figure \ref{fig:two_param_landscape}, right).
We will use this observation to analyze the gradient descent (GD) dynamics directly to find
configurations where these trajectories show edge-of-stability behavior.

\begin{figure}
    \centering
    \begin{tabular}{cc}
    \includegraphics[height=0.3\linewidth]{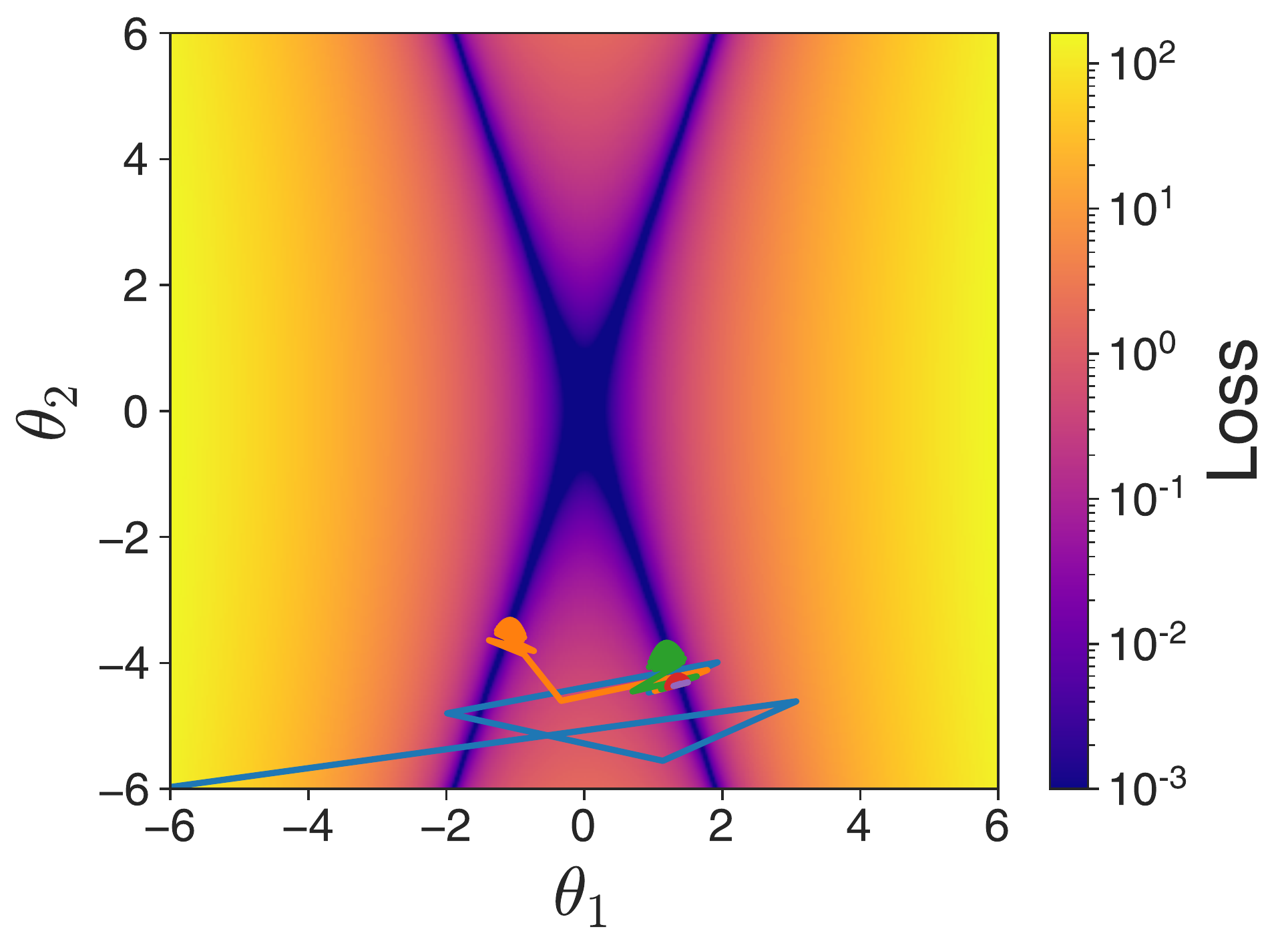} &
    \includegraphics[height=0.3\linewidth]{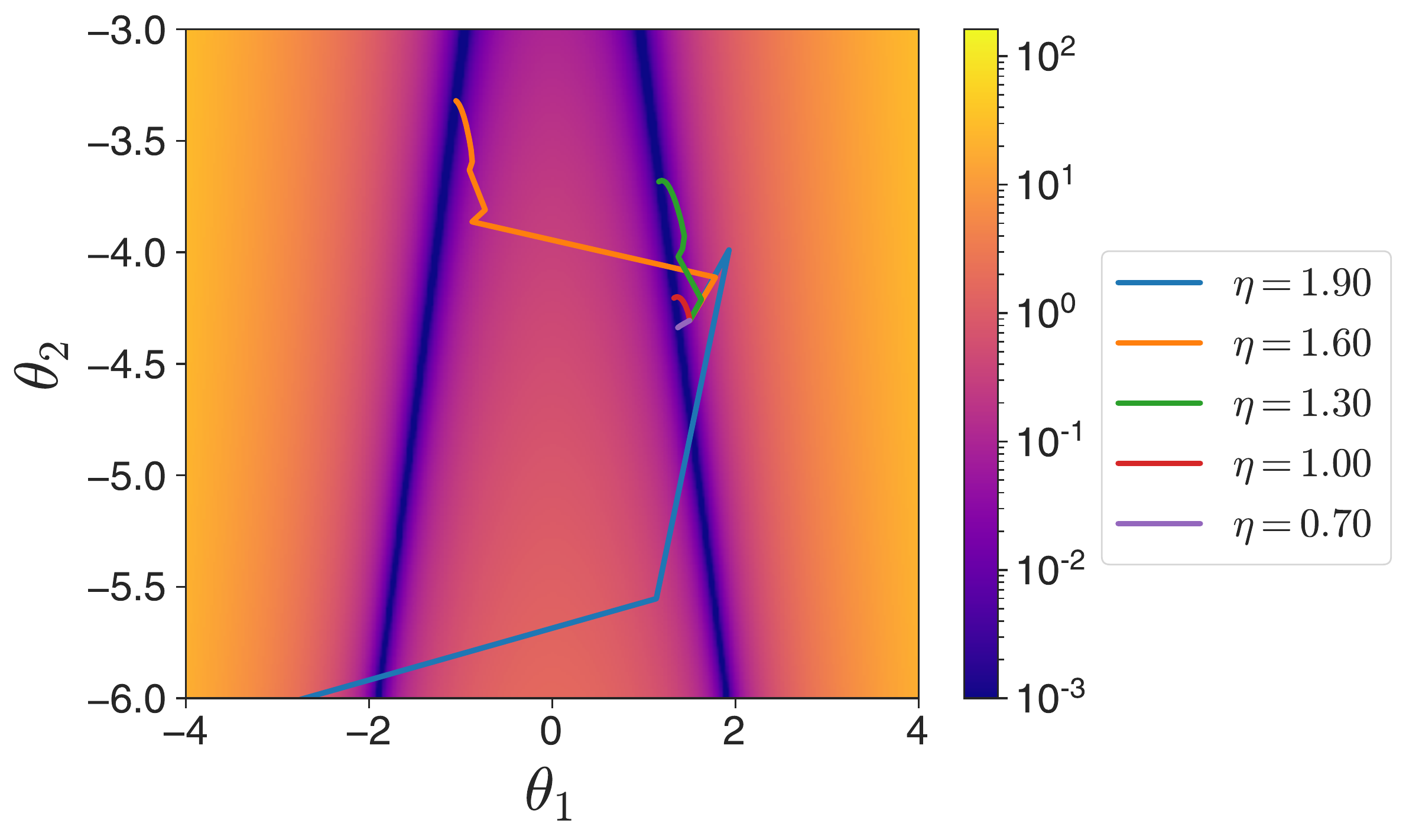}
    \end{tabular}
    \caption{Quartic loss landscape $\mathcal{L}(\cdot)$ as a function of the parameters $\th$, where $D=2, \E = 0$ and $\Qm$ has eigenvalues $1$ and $-0.1$.
    The GD trajectories converge to minima with larger curvature than at initialization and therefore show progressive
    sharpening (left). The two-step dynamics, in which we consider only even iteration numbers, exhibit fewer oscillations near the edge
    of stability (right).}
    \label{fig:two_param_landscape}
\end{figure}

In the eigenbasis coordinates, the gradient descent equations are
\begin{align}\label{eq:zz_dyn}
\tzz_{t+1}-\tzz_{t} &= - \tzz_{t}\sum_{i=1}^{\P}\tjjl_{t}^{2} + \frac{1}{2}(\tzz_{t}^2)\sum_{i=1}^{\P}\qlam_{i}\tjjl_{t}^{2} \\
\tjjl^2_{t+1}-\tjjl^2_{t} &= -\tzz_{t}\qlam_{i}(2-\tzz_{t}\qlam_{i})\tjjl_{t}^2\, \text{ for all $1 \leq i \leq P$}.
\end{align}
We'll find it convenient in the following to write the dynamics in terms of weighted averages of
$\tjjl^2$ instead of the modes $\tjjl$:
\begin{equation}
\TT(\al) = \sum_{i=1}^{\P}\qlam_{i}^{\al}\tjjl^2\,.
\end{equation}
The dynamical equations become:
\begin{align}\label{eq:Tzmodel}
\tzz_{t+1}-\tzz_{t} &= - \tzz_{t}\TT_{t}(0) + \frac{1}{2}(\tzz_{t}^2)\TT_{t}(1) \\
\TT_{t+1}(k)-\TT_{t}(k) &= -\tzz_{t}(2\TT_{t}(k+1)-\tzz_{t}\TT_{t}(k+2))\,. \label{eq:Tzmodel2}
\end{align}
If $\Qm$ is invertible, then we have $\E = \TT_{t}(-1)-2\tzz_{t}$.
Note that by definition $\TT_t(0) = \lr\J_t\J_t^{\top}$ is the (rescaled) NTK. edge-of-stability behavior
corresponds to dynamics
which keep
$\TT_{t}(0)$ near the value $2$ as $\tzz_{t}$ goes to $0$. 

\subsubsection{Reduction to catapult dynamics}

If the eigenvalues of $\Qm$ are $\{-\qlam, \qlam\}$, and $\E = 0$, the model becomes equivalent
to a single hidden layer linear network with one training datapoint (Appendix
\ref{app:one_hidden_layer}) - also known as the catapult phase dynamics. This
model doesn't exhibit sharpening or edge-of-stability behavior 
\citep{lewkowycz_large_2020}.
We will analyze this model in our $\tzz-\TT(0)$ variables as a warmup, with an eye towards analyzing
a different parameter setting which does show sharpening and edge of stability.

We assume without loss of generality that the eigenvalues are $\{-1, 1\}$ - which can be
accomplished by rescaling $\tzz$. The loss function is then the square of a hyperbolic
parabaloid. Since there are only $2$ variables, we can rewrite the dynamics in terms
of $\tzz$ and the curvature $\TT(0)$ only (Appendix \ref{app:two_param_basics}):
\begin{equation}
\tzz_{t+1}-\tzz_{t} = - \tzz_{t}\TT_{t}(0) + \frac{1}{2}(\tzz_{t}^2)(2\tzz_{t}+\E)
\label{eq:z_sym_red}
\end{equation}
\begin{equation}
\TT_{t+1}(0)-\TT_{t}(0) = -2\tzz_{t}(2\tzz_{t}+\E)+\zz_{t}^{2}\TT_{t}(0)\,.
\label{eq:T_sym_red}
\end{equation}
For $\E = 0$, we can see that $\sign(\Delta\TT(0)) = \sign(\TT_{t}(0)-4)$, as in
\cite{lewkowycz_large_2020} - so convergence requires strictly decreasing curvature.
For $\E\neq 0$, there is a region where the curvature can increase (Appendix \ref{app:two_param_basics}). However, there is still no edge-of-stability behavior -
there is no set of initializations which starts with $\lam_{\max}$
far from $2/\lr$, which ends up near $2/\lr$.
In contrast, we will show that asymmetric eigenvalues
can lead to EOS behavior.

\subsubsection{Edge of stability regime}

In this section, we consider the case in which $\Qm$ has two eigenvalues -
one of which is large and positive, and the other one small and negative. Without loss of 
generality, we assume that the largest eigenvalue of $\Qm$ is $1$. We denote the
second eigenvalue by
$-\eps$, for $0<\eps\leq 1$. With this notation we can write the dynamical equations
(Appendix \ref{app:two_param_basics}) as
\begin{equation}
\tzz_{t+1}-\tzz_{t} = - \tzz_{t}\TT_{t}(0) + \frac{1}{2}(\tzz_{t}^2)((1-\eps)\TT_{t}(0)+\eps(2\tzz_{t}+\E))
\label{eq:z_2_red_eps}
\end{equation}
\begin{equation}
\TT_{t+1}(0)-\TT_{t}(0) = -2\tzz_{t}(\eps(2\tzz_{t}+\E)+(1-\eps)\TT_{t}(0))+\tzz_{t}^{2}\left[\TT_{t}(0)+\eps\left(\eps-1\right)(\TT_{t}(0)-\E-2\tzz_{t})\right]\,.
\label{eq:T_2_red_eps}
\end{equation}
For small $\eps$, there are trajectories where $\lam_{\max}$ is initially away from $2/\lr$ but converges towards it (Figure \ref{fig:two_param_asym_two_step}, left) - in other
words, EOS behavior. We used a variety of step sizes $\lr$ but initialized at pairs
initialized at pairs $(\lr z_{0}, \lr T_{0}(0))$ to show
the universality of the $\tzz$-$\TT(0)$ coordinates.

In order to quantitatively understand the progressive sharpening and edge of stability,
it is useful to look at the two-step dynamics.
One additional motivation for studying the two-step dynamics follows from the analysis of gradient
descent on linear least squares (i.e., linear model) with a large step size $\lambda$. For every 
coordinate $\tilde{\theta}$,
the one-step and two-step dynamics are
\begin{equation}
\tilde{\theta}_{t+1}-\tilde{\theta}_{t} = -\lam \tilde{\theta}_{t}~\text{ and }~\tilde{\theta}_{t+2}-\tilde{\theta}_{t} = (1-\lam)^{2}\tilde{\theta}_{t} \qquad \text{(GD in quadratic potential)}\,.
\end{equation}
While the dynamics converge for $\lam <2$, if $\lam>1$ the one-step dynamics oscillate when
approaching minimum, whereas the the two-step dynamics maintain the sign of $\tilde{\theta}$ and the
trajectories exhibit no oscillations.

Likewise, plotting every other iterate in the two parameter model
more clearly demonstrates the phenomenology.
For small $\eps$, the dynamics shows the distinct phases described in 
\citep{li_analyzing_2022}: an initial increase
in $\TT(0)$, a slow increase in $\tzz$, then a decrease in $\TT(0)$, and finally a slow decrease
of $\tzz$ while $\TT(0)$ remains near $2$ (Figure \ref{fig:two_param_asym_two_step}, middle).

Unfortunately, the two-step version of the dynamics defined by
Equations \ref{eq:z_2_red_eps} and \ref{eq:T_2_red_eps}
are more complicated -- they are $3$rd order in $\TT(0)$ and $9$th order in 
$\tzz$; see
Appendix \ref{app:two_step} for a more detailed discussion. However we can still
analyze the dynamics as $\tzz$ goes to $0$. In order to understand the mechanisms
of the EOS behavior, it is useful to understand the \emph{nullclines} of the two step dynamics.
The nullcline $f_{\tzz}(\tzz)$ of $\tzz$ and $f_{\TT}(\tzz)$ of $\TT(0)$ are defined implicitly
by
\begin{equation}
(\tzz_{t+2}-\tzz_{t})(\tzz, f_{\tzz}(\tzz)) = 0,~(\TT_{t+2}(0)-\TT_{t}(0))(\tzz, f_{\TT}(\tzz)) = 0
\end{equation}
where $\tzz_{t+2}-\tzz_{t}$ and $\TT_{t+2}(0)-\TT_{t}(0)$ are the aforementioned high order
polynomials in $\tzz$ and $\TT(0)$. Since these polynomials are cubic in $\TT(0)$, there
are three possible solutions as $\tzz$ goes to $0$. We are particularly interested
in the solution that goes through $\tzz = 0$, $\TT(0) = 2$ - that is, the critical point
corresponding to EOS.

Calculations detailed in Appendix \ref{app:two_step} show that the
distance between the two nullclines is linear in $\eps$, so they become close
as $\eps$ goes to $0$.
(Figure \ref{fig:two_param_asym_two_step}, middle).
In addition, the trajectories stay near $f_{\tzz}$ - which gives rise to EOS behavior.
This suggests that the dynamics are slow near the nullclines, and trajectories appear to be
approaching an attractor.
We can find the structure of the attractor by changing variables to
$y_{t} \equiv \TT_{t}(0)-f_{\tzz}(\tzz_{t})$ - the distance from the $\tzz$ nullcline.
To lowest order in $\tzz$ and $y$, the two-step dynamical equations become (Appendix
\ref{app:two_step_approx}):
\begin{align}
\label{eq:two_step_z_big_O}
\tzz_{t+2}-\tzz_{t} &= 2y_t\tzz_t+O(y_{t}^2\tzz_{t})+O(y_{t}\tzz_{t}^{2})\\
y_{t+2}-y_{t} &= -2(4-3\eps+4\eps^2)y_t \tzz_{t}^{2}-4\eps \tzz_{t}^{2}+\eps O(\tzz_{t}^{3})+O(y^2\tzz_{t}^2)
\label{eq:two_step_y_big_O}
\end{align}
We immediately see that $\tzz$ changes slowly for small $y$ - since we chose
coordinates where $\tzz_{t+2}-\tzz_{t} = 0$ when $y=0$. We can also see that
$y_{t+2}-y_{t}$ is $O(\eps)$ for $ y_{t} = 0$ - so for small $\eps$, the $y$ dynamics is
slow too. Moreover, we see that the coefficient of the $\eps\tzz_{t}^{2}$ term is negative -
the changes in $\tzz$ tend to drive $y$ (and therefore $\TT(0)$) to decrease.
The coefficient of the $y_{t}$ term is negative as well; the dynamics of $y$ tends to be
contractive.
The key is that the contractive behavior takes $y$ to an $O(\eps)$ fixed point
at a rate proportional to $\tzz^2$, while the dynamics of $\tzz$ are proportional
to $\eps$. This suggests a separation of timescales if $\tzz^2\gg\eps$, where
$y$ first equilibrates to a fixed value, and then $\tzz$ converges to $0$ (Figure \ref{fig:two_param_asym_two_step}, right).
This intuition for the lowest order terms can be formalized,
and gives us a prediction of $\lim_{t\to\infty}y_{t} = -\eps/2$, confirmed
numerically in the full model (Appendix \ref{app:low_order_dyn}).

\begin{figure}[h]
\centering
\begin{tabular}{ccc}
\includegraphics[height=0.25\linewidth]{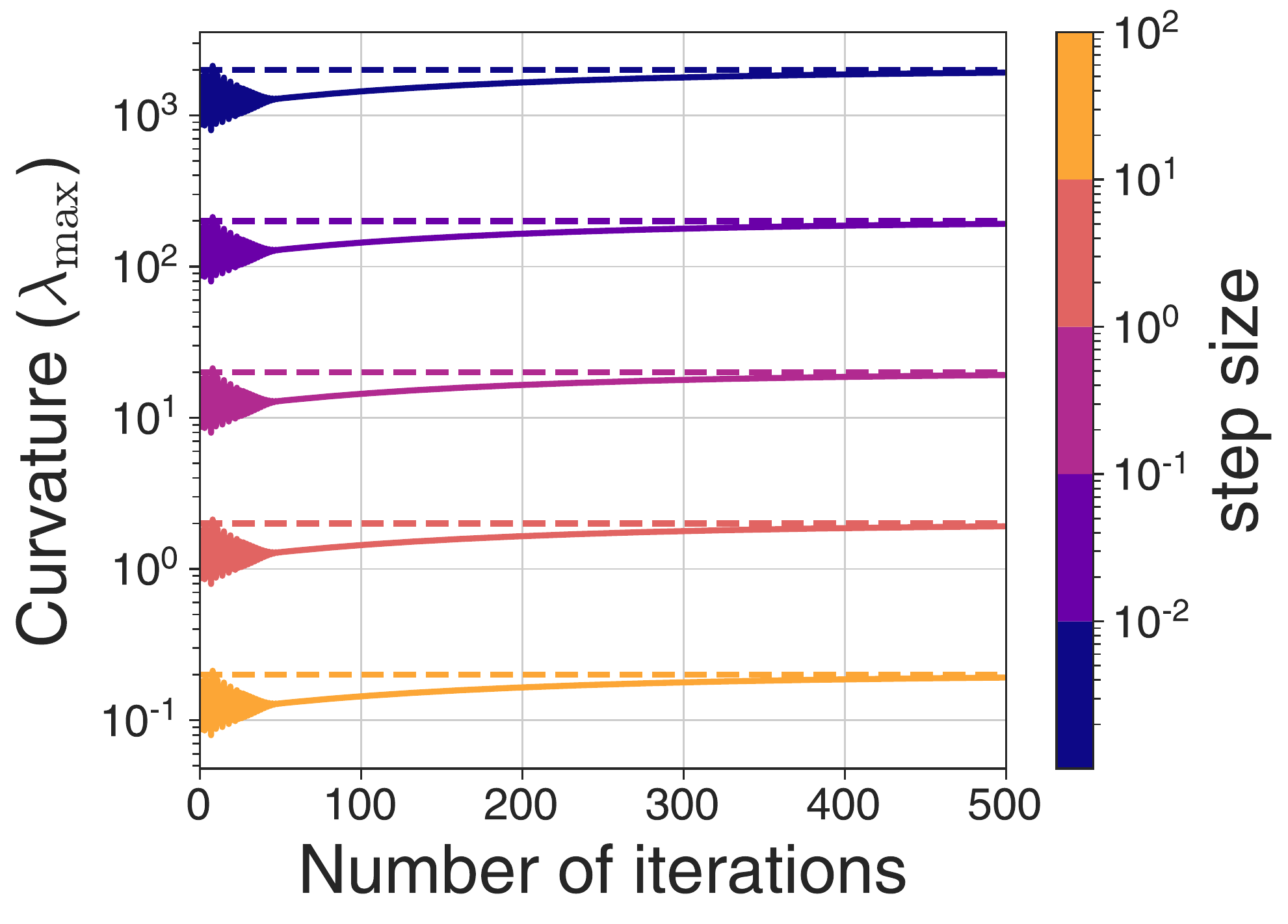} &
\includegraphics[height=0.25\linewidth]{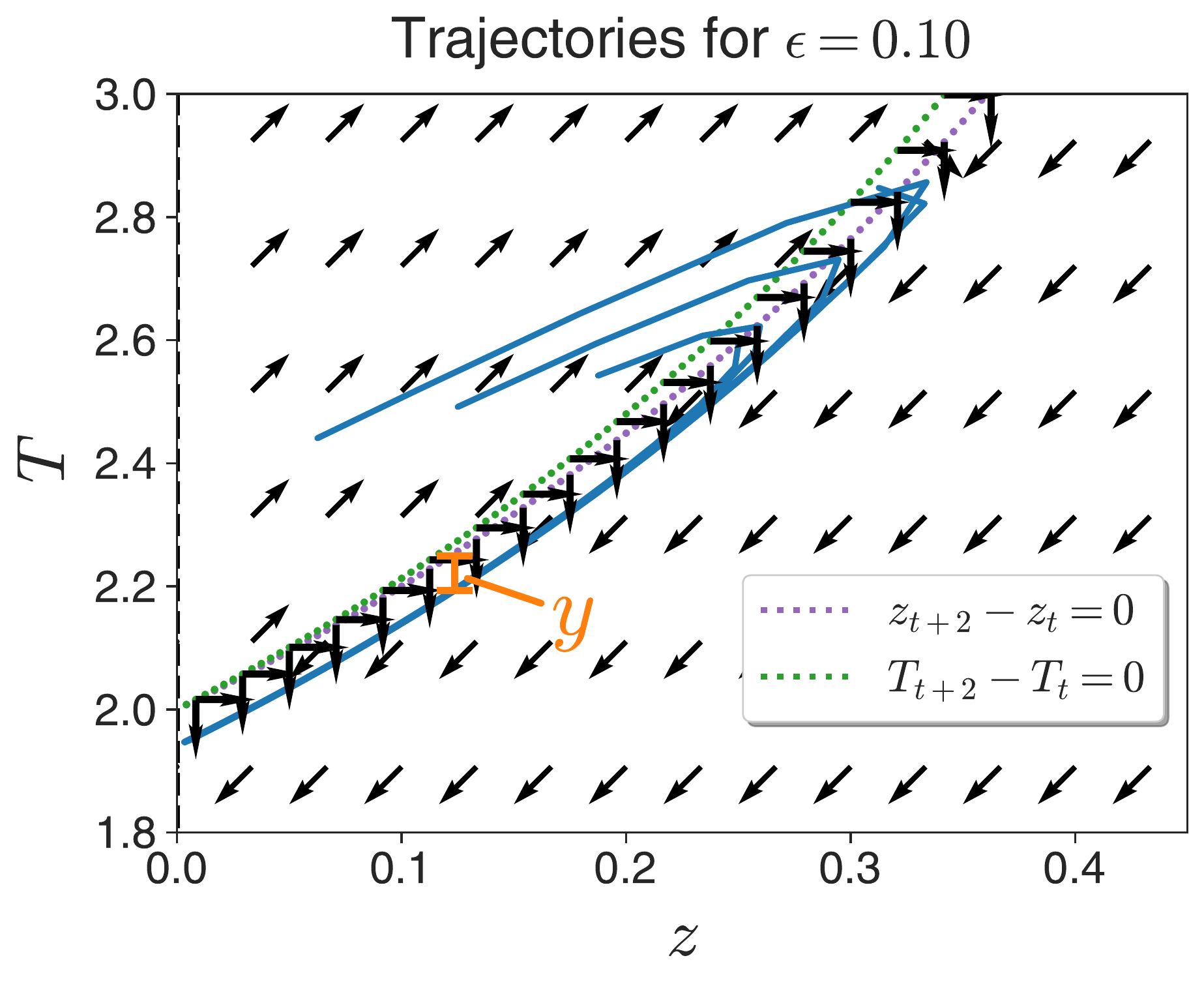}
\includegraphics[height=0.25\linewidth]{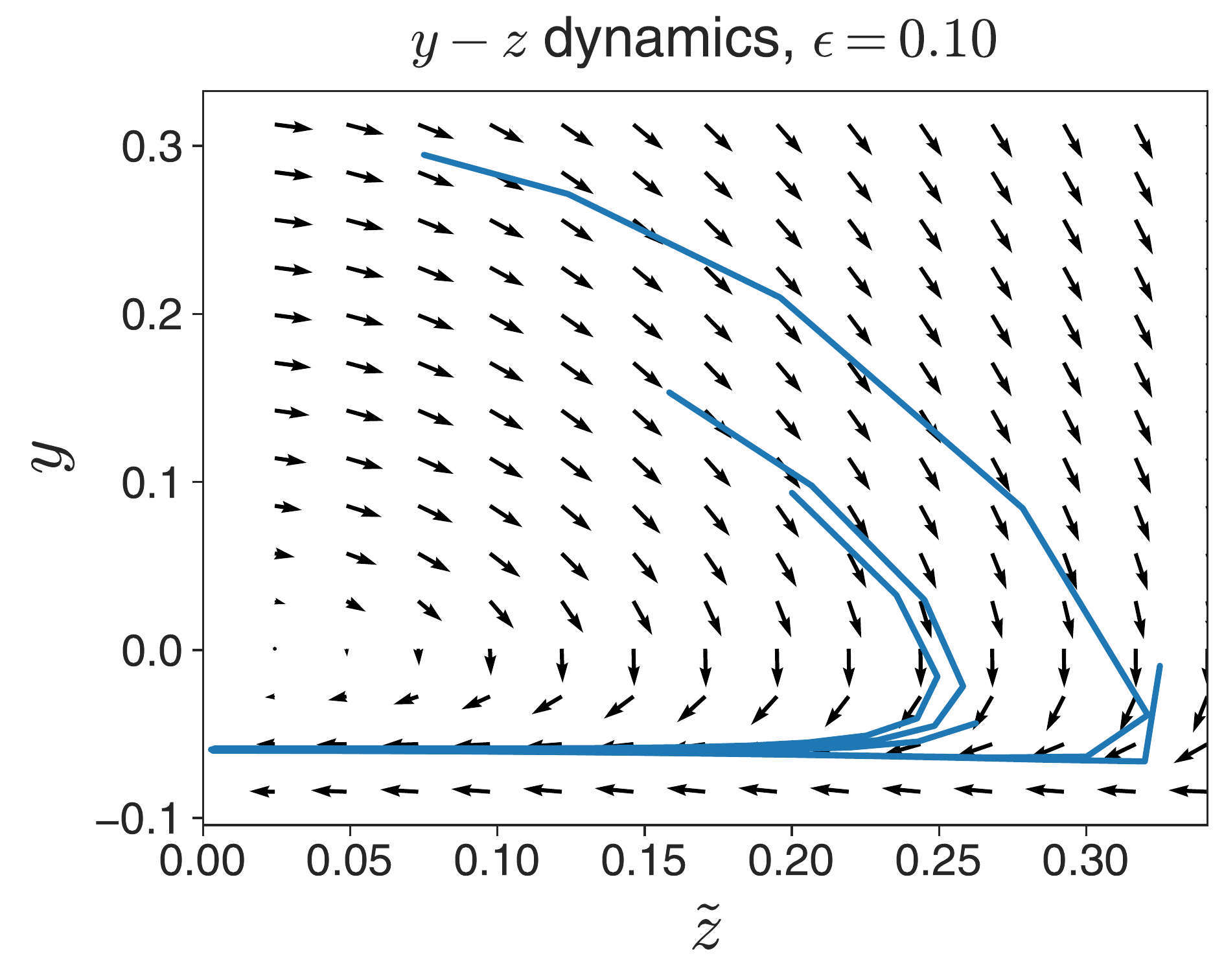}
\end{tabular}
\caption{For small $\eps$, two-eigenvalue model shows EOS behavior for various
step sizes ($\eps = 5\cdot10^{-3}$, left).
Trajectories are the same up to scaling because corresponding rescaled coordinates
$\tzz$ and $\TT(0)$ are the same at initialization.
Plotting every other iterate, we see that trajectories in
$\tzz-\TT(0)$ space stay near the
nullcline $(\tzz, f_{\tzz}(\tzz))$ - the
curve where $\tzz_{t+2}-\tzz_{t} = 0$ (middle). Changing variables to $y=\TT(0)-f_{\tzz}(\tzz)$
shows quick concentration to a curve of near-constant, small, negative $y$ (right).}
\label{fig:two_param_asym_two_step}
\end{figure}

We can prove the following theorem about the long-time dynamics of
$\tzz$ and $y$ when the higher order terms are included (Appendix \ref{app:eps_proof}):
\begin{theorem}
\label{thm:two_step_ode_approx}
There exists an $\eps_{c}>0$ such that for
a quadratic regression model with $\E = 0$ and eigenvalues $\{-\eps, 1\}$, $\eps\leq \eps_{c}$.
there exists a neighborhood $U\subset \mathbb{R}^{2}$ and interval
$[\lr_{1}, \lr_{2}]$ such that for initial $\th\in U$ and learning rate $\lr\in[\lr_{1}, \lr_{2}]$,
the model displays edge-of-stability behavior:
\begin{equation}
2/\lr-\dl_{\lam}\leq \lim_{t\to\infty}\lam_{\max}\leq 2/\lr
\end{equation}
for $\dl_{\lam}$ of $O(\eps)$.
\end{theorem}


Therefore, unlike the catapult phase model, the small $\eps$ provably has EOS behavior - whose
mechanism is well-understood by the $\tzz-y$ coordinate transformation.

\section{Quadratic regression model}

\label{sec:quad_reg_model}

\subsection{General model}

While the model defined in Equation \ref{eq:quad_model_one_data}
provable displays edge-of-stability behavior, it required tuning of
the eigenvalues of $\Qm$ to demonstrate it.
We can define a more general model which exhibits edge-of-stability behavior
with less tuning. We define the \emph{quadratic regression model} as follows.
Given a $\P$-dimensional parameter vector $\th$, the $\D$-dimensional
output vector $\z$ is given by
\begin{equation}
\z = \y+\G^\top\th+\frac{1}{2}\Q(\th, \th)
\end{equation}
Here $\y$ is a $\D$-dimensional vector,
$\G$ is a $\D\times\P$-dimensional matrix, and $\Q$ is a $\D\times\P\times\P$-
dimensional
tensor symmetric in the last two indices - that is, $\Q(\cdot, \cdot)$ takes two
$\P$-dimensional vectors as input, and outputs a $\D$-dimensional vector verifying $\Q(\th, \th)_\alpha = \th^\top \Q_{\alpha} \th$.
If $\Q = \boldsymbol{0}$, the model corresponds to linearized learning (as in the NTK regime).
When $\Q\neq \boldsymbol{0}$, we obtain the first correction to NTK regime. We note that:
\begin{equation}
\G_{\al i} = \left.\frac{\partial\z_{\al}}{\partial\th_{i}}\right|_{\th=0},~\Q_{\al ij} = \frac{\partial^{2}\z_{\al}}{\partial\th_{i}\partial\th_{j}},\to\J = \G+\Q(\th, \cdot)\,,
\end{equation}
for the $\D\times\P$ dimensional Jacobian $\J$. For $\D = 1$, we recover the model of Equation \ref{eq:quad_model_one_data}.
In the remainder of this section, we will study the limit as $\D$ and $\P$ 
increase with fixed ratio $\D/\P$.

The quadratic regression model corresponds to a model with a constant second 
derivative
with respect to parameter changes - or a second order expansion of a more complicated ML
model. Quadratic expansions of
shallow MLPs have been previously studied \citep{bai_linearization_2020, zhu_quadratic_2022},
but we will provide evidence that even random, unstructured quadratic regression models lead to
EOS behavior. We note that this model is related to, but not equivalent to
the second order expansion in the neural tangent hierarchy
\citep{huang_dynamics_2020} (see Appendix \ref{app:nth} for details).

\subsection{Gradient flow dynamics}

We will focus on training with squared loss
$\Lo(\z) = \frac{1}{2}\sum_{\al}\z_{\al}^{2}$.
We begin by considering the dynamics under gradient flow (GF):
\begin{equation}
\dot{\th} = -\frac{\partial \Lo(\z)}{\partial\th} = -\J^{\top}\z\,.
\end{equation}
We can write the dynamics in the output space $\z$ and the Jacobian $\J$ as
\begin{equation}
\dot{\z} = \J\dot{\th} = -\J\J^{\top}\z,~\dot{\J} = - \Q(\J^{\top}\z,\cdot)
\end{equation}
When $\Q = \boldsymbol{0}$ (linearized/NTK regime), $\J$ is constant,
the dynamics are then linear in $\z$, and are controlled by the eigenstructure of $\J\J^{\top}$,
the empirical NTK. In this regime there is no EOS behavior.

We are interested in settings where progressive sharpening occurs under GF.
We can study the dynamics of the maximum eigenvalue $\lam_{\max}$ of $\J\J^{\top}$
at early times for random initializations.
In Appendix \ref{app:gf_dynamics}, we prove the following theorem:
\begin{theorem}
\label{thm:ave_curv_deriv}
Let $\z$, $\J$, and $\Q$ be initialized with i.i.d. elements with zero mean and
variances $\sgz^2$, $\sgJ^2$, and $1$ respectively, with distributions
invariant to rotation in data and parameter space, and have finite
fourth moments. Let $\lam_{\max}$ be the largest
eigenvalue of $\J\J^{\top}$. In the limit of large $\D$ and $\P$, with fixed ratio
$\D/\P$, at initialization we have
\begin{equation}
\expect[\dot{\lam}_{\max}(0)] = 0,~\expect[\ddot{\lam}_{\max}(0)]/\expect[\lam_{\max}(0)] = \sgz^{2}
\end{equation}
where $\expect$ denotes the expectation over $\z$, $\J$, and $\Q$ at initialization.
\end{theorem}
Much like in the $\D = 1$ case, Theorem \ref{thm:ave_curv_deriv} suggests
that it is easy to find initializations that show progressive sharpening - 
and increasing $\sgz$ makes sharpening more prominent.

\subsection{Gradient descent dynamics}

We now consider finite-step size gradient descent (GD) dynamics. The dynamics for $\th$ are
given by:
\begin{equation}
\th_{t+1} = \th_{t}-\lr\J^{\top}_{t}\z_{t}\,.
\end{equation}
In this setting, the dynamic equations can be written as
\begin{equation}
\z_{t+1}-\z_{t} = -\lr\J_{t} \J_{t}^{\top}\z_{t}  +\frac{1}{2}\lr^2 \Q(\J_{t}^{\top}\z_{t},\J_{t}^{\top}\z_{t})
\label{eq:GD_in_z_general}
\end{equation}
\begin{equation}
\J_{t+1} -\J_{t} = -\lr \Q(\J_{t}^{\top}\z_{t}, \cdot)\,.
\label{eq:GD_in_J_general}
\end{equation}
If $\Q = \boldsymbol{0}$, the dynamics reduce to discrete gradient
descent in a quadratic potential - which converges iff $\lam_{\max}<2/\lr$.

One immediate question is: when does the $\lr^2$ in Equation \ref{eq:GD_in_z_general} affect
the dynamics? Given that it scales with higher powers of $\lr$ and $\z$ than the first term,
we can conjecture that the ratio of the magnitudes of the terms, $r_{NL}$, is proportional
to $||\z||_{2}$ and $\lr$. A calculation in Appendix \ref{app:gd_timescales} shows that,
for the random rotationally invariant initialization, we have:
\begin{equation}
r_{NL} \equiv \left(\frac{\expect[|| \frac{1}{2}\lr^2 \Q(\J_{0}^{\top}\z_{0},\J_{0}^{\top}\z_{0})||_{2}^{2}]}{\expect[||\lr\J_{0} \J_{0}^{\top}\z_{0} ||_{2}^{2}]}\right)^{1/2} = \frac{1}{2}\lr\sgz\D\,,
\end{equation}
where as before the expectation is taken over the initialization of $\z$, $\J$, and $\Q$.
This suggests that increasing the learning rate increases the deviation of the dynamics
from GF (which is obvious), but increasing $||\z||$ \emph{also}
increases the deviation from GF.

We can see this phenomenology in the dynamics of the GD equations (Figure \ref{fig:gd_traj}).
Here we plot different trajectories for random initializations of the type in
Theorem \ref{thm:ave_curv_deriv} with $\D = 60$, $\P = 120$, and $\lr = 1$. As $\sgz$ increases,
so does the curvature $\lambda_{\max}$  (as suggested by Theorem \ref{thm:ave_curv_deriv}),
and when $\sgz$ is $O(1)$, the dynamics is non-linear (as predicted by $r_{NL}$) and
EOS behavior emerges. This suggests that the second term in
Equation \ref{eq:GD_in_z_general} is crucial for the stabilization of $\lam_{max}$.

\begin{figure}
    \centering
    \begin{tabular}{cc}
    \includegraphics[height=0.28\linewidth]{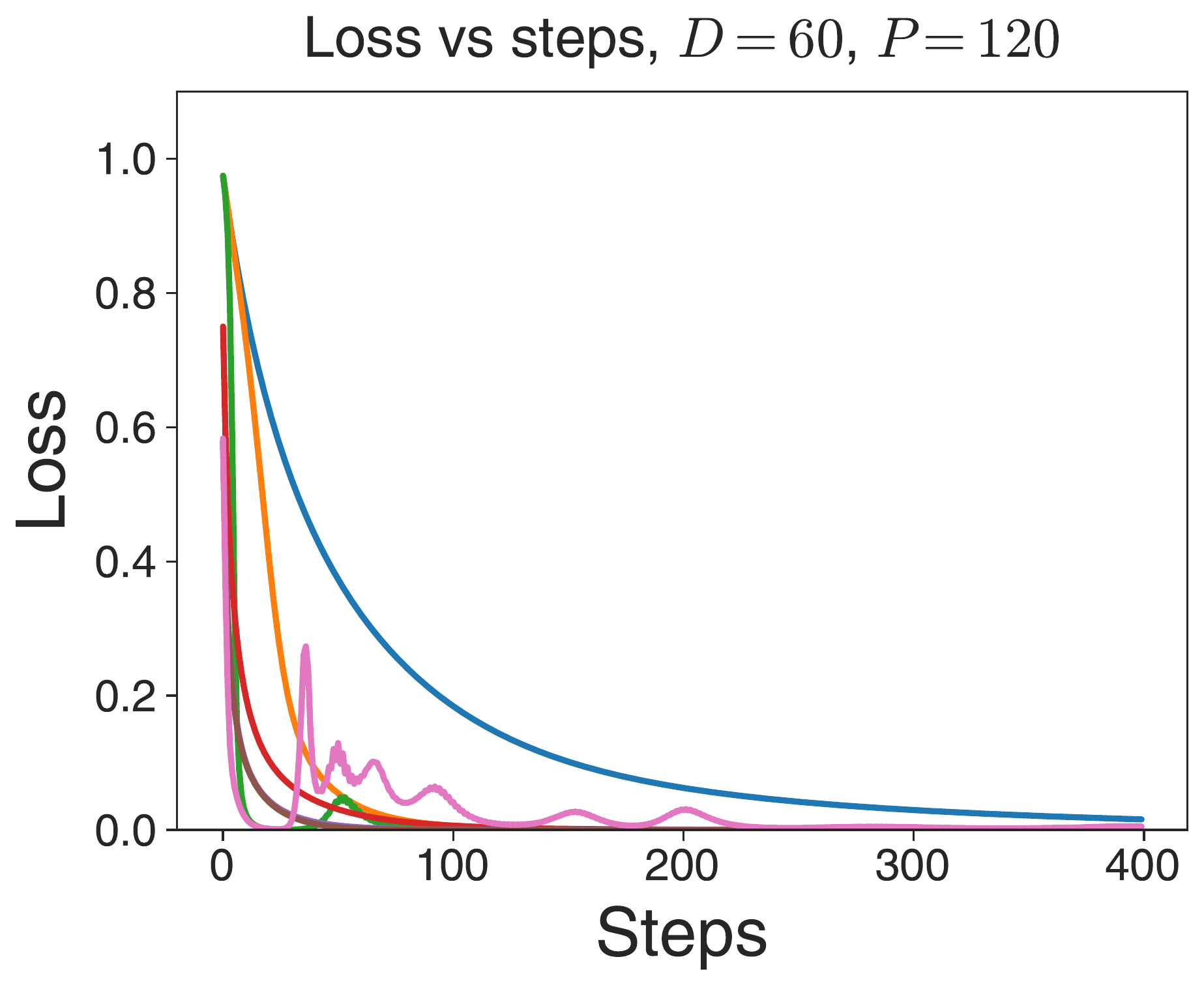} &
    \includegraphics[height=0.28\linewidth]{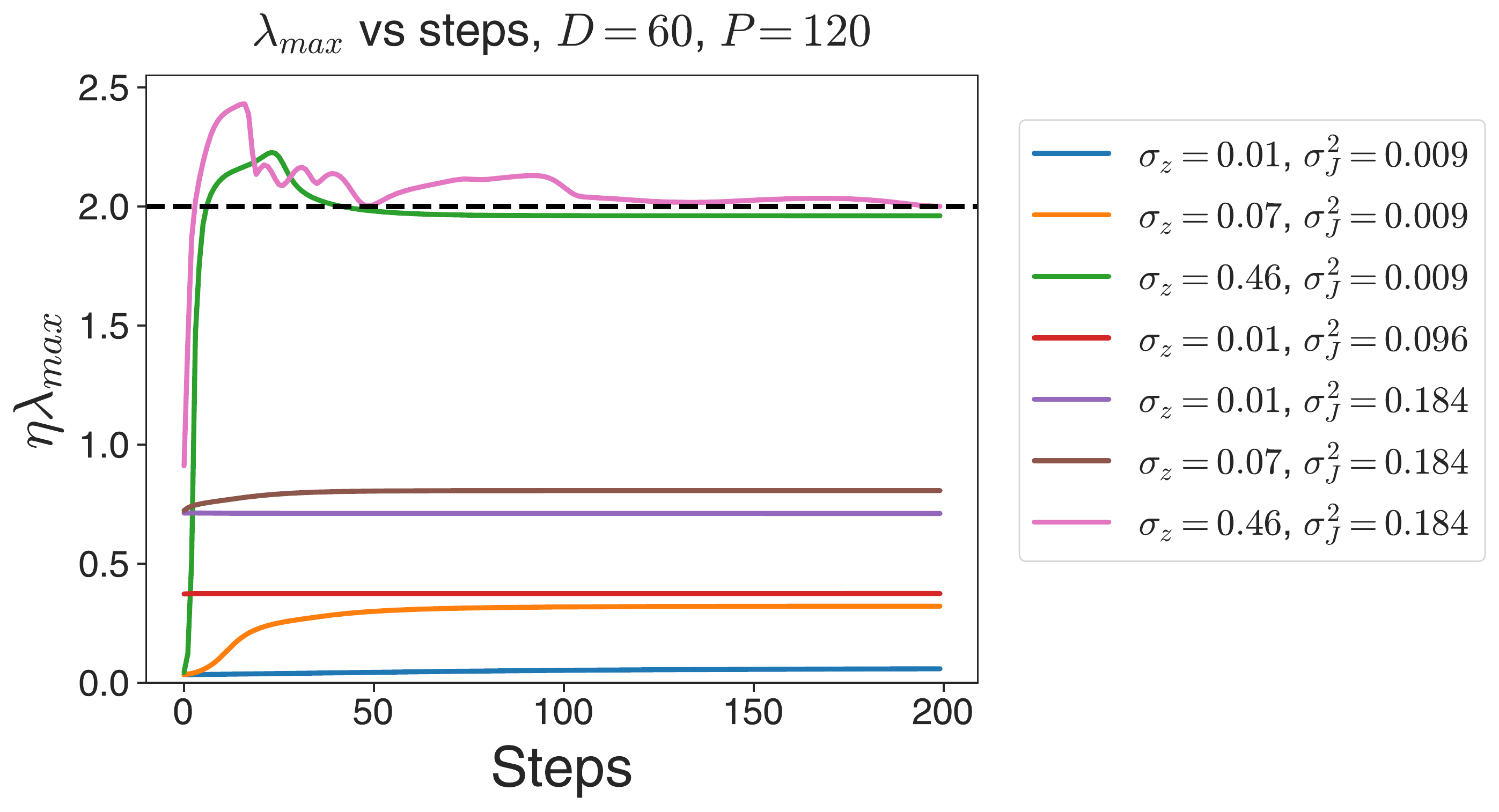}
    \end{tabular}
    
    \caption{Gradient descent dynamics in the quadratic regression model. As $\z$ initialization
    variance $\sgz^2$ increases, so does the curvature $\lambda_{\max}$ upon convergence. As sharpening drives $\lr\lam_{\max}$ near $2$, larger
    $\sgz$ allows for non-linear effects to induce edge-of-stability behavior (right).
    Resulting loss trajectories are non-monotonic but still converge to $0$ (left).}
    \label{fig:gd_traj}
\end{figure}

We can confirm this more generally by initializing over various $\lr$,
$\D$, $\P$,
$\sgz$, and $\sgJ$ over multiple seeds, and plotting the resulting phase diagram of
the final $\lam_{\max}$ reached. We can simplify the plotting with some rescaling
of parameters and initializations. For example, in the rescaled variables
\begin{equation}
\tl{\z} = \lr\z,~\tl{\J} = \lr^{1/2}\J\,,
\end{equation}
the dynamics are equivalent to Equations \ref{eq:GD_in_z_general} and 
\ref{eq:GD_in_J_general} with $\lr = 1$. As in the $\tzz-\TT(0)$ model of Equations \ref{eq:Tzmodel}--\ref{eq:Tzmodel2},
$\lam_{\max}$ in the rescaled coordinates is equivalent to
$\lr\lam_{\max}$ in the unscaled coordinates.
We can also define rescaled initializations for $\z$ and $\J$. If we set
\begin{equation}
\sigma_{z} = \tilde{\sigma}_{z}/\D,~\sigma_{J} = \tilde{\sigma}_{J}/\left(\D\P\right)^{1/4}\,,
\end{equation}
then we have $r_{NL} = \tsgz$
which allows for easier comparison across $(\D, \P)$ pairs.

Using this initialization scheme, we can plot the final value of $\lam_{\max}$ reached
as a function of $\tsgz$ and $\tsgJ$ for $100$ independent
random initializations for each $\tsgz$, $\tsgJ$ pair
(Figure \ref{fig:phase_plane_GD}). We see that
the key is for $r_{NL} = \tsgz$ to be $O(1)$ - corresponding to both progressive sharpening
and non-linear dynamics near initialization. In particular, initializations with
small $\tsgJ$ values which converge at the EOS correspond to trajectories which
first sharpen, and then settle near $\lam_{\max} = 2/\lr$.
Large $\tsgz$ and large $\tsgJ$ dynamics diverge. There
is a small band of initial $\tsgJ$ over a wide range of $\tsgz$ which have final
$\lam_{\max} \approx 2/\lr$; these correspond to models initialized near the EOS,
which stay near it.

\begin{figure}[h]
    \centering
    \begin{tabular}{ccc}
    \includegraphics[width=0.35\linewidth]{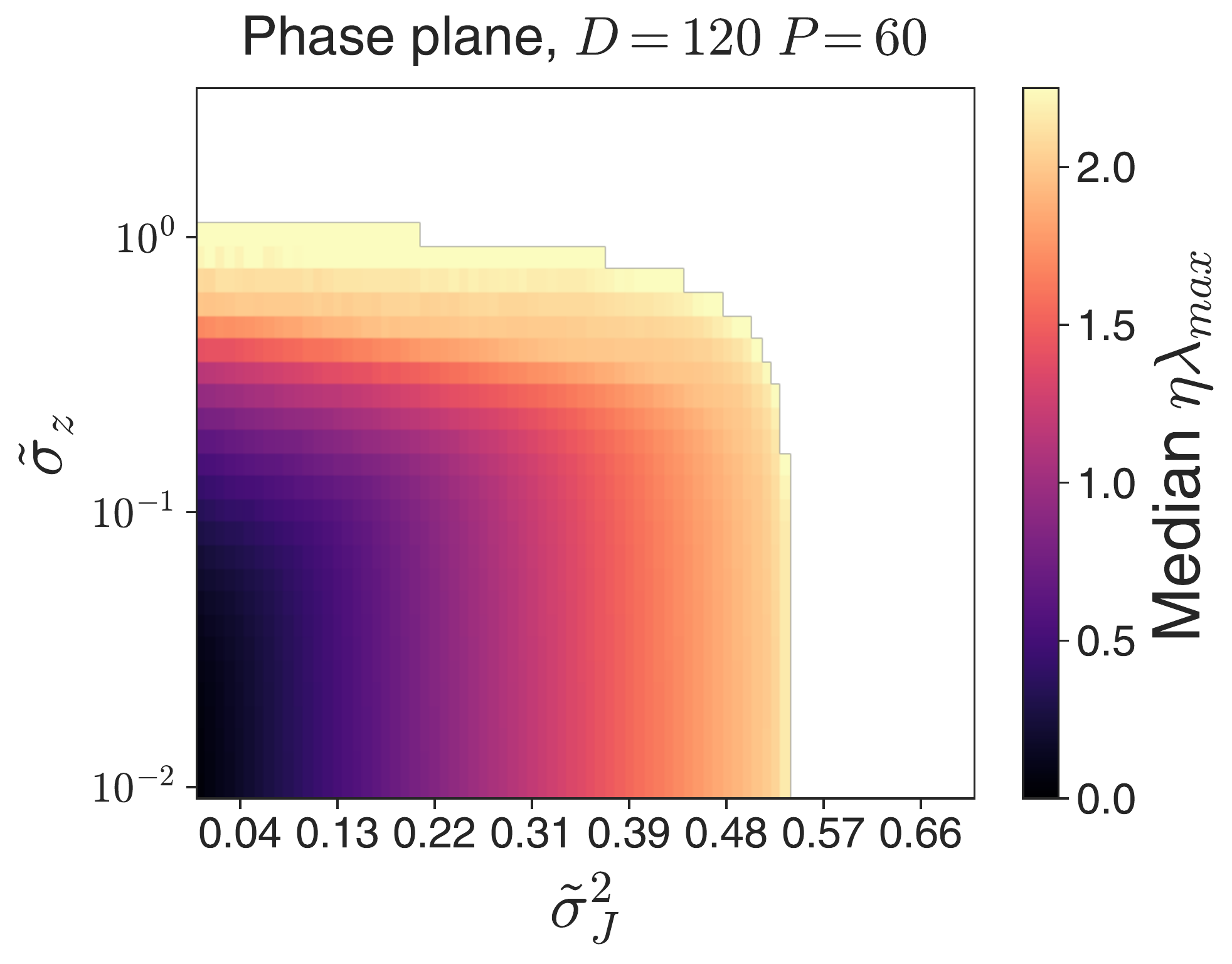} & 
    \includegraphics[width=0.35\linewidth]{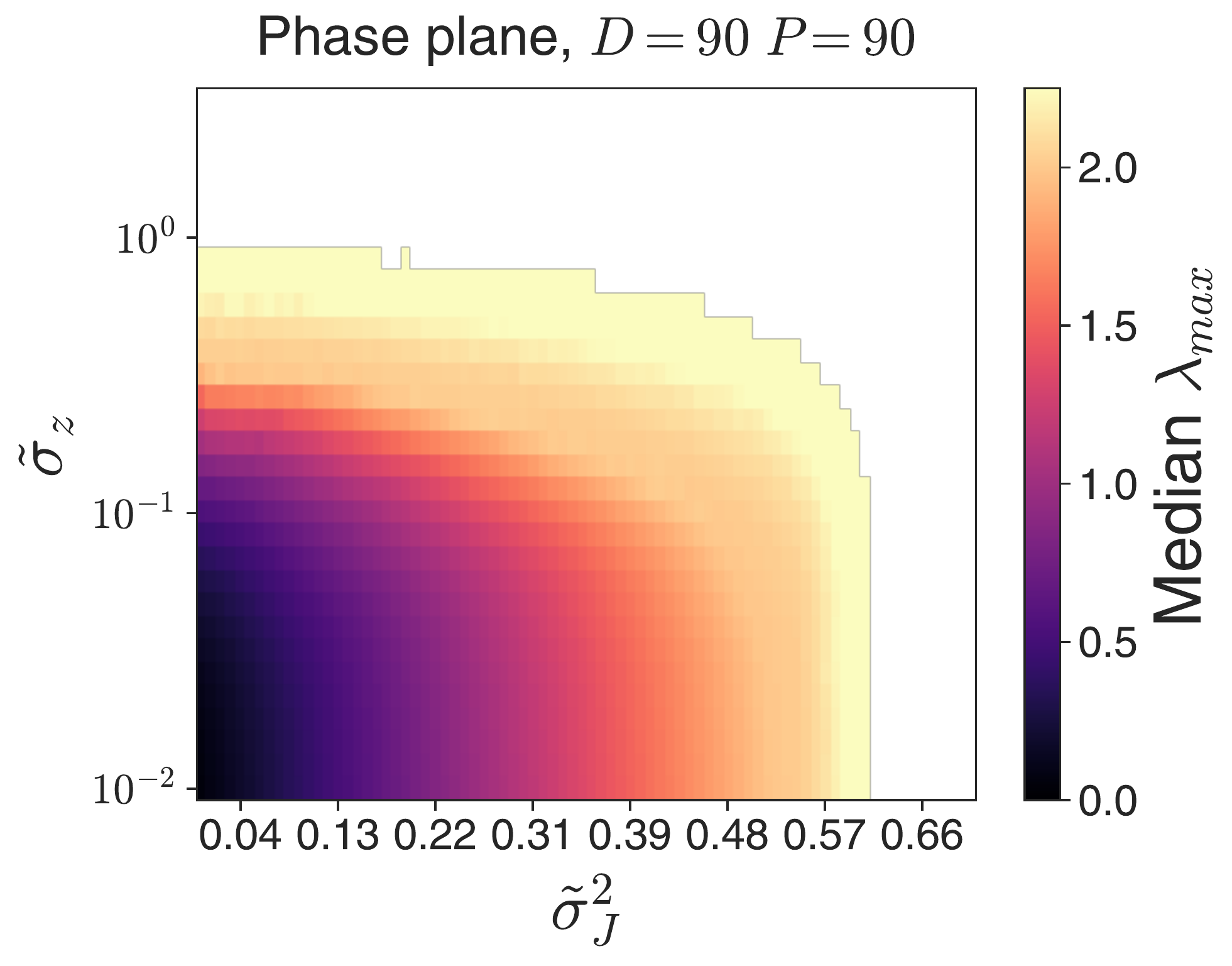} &
    \includegraphics[width=0.35\linewidth]{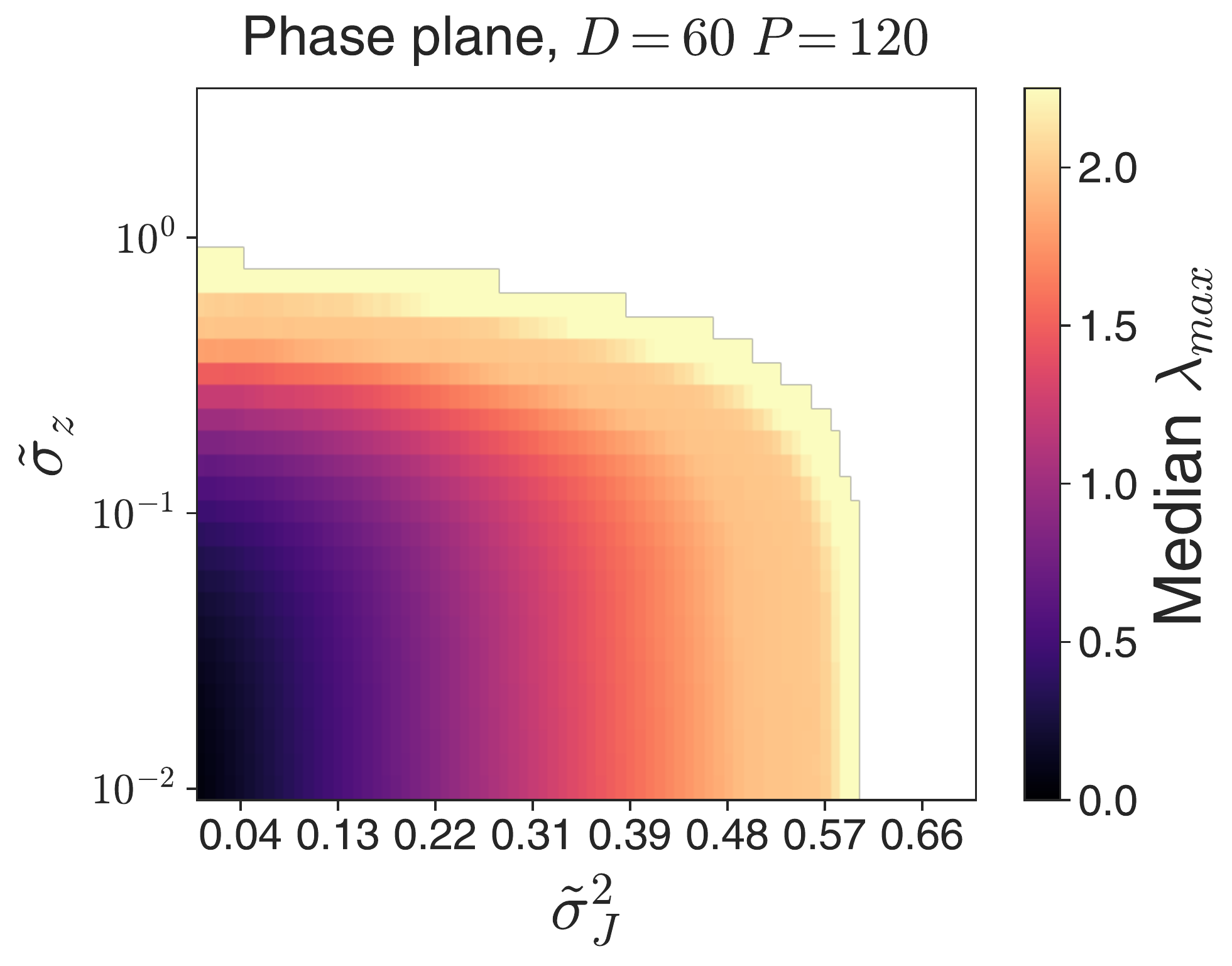} 
    \end{tabular}
    \caption{$\tsgz$/$\tsgJ^2$ phase planes for quadratic regression models, for various $\D$ and $\P$.
    Models were initialized with $100$ random seeds for each $\tsgz$, $\tsgJ$ pair and
    iterated until convergence. For each pair $\tsgz, \tsgJ^2$ we plot the median $\lam_{\max}$ of the NTK $\J^{\top}\J$.
    For intermediate $\tsgz$, where both sharpening and non-linear $\z$ dynamics occur,
    trajectories tend to converge so $\lam_{\max}$ of the NTK
    is near $2/\lr$ (EOS).}
    \label{fig:phase_plane_GD}
\end{figure}

This suggests that progressive sharpening and edge of stability aren't uniquely
features of neural network models, and could be a more general property of learning in high-dimensional, 
non-linear models.

\section{Connection to real world models}

\label{sec:real_world_model}

In this section we examine how representative is the proposed model and the developed theory to the behavior of ``real world''
models. Following \cite{cohen_gradient_2022}, we trained a $2$-hidden layer $\tanh$ 
network using the squared loss on
$5000$ examples from CIFAR10 with learning rate $10^{-2}$
- a setting which shows edge of
stability behavior. Close to the onset of EOS, we approximately
computed $\lam_{1}$, the largest eigenvalue of $\J\J^{\top}$, and its corresponding eigenvector
$\v_{1}$
using a Lanczos method \citep{ghorbani_investigation_2019, novak_neural_2019}. We use
$\v_{1}$ to compute $\zz_{1} = \v_{1}^\top\z$, where $\z$ is the vector of residuals
$f(\X, \th)-\Y$ for neural network function $f$, training inputs $\X$, labels $\Y$, and
parameters $\th$.
The EOS behavior in the NTK is similar to the EOS behavior defined with respect to the full Hessian in \cite{cohen_gradient_2022}
(Figure \ref{fig:cifar_eos}, left and right).
Once again, plotting the trajectories at every other step gets rid of the high frequency oscillations (Figure \ref{fig:cifar_eos}, middle). Unlike the $\D = 1$, $\P = 2$ model,
there are multiple crossings of the critical line
$\lam_{\max} = 2/\lr$ line.
\begin{figure}[h]
    \centering
    \begin{tabular}{ccc}
    \includegraphics[width=0.3\linewidth]{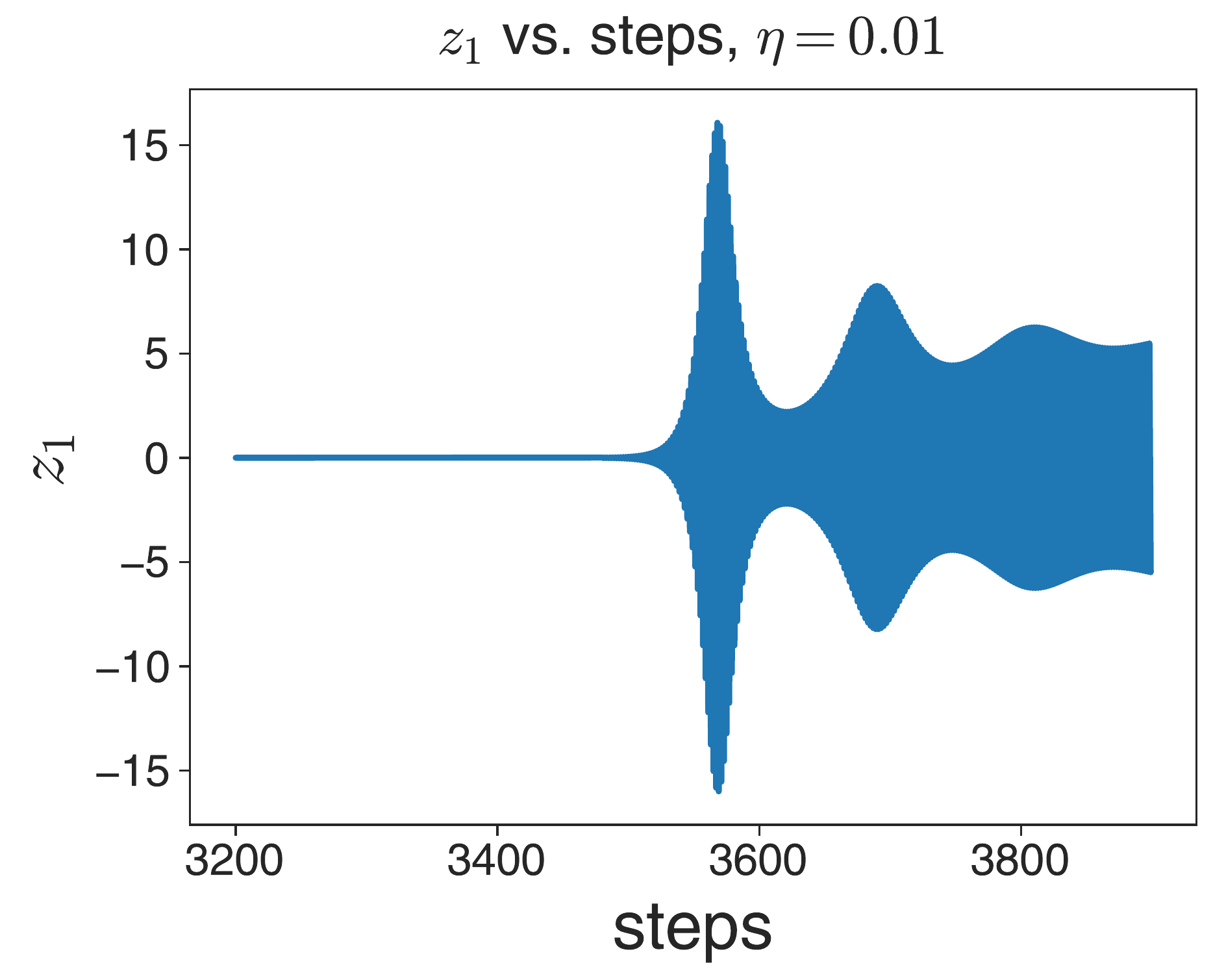} &
 \includegraphics[width=0.3\linewidth]{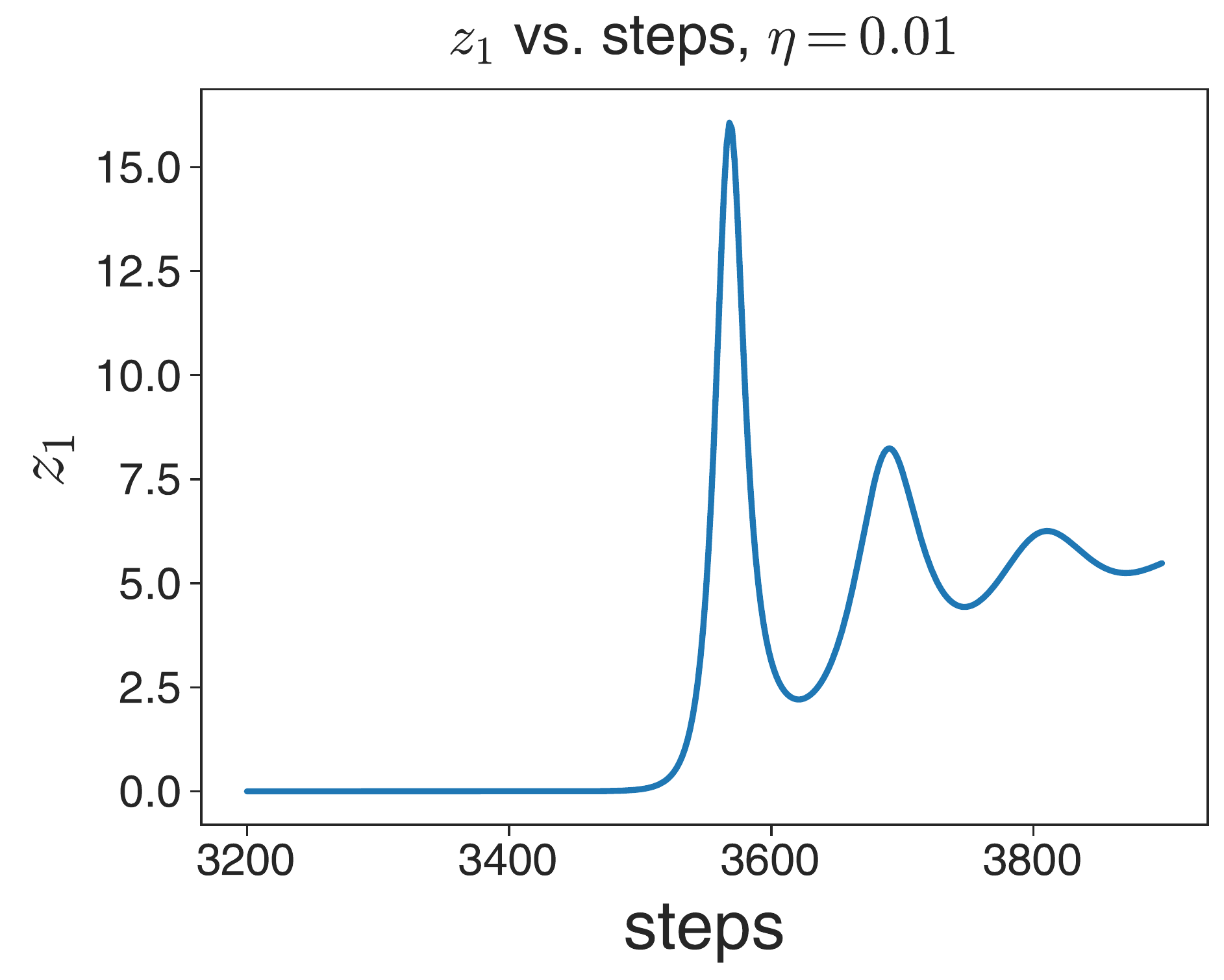} & \includegraphics[width=0.3\linewidth]{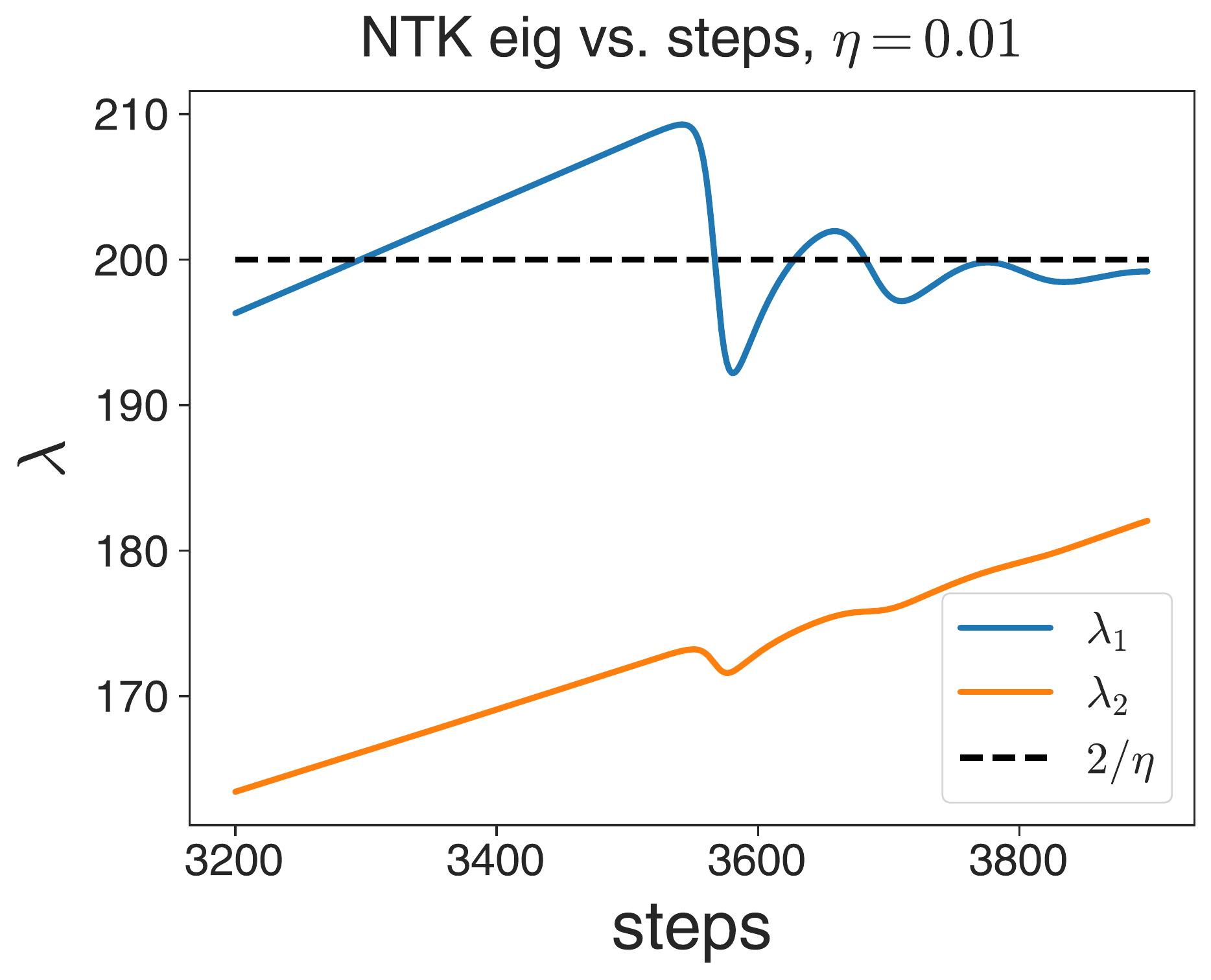}
    \end{tabular}
    \caption{A FCN trained on CIFAR shows multiple cycles of sharpening and edge-of-stability
    behavior. $\zz_{1}$, the projection of the training set residuals $f(\X,\th)-\Y$
    onto the top NTK eigenmode $\v_{1}$, increases in magnitude and
    oscillates around $0$ (left). Plotting dynamics every two steps removes high
    frequency oscillations
    (middle). The largest eigenvalue $\lam_{1}$ crosses the edge of stability multiple times,
    but the second largest eigenvalue $\lam_{2}$ remains below the edge of stability.}
    \label{fig:cifar_eos}
\end{figure}

There is evidence that low-dimensional features of a quadratic regression model
could be used to explain some aspects of EOS behavior. We empirically
compute the the second derivative of the output $f(\x,\th)$ by automatic differentiation. We denote by $\Q(\cdot, \cdot)$ the resulting tensor. We can use matrix-vector products to compute the
spectrum of the matrix
$\Qm_{1}\equiv \v_{1}\cdot\Q(\cdot, \cdot)$, which is projection of the output of $\Q$ in the $\v_{1}$
direction, without instantiating $\Q$ in memory (Figure \ref{fig:cifar_comparison}, left). This figure reveals that the spectrum does not shift much from
step $3200$ to $3900$ (the range of our plots). This suggests that $\Q$ doesn't change much
as these EOS dynamics are displayed.
We can also see that $\Q$ is much larger in the $\v_{1}$ direction than
a random direction.

Let $y$ be defined as $y = \lam_{1}\lr-2$. Plotting the two-step dynamics of $\zz_{1}$ versus $2yz$ we see a remarkable agreement (Figure \ref{fig:cifar_comparison}, middle). This is the same form
that the dynamics of $\tzz$ takes in our simplified model. It can also be found by iterating
Equation \ref{eq:GD_in_z_general} twice with fixed Jacobian for $y = \lam_{1}\lr-2$ and
discarding terms higher order in $\lr$. This suggests that during this particular EOS behavior,
much like in our simplified model the dynamics of the eigenvalue is more important
than any rotation in the eigenbasis.

The dynamics of $y$ is more complicated; $y_{t+2}-y_{t}$ is anticorrelated with
$\zz_{1}^{2}$ but there is no low-order functional form in terms of $y$ and $\zz_{1}$
(Appendix \ref{app:y_dyn}).
We can get some insight into the stabilization by plotting the ratio of
$\lr^2\Qm_{1}(\J\zz_{1}\v_{1},\J\zz_{1}\v_{1})$
(the non-linear contribution to the $\zz_{1}$ dynamics from the $\v_{1}$
direction) and $\lam_{1}\zz_{1}$
(the linearized contribution), and compare it to the dynamics of $y$
(Figure \ref{fig:cifar_comparison}, right).
The ratio is small during the initial sharpening,
but becomes $O(1)$ shortly before the curvature decreases for the first time. It remains
$O(1)$ through the rest of the dynamics. This suggests that the non-linear feedback from the
dynamics of the top eigenmode onto itself is crucial to understanding
the EOS dynamics.

\begin{figure}[h]
    \centering
    \begin{tabular}{ccc}
    \includegraphics[height=0.22\linewidth]{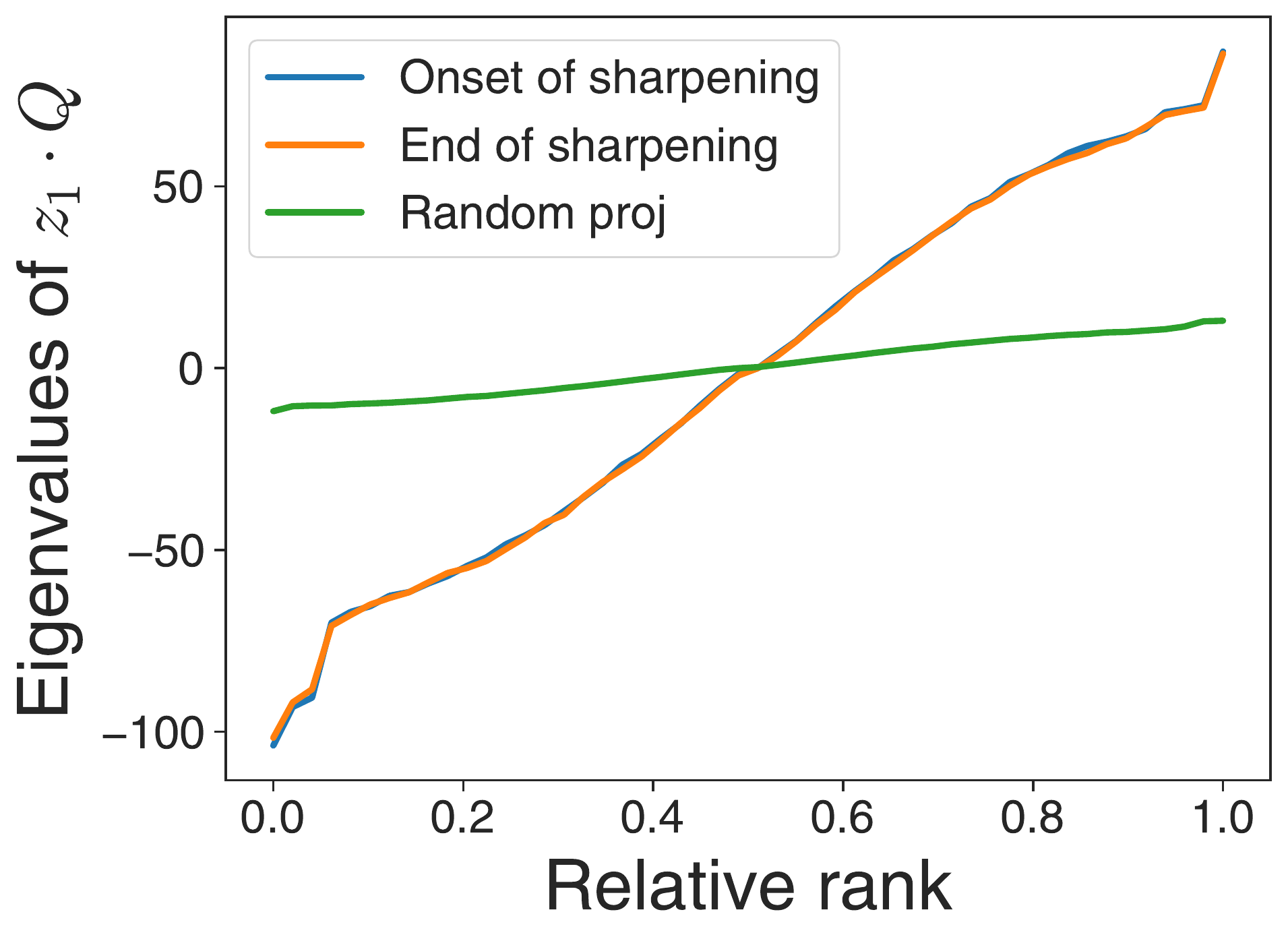} &
    \includegraphics[height=0.22\linewidth]{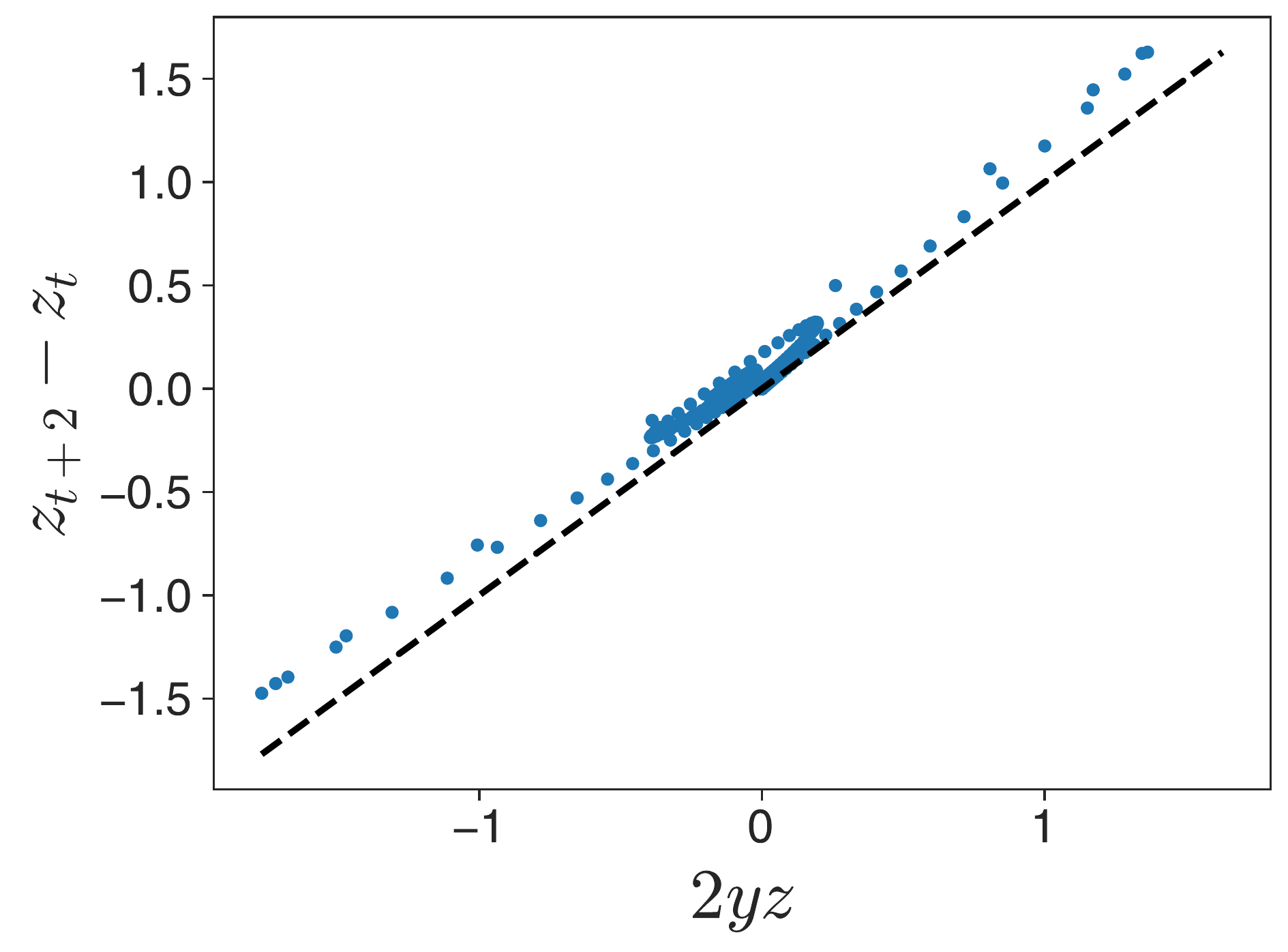} &
  \includegraphics[height=0.22\linewidth]{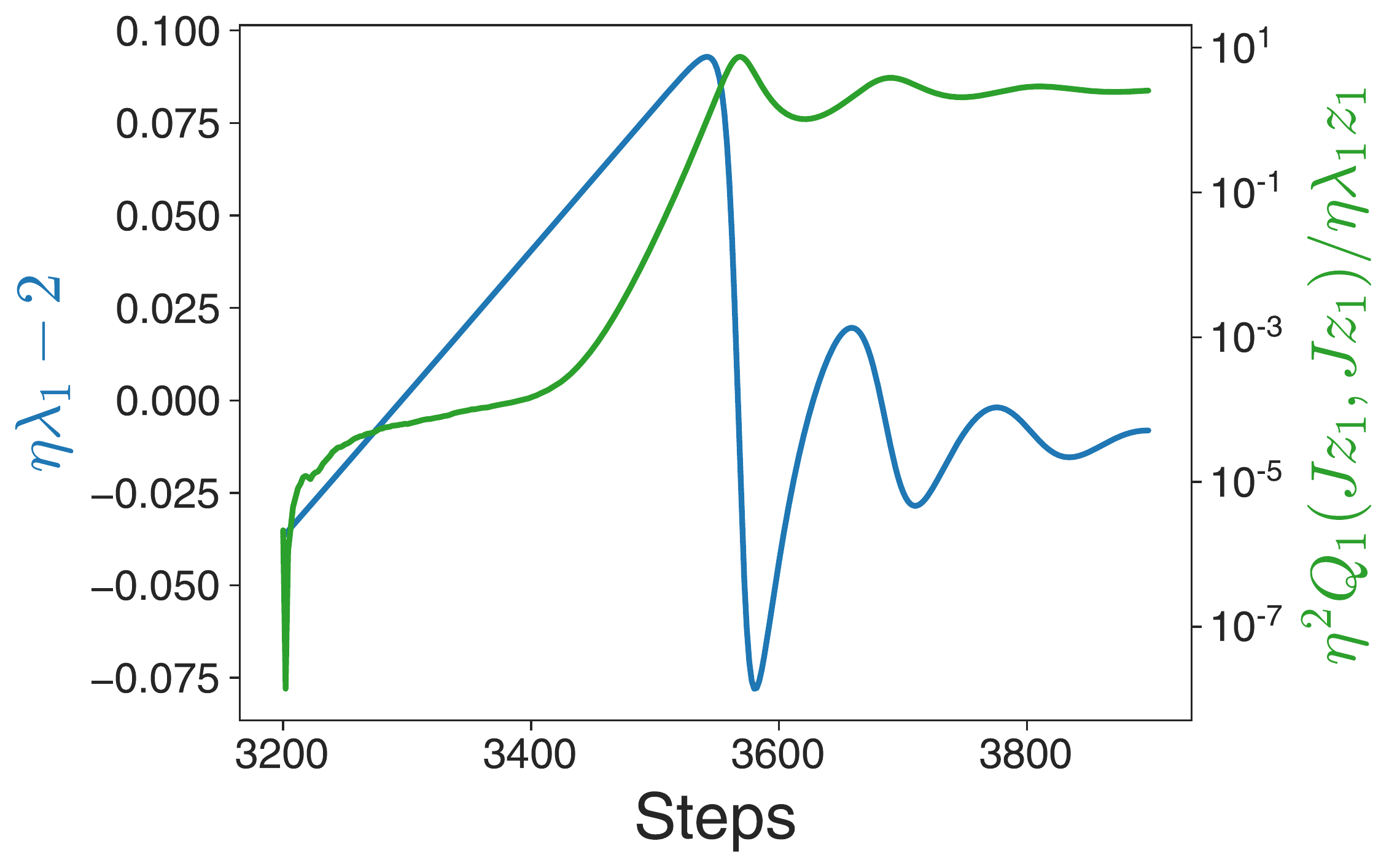}
    \end{tabular}
    \caption{$\Q$ is approximately constant during edge-of-stability dynamics for FCN trained
    on CIFAR10 (left). Projection onto largest eigendirection $\v_{1}$ (blue and orange) is larger than projection
    onto random direction (green).
    Two step difference $(\zz_{1})_{t+2}-(\zz_{1})_{t}$ is well approximated by $2\zz_{1} y$ (middle),
    leading order term of models with fixed eigenbasis.  Non-linear dynamical contribution $\lr^2\Qm_{1}(\J\zz_{1}\v_{1}, \J\zz_{1}\v_{1})$ is small during
    sharpening, but becomes large immediately preceding decrease in top eigenvalue (right) -
    as is the case in the simple model.}
    \label{fig:cifar_comparison}
\end{figure}

For smaller models, we can compute the full $\Q$ and then numerically integrate
Equations \ref{eq:GD_in_z_general} and \ref{eq:GD_in_J_general} directly. This is equivalent
to training a quadratic Taylor expansion of the full model.
In Appendix \ref{app:quad_expo}, we perform such a quadratic expansion of a
fully connected model on a two-class CIFAR dataset. Expanding at initialization, we see that the
maximum eigenvalue is well approximated by the quadratic model for early times, but misses
the sharpening regime (Figure \ref{fig:quad_expo}, left). Expanding closer to the sharpening
regime, we see that the quadratic model captures some features of the EOS, especially the
first crossing (Figure \ref{fig:quad_expo}, middle), but
the period and magnitude of the 
oscillations around $y = 0$ are not correctly captured by the quadratic expansion.
Nonetheless, quadratic models show the convergence to a stable two-cycle above
and below $y = 0$,
with a negative mean value (Figure \ref{fig:quad_expo}, right) - as seen in both
the full model and the simpler two-parameter model.

\begin{figure}[h]
 \centering
    \begin{tabular}{ccc}
    \includegraphics[height=0.23\linewidth]{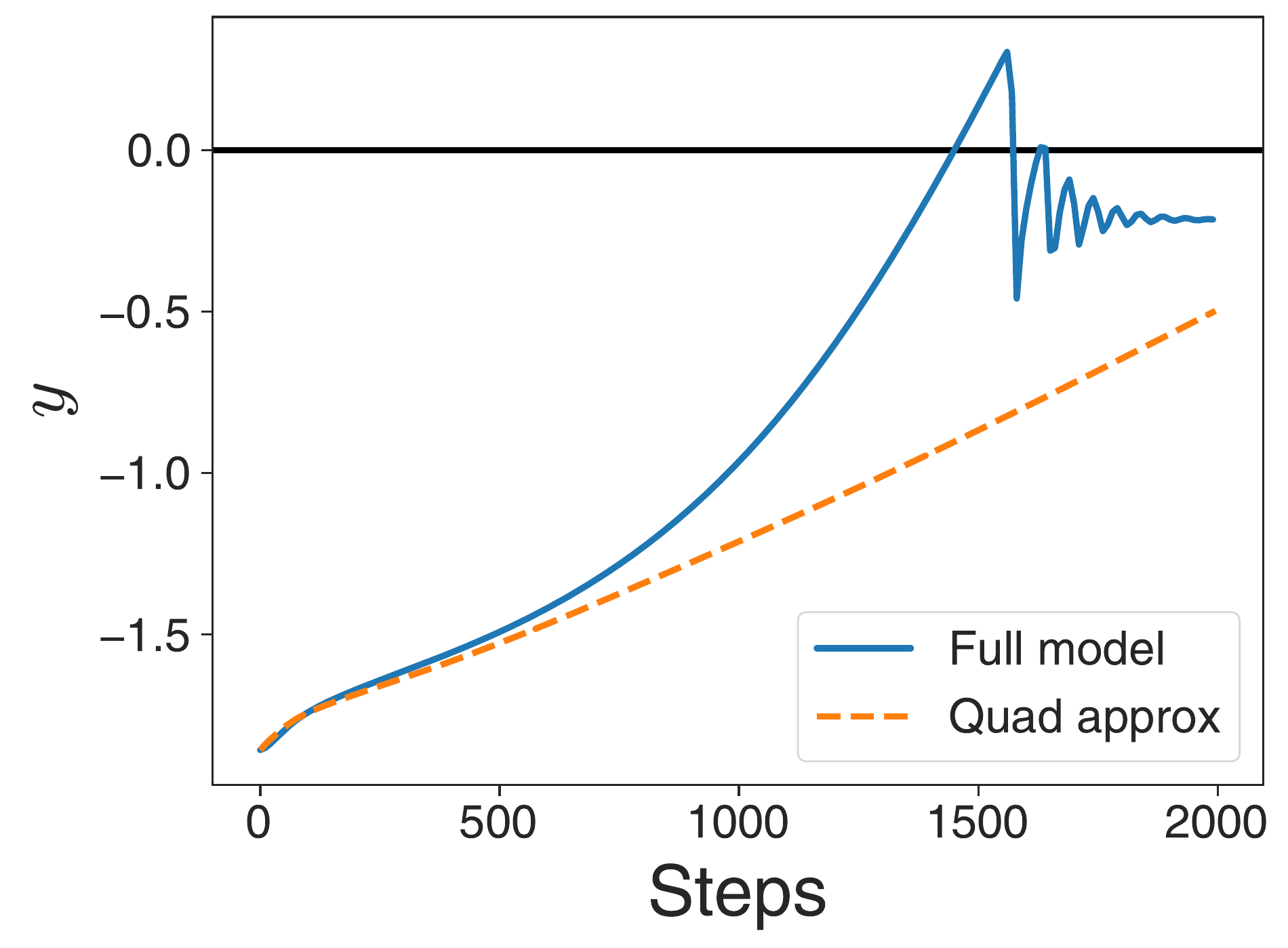} & \includegraphics[height=0.23\linewidth]{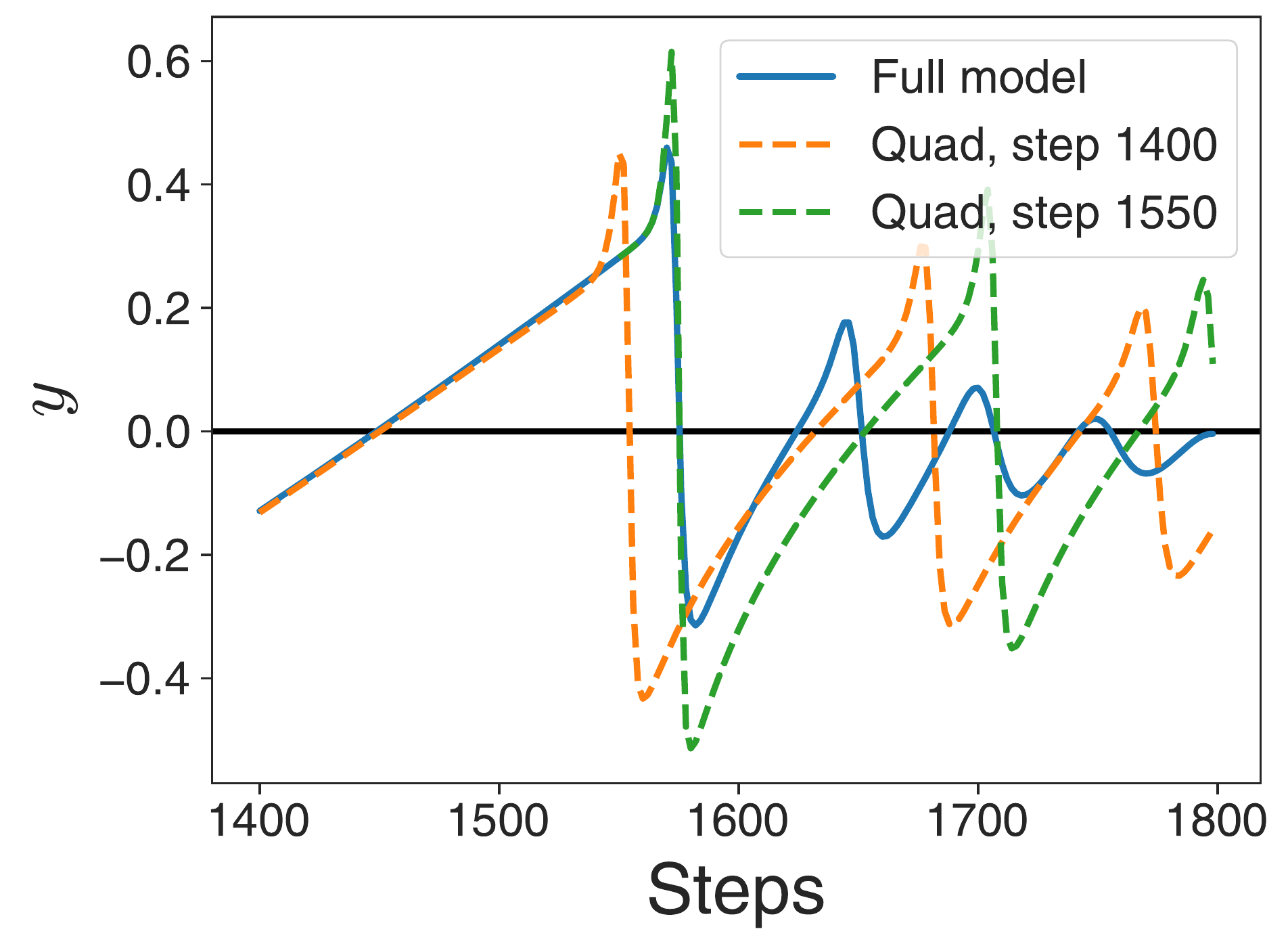} &
    \includegraphics[height=0.23\linewidth]{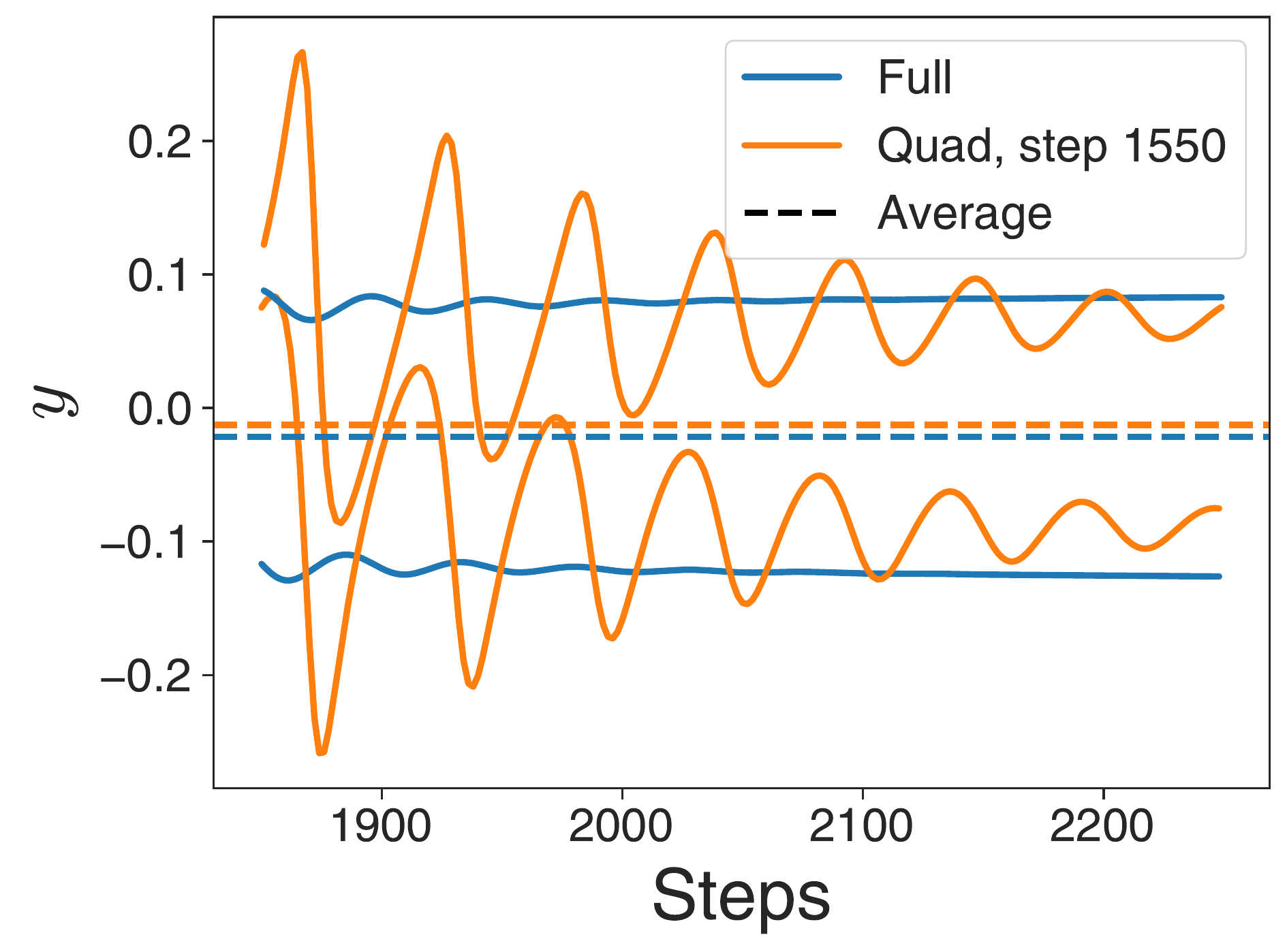}
    \end{tabular}
\caption{Quadratic expansion at initialization of a FCN model trained on $2$-class CIFAR captures 
early curvature dynamics of full model (left). Expanding closer to the first $y = 0$
crossing shows multiple oscillations in two-step dynamics, but period and magnitude quickly become 
different
from full dynamics (middle). Trajectories of even steps (top curves) and odd steps (bottom curves)
eventually stabilize, and average $y$ is non-zero for both models (right).}
\label{fig:quad_expo}
\end{figure}

\section{Discussion}

\subsection{Lessons learned from quadratic regression models}

The main lesson to be learned from the quadratic regression models is that behavior like progressive
sharpening (for both GF and GD) and edge-of-stability behavior (for GD) may be common
features of high-dimensional gradient-based training of non-linear models. Indeed, these phenomena can be revealed in simple settings without any connection to
deep learning models: with mild tuning our simplified model, which corresponds to $1$
datapoint and $2$ parameters can provably show EOS behavior. This combined with the
analysis of the CIFAR model suggest that the general mechanism may have a low-dimensional
description.

Quadratic approximations of real models quantitatively can capture the early features of
EOS behavior (the initial return to $\lam_{max} < 2/\lr$), but do not necessarily capture the magnitude and period of subsequent oscillations -- these require
higher order terms (Appendix \ref{app:quad_expo}). Nevertheless, the quadratic approximation
does correctly describe much of the qualitative behavior, including the convergence of $\lam_{max}$ to a limiting two-cycle that oscillates around $2/\lr$, with an average value \emph{below}
$2/\lr$. In the simplified two-parameter model, it is possible to analytically predict the final value at convergence, and indeed we find that it deviates slightly from the value $2/\lr$.

A key feature of all the models studied in this work is that looking at every-other iterate
(the two-step dynamics) greatly aids in understanding the models theoretically
and empirically. Near the edge of stability, this makes the changes in the top eigenmode
small.
In the simplified model, the slow $\tzz$ dynamics (and related slow $\TT(0)$ dynamics)
allowed for the detailed theoretical analysis; in the CIFAR model, the two-step dynamics
is slowly varying in both $\zz_{1}$ and $\lam_{\max}$. The quantitative comparisons of these
small changes may help uncover any universal mechanisms/canonical forms that
explain EOS behavior in other systems and scenarios.

\subsection{Future work}

One avenue for future work is to quantitatively understand progressive sharpening and
EOS behavior
in the quadratic regression model for large $\D$ and $\P$. In particular, it may be possible to predict
the final deviation $2-\lr\lam_{\max}$ in the edge-of-stability regime as a function of
$\sgz$, $\sgJ$, and $\D/\P$.
It would also be useful to understand how higher order terms affect the training dynamics.
One possibility is that a small number of statistics of the higher order derivatives
of the loss function are sufficient to obtain a better quantitative understanding of the
oscillations around $y = 2$.

Finally, our analysis has not touched on the feature learning aspects of the model.
In the quadratic regression model, feature learning is encoded in the relationship between $\J$
and $\z$, and in particular the relationship between $\z$ and the eigenstructure of
$\J\J^{\top}$. Understanding how $\Q$ mediates the dynamics of these two quantities
may provide a quantitative basis for understanding feature learning which is complementary
to existing theoretical approaches
\citep{roberts_principles_2022, bordelon_selfconsistent_2022, yang_tensor_2022}.

\bibliography{quad_learning_bibs}
\bibliographystyle{iclr2023_conference}

\appendix

\section{Connection to other models}

\subsection{One-hidden layer linear network}

\label{app:one_hidden_layer}

Consider a one hidden layer network with a scalar output:
\begin{equation}
f(\x) = \v^{\top}\U\x
\end{equation}
where $\x$ is an input vector of length $\N$, $\U$ is a $\K\times\N$ dimensional
matrix, and $\v$ is a $\K$ dimensional vector. We note that
\begin{equation}
\frac{\partial^{2} f(\x)}{\partial\v_{i}\partial\v_{j}} = \frac{\partial^{2} f(\x)}{\partial\U_{ij}\partial\U_{kl}} = 0,~\frac{\partial^{2} f(\x)}{\partial\v_{i}\partial\U_{jk}} = \dl_{ij}\x_{k}
\end{equation}
where $\dl_{ij}$ is the Kroenecker delta. For a fixed training set, this second derivative
is constant; therefore, the one-hidden layer linear network is a quadratic regression model
of the type studied in Section \ref{sec:quad_reg_model}.

In the particular case of a single datapoint $\x$, we can compute the eigenvectors of
the $\Q$ matrix. Let $(\w, \W)$ be an eigenvector of $\Q$, representing
the $\v$ and $\U$ components respectively. The eigenvector equations are
\begin{equation}
\qlam\w_{i} = \x_{m}\delta_{ij}\W_{jm}
\end{equation}
\begin{equation}
\qlam \W_{jm} = \x_{m}\delta_{ij}\w_{i}
\end{equation}
Simplifying, we have:
\begin{equation}
\qlam \w = \W\x
\end{equation}
\begin{equation}
\qlam \W = \w\x^{\top}
\label{eq:W_eig_eq}
\end{equation}
We have two scenarios. The first is that $\qlam = 0$.
In this case, we have
$\w = 0$, and $\W$ is a matrix with $\x$ in its nullspace. The latter condition gives us
$\M$ constraints on $\M\times\N$ equations - for a total of $\M(\N-1)$ of our $\M(\N+1)$ total
eigenmodes.

If $\qlam\neq 0$, then combining the equations we have the conditions:
\begin{equation}
\qlam^2 \w = (\x\cdot\x) \w
\end{equation}
\begin{equation}
\qlam^2 \W = \W\x\x^{\top}
\end{equation}
This gives us $\qlam = \pm \sqrt{ \x\cdot\x}$.
We know from Equation \ref{eq:W_eig_eq} that $\W$ is low rank. Therefore, we can guess a solution
of the form
\begin{equation}
\W_{\pm, i} = \pm\m{e}_{i}\x^{\top}
\end{equation}
where the $\m{e}_{i}$ are the $\M$ coordinate vectors. This suggests that we have
\begin{equation}
\w_{\pm, i} = (\sqrt{\x\cdot\x}) \m{e}_{i}
\end{equation}
This gives us our final $2\M$ eigenmodes.

We can analyze the initial values of of the $\tjjl$ as well. The components of the Jacobian
can be written as:
\begin{equation}
(\J_{v})_{i}\equiv \frac{\partial f(\x)}{\partial \v_{i}} = \U_{i m}\x_{m}
\end{equation}
\begin{equation}
(\J_{U})_{jm}\equiv \frac{\partial f(\x)}{\partial \U_{jm}} = \v_{j}\x_{m}
\end{equation}
From this form, we can deduce that $\J$ is orthogonal to the $0$ modes.
We can also compute the conserved quantity. Let $\jj_{+}^{2}$ be the total weight
in the positive eigenmodes, and $\jj_{-}^{2}$ be the total weight in the negative
eigenmodes. A direct calculation shows that
\begin{equation}
\qlam^{-1}(\jj_{+}^{2}-\jj_{-}^{2}) = 2f(\x)
\end{equation}
which implies that $E = 0$. 

Therefore, the single-hidden layer linear model on one datapoint is equivalent to the
quartic loss model with $E = 0$ and eigenvalues $\pm\sqrt{\x\cdot\x}$.

\subsection{Connection to \cite{bordelon_selfconsistent_2022}}

\jp{
The quadratic regression model should be related to the dynamics in Section F.1.1. Something like: $y = -E/2$, $P=2$, $\lambda_0 = \lambda_1 = 4 \gamma_0$, $H_y = J_0^2$, $G = J_1^2$, $\Delta = z$ gives the agreement I think?
}
\aga{Fixed, I think. There may be a way to make F.1 correspond but I can't immediately see
it. You need to factorize $\m{H}+G\m{K}$ into $\J\J^{\top}$ somehow. If you can do that,
you can infer $\Q$. Alternatively, you might be able to solve for $\Q$ using the structure of
the $1$-hidden layer NN, and then use that to work out $\J$?}

Since the one-hidden layer linear model has constant $\Q$, the models in Section
F.1 of \cite{bordelon_selfconsistent_2022} fall into the quadratic regression class.
In the case of Section F.1.1, Equation 67, we can make the mapping to a
$\D = 1$ model explicit. The dynamics are equivalent to said model with
a single eigenvalue $\qlam_{0}$ if we make the identifications
\begin{equation}
\Delta = \tzz,~H_y = J_{0}^2,~\gamma_{0} = \sqrt{2\qlam},~y=-\E/2
\end{equation}

\subsection{Connection to NTH}

\label{app:nth}

The Neural Tangent Hierarchy (NTH) equations extend the NTK dynamics to account for changes
in the tangent kernel by constructing an infinite sequence of higher order tensors
which control the non-linear dynamics of learning \cite{huang_dynamics_2020}. Truncation
of the NTH equations at $3$rd order is related to, but not the same as the quadratic
regression model, as we will show here.

The $3$rd order NTH equation describes the change in the tangent kernel $\J\J^{\top}$.
Consider the $\D\times\D\times\D$-dimensional kernel $\tens{K}_{3}$ whose elements are given by
\begin{equation}
(\tens{K}_{3})_{\al\bt\gm} = \frac{\partial^2\z_{\al}}{\partial\th_{i}\partial\th_{j}}\J_{i\gm}\J_{j\bt}
+\frac{\partial^2\z_{\bt}}{\partial\th_{i}\partial\th_{j}}\J_{i\gm}\J_{j\al}
\end{equation}
where repeated indices are summed over. In the NTH, for
squared loss the change in the NTK $\J\J^{\top}$
is given by
\begin{equation}
\frac{d}{dt}\left(\J\J^{\top}\right)_{\al\bt} = -\lr(\tens{K}_{3})_{\al\bt\gm}\z_{\gm}
\end{equation}
For fixed $\Q = \frac{\partial^2\z}{\partial\th\partial\th'}$, this equation is identical
to the GF equations for the NTK in the quadratic regression model. We
note that $\tens{K}_{3}$ is not constant under the quadratic regression model. Conversely,
for fixed $\tens{K}_{3}$, $\frac{\partial^2\z}{\partial\th\partial\th'}$ is not constant
either. Therefore, the two methods can be used to construct different low-order expansions
of the dynamics.

\section{2 parameter model}

\subsection{Derivation of $\tzz$-$\TT(0)$ equations}

\label{app:two_param_basics}

We can use the conserved quantity $\E$ to write the dynamics in terms of $\tzz$ and $\TT(0)$ only.
Without loss of generality, let the eigenvalues are $1$ and $\lam$, with $-1\leq \lam \leq 1$.
(We can achieve this by rescaling $\tzz$.)
Recall the dynamical equations
\begin{equation}
\tzz_{t+1}-\tzz_{t} = - \tzz_{t}\TT_{t}(0) + \frac{1}{2}(\tzz_{t}^2)\TT_{t}(1)
\end{equation}
\begin{equation}
\TT_{t+1}(0)-\TT_{t}(0) = -\tzz_{t}(2\TT_{t}(1)-\tzz_{t}\TT_{t}(2))
\end{equation}
We will find substitutions for $\TT(1)$ and $\TT(2)$ in terms of $\tzz$ and $\TT(0)$.
Recall that we have
\begin{equation}
\TT(-1) = \E+2\tzz
\end{equation}
where $\E$ is conserved throughout the dynamics (and indeed is a property of the landscape).
We will use this definition to solve for $\TT(1)$ and $\TT(2)$.

Since $\P =2$,
we can write $\TT(-1) = b\TT(0)+\aa\TT(1)$, for coefficients $\aa$ and $b$ which
are valid for all combinations of $\tjj$. If $\tjj(\lam) = 0$, we have $b = 1-\aa$. If $\tjj(1) = 0$, we have
$1 = \lam(1-\aa)+\lam^{2}\aa$. Solving, we have:
\begin{equation}
\TT(-1) = (1-\aa)\TT(0)+\aa\TT(1)~{\rm for~} \aa = -\frac{1}{\lam}
\end{equation}
The restrictions on $\lam$ translate to $\aa\notin(-1,1)$. In terms of the conserved quantity
$\E = \TT(-1)-2\tzz$, we have:
\begin{equation}
\TT(-1) = \E+2\tzz
\end{equation}

In order to convert the dynamics, we need to solve for $\TT(1)$ and $\TT(2)$ in terms of
$\TT(0)$ and $\tzz$. We have:
\begin{equation}
\TT(1) = \frac{1}{\aa}\left(\TT(-1)+(\aa-1)\TT(0)\right) = \frac{1}{\aa}\left(\E+2\tzz+(\aa-1)\TT(0)\right)
\end{equation}
We also have
\begin{equation}
\TT(2) = \TT(0)+\left(\frac{1-\aa}{\aa^2}\right)(\TT(0)-\E-2\tzz)
\end{equation}
This gives us
\begin{equation}
\tzz_{t+1}-\tzz_{t} = - \tzz_{t}\TT_{t}(0) + \frac{1}{2\aa}(\tzz_{t}^2)((\aa-1)\TT_{t}(0)+2\tzz_{t}+\E)
\end{equation}
\begin{equation}
\TT_{t+1}(0)-\TT_{t}(0) = -\frac{2}{\aa}\tzz_{t}(2\tzz_{t}+\E+(\aa-1)\TT_{t}(0))+\zz_{t}^{2}\left[\TT_{t}(0)+\left(\frac{1-\aa}{\aa^2}\right)(\TT_{t}(0)-\E-2\tzz_{t})\right]
\end{equation}
If $\lam = -\eps$ (that is, $\aa = \eps^{-1}$) we recover the equations from the main text.

The non-negativity of $\tjj^2$ gives us constraints on the values of $\tzz$ and $\TT$.
For $\aa >1$ (small negative second eigenvalue), the constraints are:
\begin{equation}
\TT>2\tzz+\E,~\TT>-(2\tzz+\E)/\aa
\end{equation}
This is an upward-facing cone with vertex at $\tzz = -\E/2$ (Figure \ref{fig:one_step_lam_dep}, 
left).
For $\aa <-1$, the constraints are
\begin{equation}
-(2\tzz+E)/\aa<\TT<2\tzz+\E
\end{equation}
This is a sideways facing cone with vertex at $\tzz = -\E/2$ (Figure \ref{fig:one_step_lam_dep},
right). We see that in this case, there is a limited set of values of $\TT$ to converge to.
Indeed, for $\E = 0$, there is no convergence except at $\TT(0) = 0$.
This why we focus on the case of one positive and one negative eigenvalue.

We can also solve for the nullclines - the curves where either $\tzz_{t+1}-\tzz_{t} = 0$
(blue in Figure \ref{fig:one_step_lam_dep}),
or $\TT_{t+1}(0)-\TT_{t}(0) = 0$ (orange in Figure \ref{fig:one_step_lam_dep}).
The nullcline $(\tzz, f_{\tzz}(\tzz))$ for $\tzz$ is given by
\begin{equation}
f_{\tzz}(\tzz) = \frac{\tzz(2\tzz+\E)}{2\aa-(\aa-1)\tzz}
\end{equation}
The nullcline $(\tzz, f_{\TT}(\tzz))$ for $\TT(0)$ is given by
\begin{equation}
f_{\TT}(\tzz) = -\frac{(\aa-1)\tzz-2\aa}{(\aa^2-\aa+1)\tzz-2\aa(\aa-1)}(2\tzz+\E)
\end{equation}
The line $\tzz = 0$ is also a nullcline.

For the symmetric model $\eps = 1$, the structure of the nullclines determines the
presence or lack of progressive sharpening. For $\E = 0$, there is no sharpening;
the phase portrait (Figure \ref{fig:sym_phase_plane}, left) confirms this as the
nullcline in $\TT_{t}(0)$ divides the space into two halves, one which converges,
and the other which doesn't. However, when $\E\neq 0$, the nullclines split,
and there is a small region where progressive sharpening can occur
(Figure \ref{fig:sym_phase_plane}, middle). However, there is still no edge-of-stability
behavior in this case - there is no region where the trajectories cluster near
$\lam_{\max} = 2/\lr$ (Figure \ref{fig:sym_phase_plane}, right).

\begin{figure}
    \centering
    \begin{tabular}{ccc}
    \includegraphics[width=0.33\linewidth]{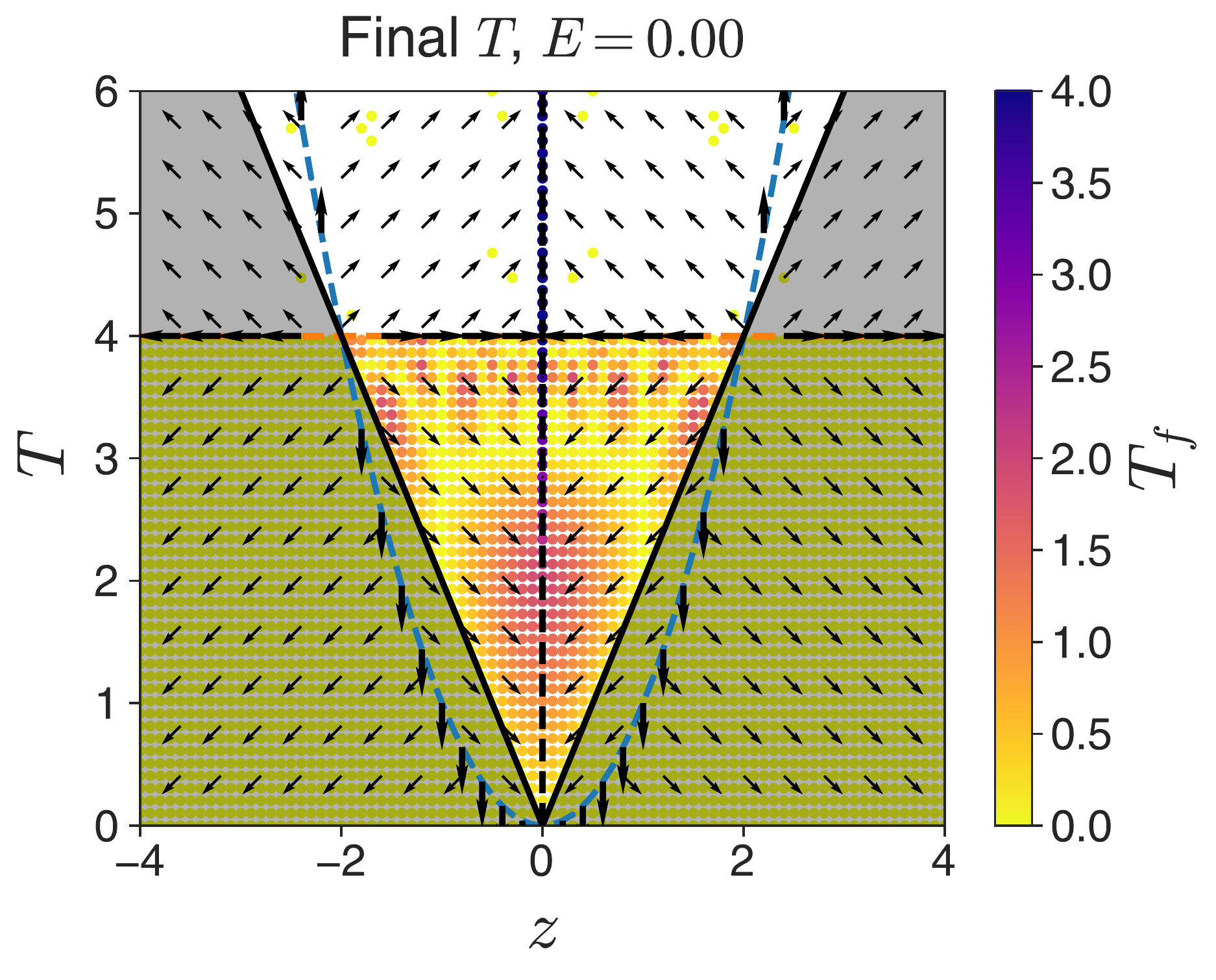} &
    \includegraphics[width=0.33\linewidth]{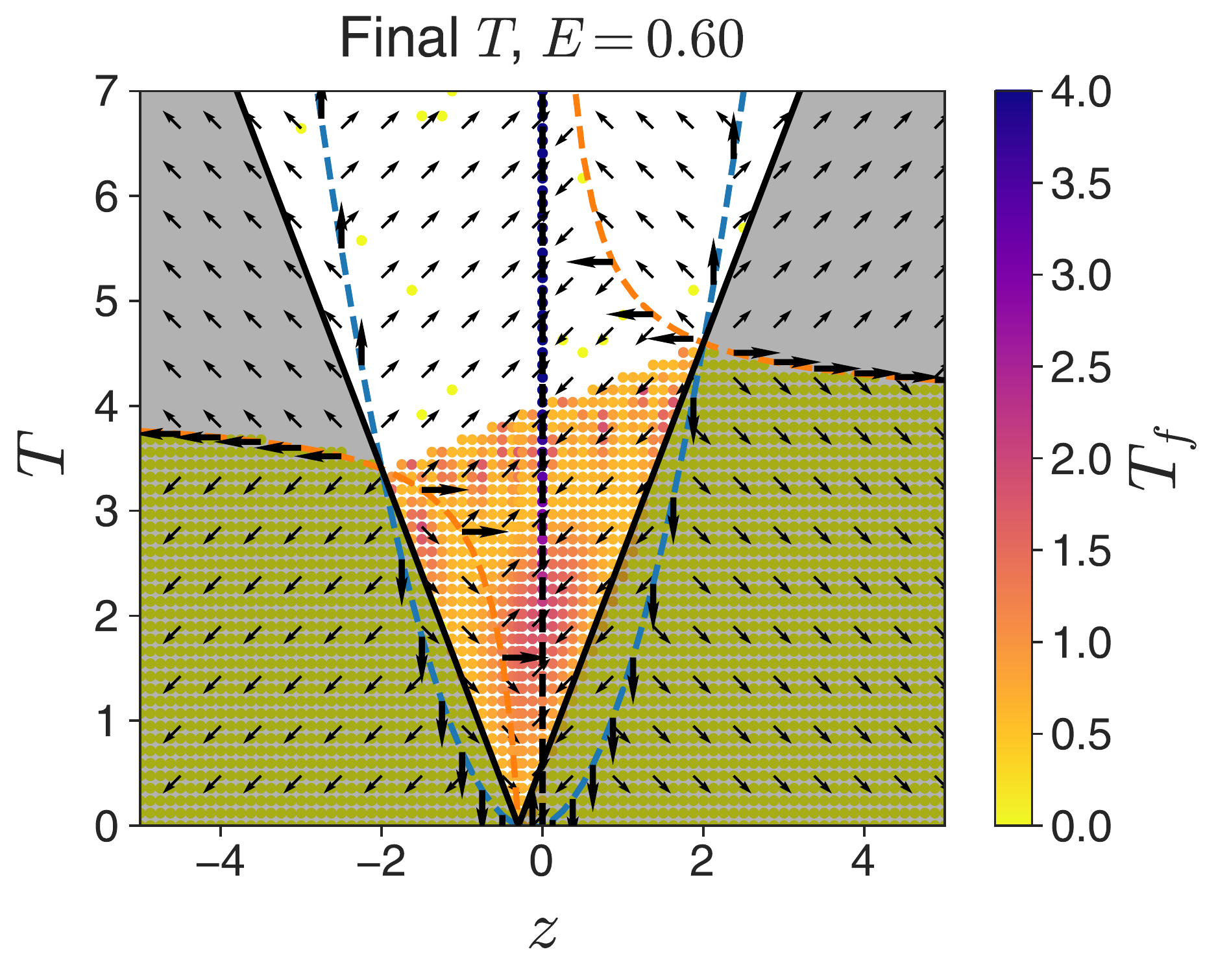} &
    \includegraphics[width=0.33\linewidth]{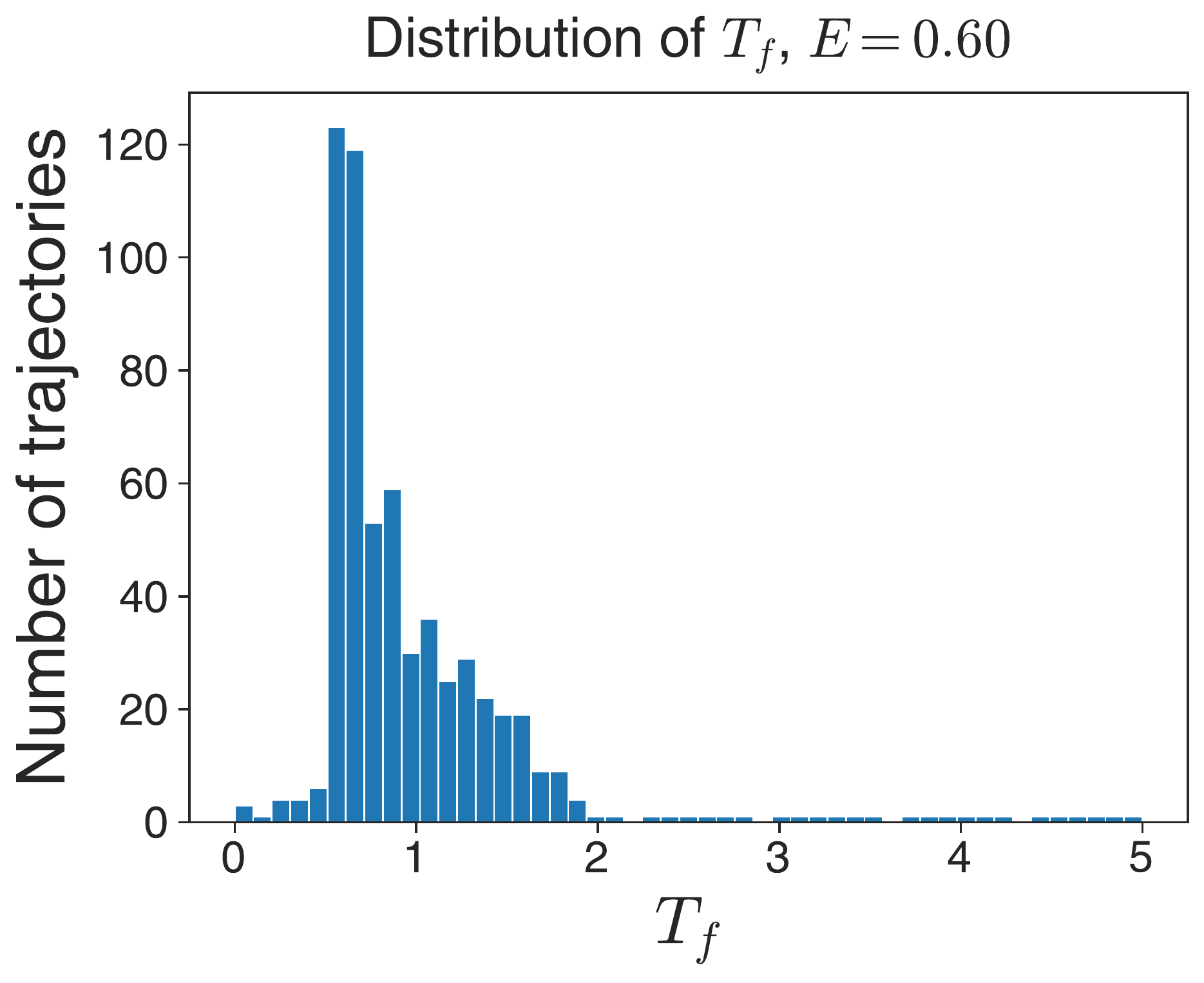}
    
    \end{tabular}
    \caption{Phase portraits for symmetric model. Arrows indicate signs of changes
    in $\z$ and $\TT$, and grey area represents disallowed coordinates. Dynamics are
    run from an evenly spaced grid of initializations, and the final value of the
    curvature $\TT(0)$ is recorded. Nullclines representing $\tzz_{t+1}-\tzz_{t} = 0$
    (blue) and $\TT_{t+1}(0)-\TT_{t}(0) = 0$ (orange) depend on $\E$. Trajectories
    show progressive sharpening but no edge-of-stability effect (right).}
    \label{fig:sym_phase_plane}
\end{figure}

\begin{figure}[h]
    \centering
    \begin{tabular}{cc}
    \includegraphics[width=0.43\linewidth]{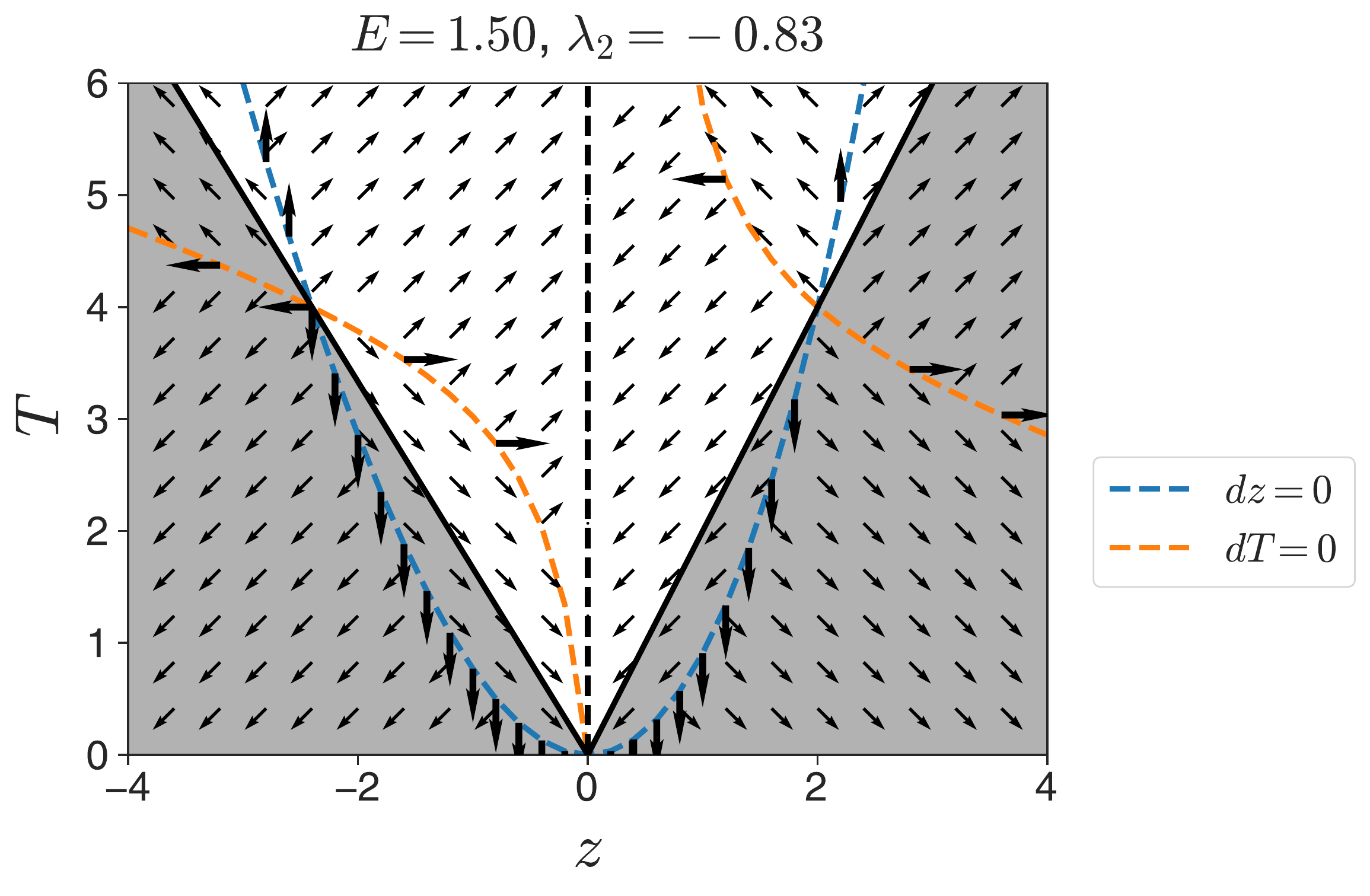} &
    \includegraphics[width=0.43\linewidth]{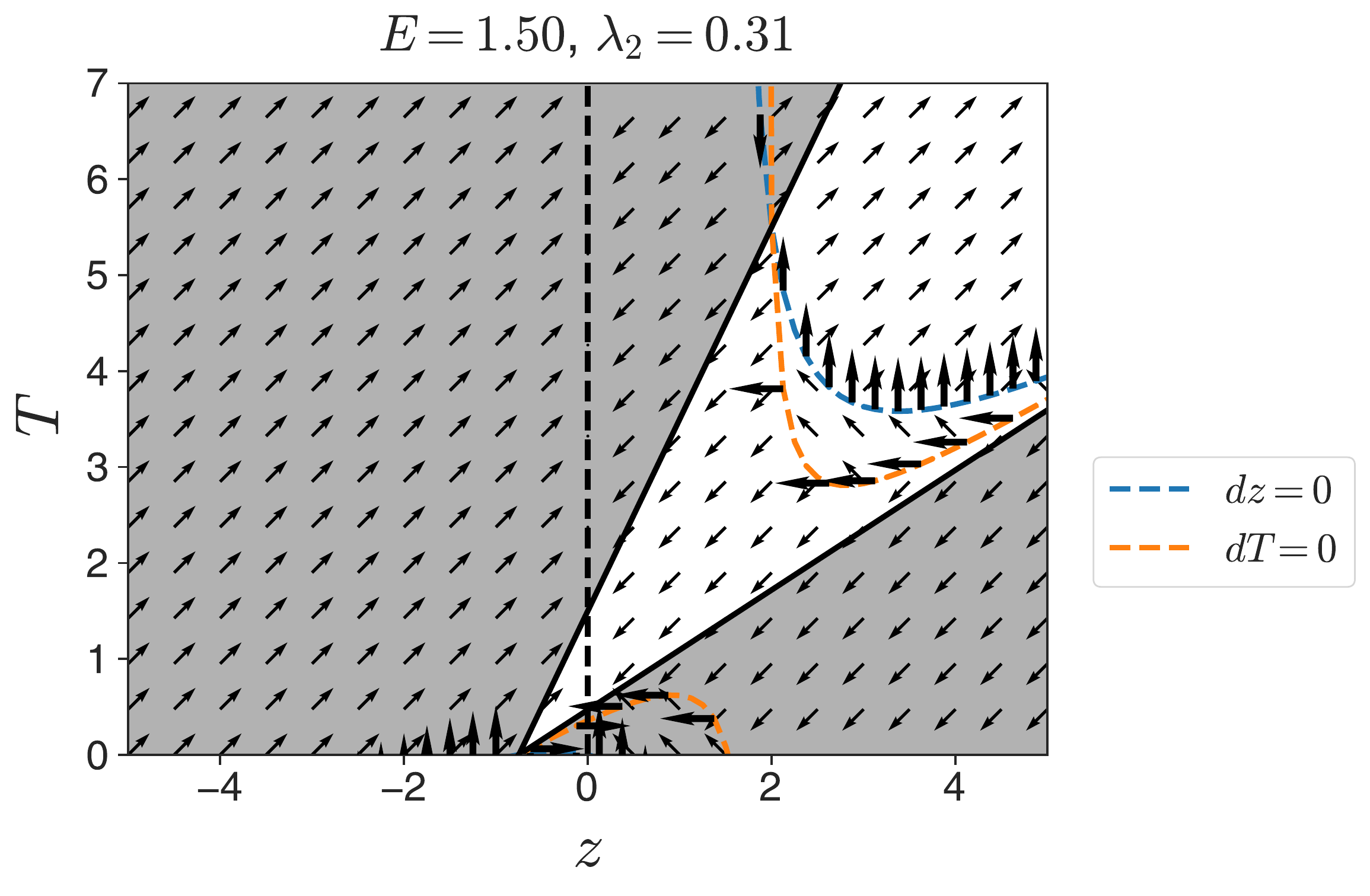}
    \end{tabular}
    \caption{Phase planes of $\D = 1$, $\P = 2$ model. Grey region corresponds to parameters
    forbidden by positivity constraints on $\tjjl^2$. For $\lam > 0$, allowed
    region is smaller and intersects $\tzz = 0$ at a small range only. Nullclines
    can be solved for analytically.}
    \label{fig:one_step_lam_dep}
\end{figure}

\subsection{Two-step dynamics}

\label{app:two_step}

The two-step difference equations can be derived by iterating Equations \ref{eq:z_2_red_eps}
and \ref{eq:T_2_red_eps}. We have
\begin{equation}
\tzz_{t+2}-\tzz_{t} = p_{0}(\tzz_{t},\eps)+ p_{1}(\tzz_{t}, \eps)\TT_{t}(0)+p_{2}(\tzz_{t}, \eps)\TT_{t}(0)^{2}
+p_{3}(\tzz_{t}, \eps)\TT_{t}(0)^{3}
\end{equation}
\begin{equation}
\TT(0)_{t+2}-\TT_{t}(0) = q_{0}(\tzz_{t}, \eps)+ q_{1}(\tzz_{t}, \eps)\TT_{t}(0)+q_{2}(\tzz_{t}, \eps)\TT_{t}(0)^{2}
+q_{3}(\tzz_{t}, \eps)\TT_{t}(0)^{3}
\end{equation}
Here the $p_{i}$ and $q_{i}$ are polynomials in $\tzz$, maximum $9$th order in
$\tzz$ and $6$th order in $\eps$. They can be computed
explicitly but we choose to omit the exact forms for now.

For fixed $\eps$, we can solve for the $\tzz$ two-step nullclines 
($\tzz_{t+2}-\tzz_{t} = 0$) and the $\TT$ nullclines ($\TT_{t+2}(0)-\TT_{t}(0) = 0$)
using Cardano's formula to solve for $\TT$ as a function of $\tzz$. In particular, each
nullcline equation has a solution that goes through $\tzz = 0$, $\TT(0) = 2$, independent of
$\eps$. This is the family of solutions that we will focus on.

Let $(\tzz, f_{\tzz, \eps}(\tzz))$ be the nullcline of $\tzz$, and let $(\tzz, f_{\TT, \eps}(\tzz))$
be the nullcline of $\TT(0)$. We will show that the $\TT$ values of the nullclines, as a function
of $\tzz$ and $\eps$, is differentiable around $\tzz = 0$, $\eps = 0$.

The nullclines are defined by the implicit equations
\begin{equation}
\begin{split}
0 & =  6\tzz^3 \eps-2 \TT \tzz-3 \TT \tzz^2 (\eps -1)-\TT \tzz^3 (\eps +2) (2 \eps +1)+\TT^2 \tzz+\frac{7}{2} \TT^2 \tzz^2 (\eps -1)\\
& +\frac{1}{2} \TT^2 \tzz^3 \left(9 \eps ^2-10 \eps +9\right)-\frac{1}{2} \TT^3 \tzz^2 (\eps -1)-\frac{1}{2} \TT^3 \tzz^3 \left(3 \eps ^2-4 \eps +3\right)+O(\tzz^{4})
\end{split}
\end{equation}
\begin{equation}
\begin{split}
0 & = -8 \tzz^2 \eps -12 \tzz^3 (\eps -1) \eps+4 \TT \tzz (\eps -1)+2 \TT \tzz^2 \left(3 \eps ^2-\eps +3\right)+4 \TT \tzz^3 (\eps -1) \left(\eps ^2+4 \eps +1\right)\\
& -2 \TT^2 \tzz (\eps -1)-\TT^2 \tzz^2 \left(7 \eps ^2-8 \eps +7\right)-\TT^2 \tzz^3 (\eps -1) \left(9 \eps ^2-\eps +9\right)+\TT^3 \tzz^2 \left(\eps ^2-\eps +1\right)\\
& +\TT^3 \tzz^3 (\eps -1) \left(3 \eps ^2-\eps +3\right)+O(\tzz^{4})
\end{split}
\end{equation}
We omit the higher order terms for now in anticipation of differentiating at $\tzz = 0$ to use
the implicit function theorem. Dividing by $\tzz$, we have the equations
\begin{equation}
\begin{split}
0 & =  6\tzz^2 \eps-2 \TT -3 \TT \tzz (\eps -1)-\TT \tzz^2 (\eps +2) (2 \eps +1)+\TT^2 +\frac{7}{2} \TT^2 \tzz (\eps -1)\\
& +\frac{1}{2} \TT^2 \tzz^2 \left(9 \eps ^2-10 \eps +9\right)-\frac{1}{2} \TT^3 \tzz (\eps -1)-\frac{1}{2} \TT^3 \tzz^2 \left(3 \eps ^2-4 \eps +3\right)+O(\tzz^{3})
\end{split}
\label{eq:dz2_null_approx}
\end{equation}
\begin{equation}
\begin{split}
0 & = -8 \tzz \eps -12 \tzz^2 (\eps -1) \eps+4 \TT  (\eps -1)+2 \TT \tzz \left(3 \eps ^2-\eps +3\right)+4 \TT \tzz^2 (\eps -1) \left(\eps ^2+4 \eps +1\right)\\
& -2 \TT^2 (\eps -1)-\TT^2 \tzz \left(7 \eps ^2-8 \eps +7\right)-\TT^2 \tzz^2 (\eps -1) \left(9 \eps ^2-\eps +9\right)+\TT^3 \tzz \left(\eps ^2-\eps +1\right)\\
& +\TT^3 \tzz^2 (\eps -1) \left(3 \eps ^2-\eps +3\right)+O(\tzz^{3})
\end{split}
\label{eq:dT2_null_approx}
\end{equation}
We immediately see that $\tzz = 0$, $\TT = 2$ solves both equations for all $\eps$.
Let $w(\eps, \tzz, \TT)$ and $v(\eps, \tzz, \TT)$ be the right hand sides of Equations
\ref{eq:dz2_null_approx} and \ref{eq:dT2_null_approx} respectively. We have
\begin{equation}
\left.\frac{\partial w}{\partial \TT}\right|_{(0, 0, 2)} = 2,~\left.\frac{\partial v}{\partial \TT}\right|_{(0, 0, 2)} = 4
\label{eq:implicit_deriv}
\end{equation}
In both cases the derivative is invertible. Therefore, $f_{\tzz, \eps}(\tzz)$ and $f_{\TT, \eps}(\tzz)$ are continuously differentiable in both $\tzz$ and $\eps$ in some neighborhood of
$0$. In fact, since $w$ and $v$ are analytic in all three arguments,
$f_{\tzz, \eps}(\tzz)$ and $f_{\TT, \eps}(\tzz)$ are analytic as well.

We can use the analyticity to solve for the low-order structure of the nullclines. One way
to compute the values of the derivatives is to define the nullclines as formal power series:
\begin{equation}
f_{\tzz}(\tzz) = 2+\sum_{j=1}^{\infty}\sum_{k=1}^{\infty}a_{j, k}\eps^{j}\tzz^{k}
\end{equation}
\begin{equation}
f_{\TT}(\tzz) = 2+\sum_{j=1}^{\infty}\sum_{k=1}^{\infty}b_{j, k}\eps^{j}\tzz^{k}
\end{equation}
We can then solve for the first few terms of the series using Equations \ref{eq:dz2_null_approx}
and \ref{eq:dT2_null_approx}. From this procedure, we have:
\begin{equation}
f_{\tzz,\eps}(\tzz) = 2+2\left(1-\eps\right)\tzz+2 \left(1-\eps+\eps ^2\right)\tzz^2+O(\tzz^{3})
\end{equation}
\begin{equation}
f_{\TT,\eps}(\tzz) = 2-\frac{\left( 2-3 \eps+2 \eps ^2\right)}{1-\eps}\tzz+\frac{1}{2} \left(4-\eps+4 \eps ^2\right)\tzz^2+O(\tzz^{3})
\end{equation}
The difference $f_{\Delta, \eps}(\tzz)$ between the two is:
\begin{equation}
f_{\Delta}(\tzz)\equiv f_{\tzz}(\tzz)  - f_{\TT}(\tzz)  = -\frac{\eps}{1-\eps}\tzz-\frac{3}{2}\eps\tzz^{2}+O(\tzz^{3})
\end{equation}
As $\eps$ decreases, for the low order terms
the distance between the nullclines also decreases.

We can show that the difference goes as $\eps$. The one-step dynamical equations for
$\eps = 0$ are
\begin{equation}
\tzz_{t+1}-\tzz_{t} = - \tzz_{t}\TT_{t}(0) + \frac{1}{2}\tzz_{t}^2\TT_{t}(0)
\end{equation}
\begin{equation}
\TT_{t+1}(0)-\TT_{t}(0) = -2\tzz_{t}\TT_{t}(0)+\zz_{t}^{2}\TT_{t}(0)
\end{equation}
Therefore, $\Delta\tzz = 2\Delta\TT$. This means that both the one step AND two-step
nullclines are identical. Since $f_{\tzz, 0}(\tzz) =  f_{\TT, 0}(\tzz)$, and both are
differentiable with respect to $\eps$, we have:
\begin{equation}
f_{\tzz, \eps}(\tzz) - f_{\TT, \eps}(\tzz) = \eps f_{\Delta, \eps}(\tzz)
\end{equation}
for some function $f_{\Delta, \eps}(\tzz)$ which is analytic in $\eps$ and $\tzz$
in a neighborhood around $(0, 0)$.

\subsection{Two-step dynamics of $y$}

\label{app:two_step_approx}

It is useful to define dynamical equations in coordinates $(\tzz, y)$ where $y$
is the difference between $\TT(0)$ and the $\tzz$ nullcline:
\begin{equation}
y \equiv \TT(0)-f_{\tzz,\eps}(\tzz)
\end{equation}
To lowest order in $\tzz$ and $\eps$ we have
\begin{equation}
y = \TT(0)- 2-2\left(1-\eps\right)\tzz-2 \left(1-\eps+\eps ^2\right)\tzz^2+O(\tzz^{3})
\label{eq:nullcline_small}
\end{equation}
We note that $y = 0$, at $\tzz = 0$ corresponds to $\TT(0) = 2$.
$y$ near but slightly less than $0$
is equivalent to edge-of-stability behavior. For positive $\tzz$, $y = 0$ implies
$\TT(0)>2$.

We can write the dynamics of $\tzz$ and $y$. The dynamics for $\tzz$ are given by:
\begin{equation}
\tzz_{t+2}-\tzz_{t} = p_{0}(\tzz_{t},\eps)+ p_{1}(\tzz_{t}, \eps)(y_{t}+f_{\tzz,\eps}(\tzz_{t}))+p_{2}(\tzz_{t}, \eps)(y_{t}+f_{\tzz,\eps}(\tzz_{t}))^{2}
+p_{3}(\tzz_{t}, \eps)(y_{t}+f_{\tzz,\eps}(\tzz_{t}))^{3}
\end{equation}
We know that the right hand side of this equation is analytic in $\tzz$,
$\eps$, and (trivially) $y$ as well.
By evaluating the multiple continuous derivatives of $f$, we can write:
\begin{equation}
\tzz_{t+2}-\tzz_{t} = 2y_t\tzz_t+y^2_{t}\tzz_{t}f_{1,\eps}(\tzz_{t}, y_{t})+y_{t}\tzz_{t}^{2}f_{2,\eps}(\tzz_{t})
\end{equation}
Here, $f_{1,\eps}$ and $f_{2,\eps}$ are analytic in $\tzz$, $\eps$, and $y$ in some neighborhood
around $0$.

This means that we have the bounds
\begin{equation}
|f_{1,\eps}(\tzz, y)| < F_{1},~|f_{2,\eps}(\tzz, y)| < F_{2}
\end{equation}
for $(\tzz, \eps, y)\in [-\tzz_{d},\tzz_{d}]\times[0,\eps_{d}]\times[-y_{d}, y_{d}]$
for some non-negative constants $F_{1}$ and $F_{2}$. Note that this bound is independent
of $\eps$.

Now we consider the dynamics of $y$. We have:
\begin{equation}
y_{t+2} - y_{t} = \TT_{t+2}(0)-\TT_{t}(0)-f_{\tzz,\eps}(\tzz_{t+2}) +f_{\tzz,\eps}(\tzz_{t})
\end{equation}
Since $\lim_{\tzz\to0,y\to0}\tzz_{t+2} = 0$, $f_{\tzz,\eps}(\tzz_{t+2})$ is analytic in some
neighborhood of $(0, 0, 0)$. Therefore $y_{t+2} - y_{t}$ is analytic as well. Substituting, we
have
\begin{equation}
\begin{split}
y_{t+2} - y_{t} & = q_{0}(\tzz_{t}, \eps)+ q_{1}(\tzz_{t}, \eps)[y+f_{\tzz,\eps}(\tzz)]+q_{2}(\tzz_{t}, \eps)[y+f_{\tzz,\eps}(\tzz)]^{2}
+q_{3}(\tzz_{t}, \eps)[y+f_{\tzz,\eps}(\tzz)]^{3}\\
& -f_{\tzz,\eps}(\tzz_{t}+2y_t\tzz_t+y^2_{t}\tzz_{t}f_{1,\eps}(\tzz_{t}, y_{t})+y_{t}\tzz_{t}^{2}f_{2,\eps}(\tzz_{t})) +f_{\tzz,\eps}(\tzz_{t})
\end{split}
\end{equation}

If we write $f_{\tzz,\eps}(\tzz) = f_{\TT,\eps}(\tzz)+\eps f_{\Delta, \eps}(\tzz)$, then
we can write:
\begin{equation}
\begin{split}
y_{t+2} - y_{t} & = q_{0}(\tzz_{t}, \eps)+ q_{1}(\tzz_{t}, \eps)[f_{\TT,\eps}(\tzz)]+q_{2}(\tzz_{t}, \eps)[f_{\TT,\eps}(\tzz)]^{2}+q_{3}(\tzz_{t}, \eps)[f_{\TT,\eps}(\tzz)]^{3}\\
&  2q_{2}(\tzz_{t}, \eps)[f_{\TT,\eps}(\tzz)(y+\eps f_{\Delta,\eps}(\tzz))]+3q_{3}(\tzz_{t}, \eps)[(f_{\TT,\eps}(\tzz))(y+\eps f_{\Delta,\eps}(\tzz))^{2}+(f_{\TT,\eps}(\tzz))^2(y+\eps f_{\Delta,\eps}(\tzz))]\\
&q_{0}(\tzz_{t}, \eps)+ q_{1}(\tzz_{t}, \eps)[y+\eps f_{\Delta,\eps}(\tzz)]+q_{2}(\tzz_{t}, \eps)[y+\eps f_{\Delta,\eps}(\tzz)]^{2}
+q_{3}(\tzz_{t}, \eps)[y+\eps f_{\Delta,\eps}(\tzz)]^{3}\\
& -f_{\tzz,\eps}(\tzz_{t}+2y_t\tzz_t+y^2_{t}\tzz_{t}f_{1,\eps}(\tzz_{t}, y_{t})+y_{t}\tzz_{t}^{2}f_{2,\eps}(\tzz_{t})) +f_{\tzz,\eps}(\tzz_{t})
\end{split}
\end{equation}
By the definition of the nullclines, the first four terms vanish. Once again using
the differentiability of the nullclines, as well as $f_{1,\eps}$ and $f_{2,\eps}$,
we can rewrite the dynamics in terms of the expansion:
\begin{equation}
y_{t+2}-y_{t} = -2(4-3\eps+4\eps^2)y_t \tzz_{t}^{2}-4\eps \tzz_{t}^{2}
+y_{t}^{2}\tzz_{t}^{2}g_{1,\eps}(\tzz_{t}, y_{t})+\eps\tzz_{t}^{3}g_{2,\eps}(\tzz_{t})
\end{equation}
Here $g_{1,\eps}$ and $g_{2,\eps}$ are analytic near zero in $\tzz$, $y$, and $\eps$.
We have the bounds
\begin{equation}
|g_{1,\eps}(\tzz, y)| < G_{1},~|g_{1,\eps}(\tzz, y)| < G_{2}
\end{equation}
for $(\tzz, \eps, y)\in [-\tzz_{d},\tzz_{d}]\times[0,\eps_{d}]\times[-y_{d}, y_{d}]$
for some non-negative constants $G_{1}$ and $G_{2}$. This bound is also independent of
$\eps$.

We can summarize these bounds in the following lemma:
\begin{lemma}
Define $y = \TT-f_{\tzz}(\tzz)$. The two step dynamics of $\tzz$ and $y$ are given by
\begin{equation}
\tzz_{t+2}-\tzz_{t} = 2y_t\tzz_t+y^2_{t}\tzz_{t}f_{1,\eps}(\tzz_{t}, y_{t})+y_{t}\tzz_{t}^{2}f_{2,\eps}(\tzz_{t})
\label{eq:dz2_HOT}
\end{equation}
\begin{equation}
y_{t+2}-y_{t} = -2(4-3\eps+4\eps^2)y_t \tzz_{t}^{2}-4\eps \tzz_{t}^{2}
+y_{t}^{2}\tzz_{t}^{2}g_{1,\eps}(\tzz_{t}, y_{t})+\eps\tzz_{t}^{3}g_{2,\eps}(\tzz_{t}, y_{t})
\label{eq:dy2_HOT}
\end{equation}
Where $f_{1,\eps}$, $f_{2,\eps}$, $g_{1,\eps}$, $g_{2,\eps}$ are all analytic in $\tzz$,
$y$, and $\eps$. Additionally, there exist positive $\tzz_{c}$, $y_{c}$, and $\eps_{c}$ such that
\begin{equation}
|f_{1,\eps}(\tzz, y)| < F_{1},~|f_{2,\eps}(\tzz, y)| < F_{2},~|g_{1,\eps}(\tzz, y)| < G_{1},~|g_{1,\eps}(\tzz, y)| < G_{2}
\end{equation}
for all $(\tzz, \eps, y)\in [-\tzz_{d},\tzz_{d}]\times[0,\eps_{d}]\times[-y_{d}, y_{d}]$,
where $F_{1}$, $F_{2}$, $G_{1}$, and $G_{2}$ are all non-negative constants.
\label{lem:zy_dyn_bounds}
\end{lemma}

We can use this Lemma to analyze the dynamics for small fixed $\eps$, for small initializations
of $\tzz$, $y$. The control of the higher order terms will allow for an analysis which focuses
on the effects of the lower order terms.

\subsection{Proof of Theorem \ref{thm:low_order_converge}}

\label{app:eps_proof}

Using Lemma \ref{lem:zy_dyn_bounds}, the dynamics in $\tzz$ and $y$ can be written as:
\begin{equation}
\tzz_{t+2}-\tzz_{t} = 2y_t\tzz_t+y^2_{t}\tzz_{t}f_{1,\eps}(\tzz_{t}, y_{t})+y_{t}\tzz_{t}^{2}f_{2,\eps}(\tzz_{t})
\end{equation}
\begin{equation}
y_{t+2}-y_{t} = -2(4-3\eps+4\eps^2)y_t \tzz_{t}^{2}-4\eps \tzz_{t}^{2}
+y_{t}^{2}\tzz_{t}^{2}g_{1,\eps}(\tzz_{t}, y_{t})+\eps\tzz_{t}^{3}g_{2,\eps}(\tzz_{t}, y_{t})
\end{equation}
Let $\eps <\eps_{d}$. Then we can use the bounds from Lemma \ref{lem:zy_dyn_bounds}
to control
the contributions of the higher order terms to the dynamics:
\begin{lemma}
\label{lem:hot_bounds}
Given constants $A>0$ and $B>0$,
there exist $\tzz_{c}$ and $y_{c}$ such that for $\tzz\in[0, 2\tzz_{c}]$, $y\in[-y_{c}, y_{c}]$,
we have the bounds: 
\begin{equation}
|y^2\tzz f_{1,\eps}(\tzz, y)+y\tzz^{2}f_{2,\eps}(\tzz)|\leq A|2y\tzz|
\end{equation}
\begin{equation}
|y^{2}\tzz^{2}g_{1,\eps}(\tzz, y)|\leq \frac{B}{8}|2(4-3\eps+4\eps^2)y \tzz^{2}|
\end{equation}
\begin{equation}
|\eps\tzz^{3}g_{2,\eps}(\tzz, y)|\leq \frac{B}{4}|4\eps \tzz^{2}|
\end{equation}
\end{lemma}
\begin{proof}
We begin by the following decomposition:
\begin{equation}
|y^2\tzz f_{1,\eps}(\tzz, y)+y\tzz^{2}f_{2,\eps}(\tzz)|\leq |y^2\tzz f_{1,\eps}(\tzz, y)|+|y\tzz^{2}f_{2,\eps}(\tzz)|
\end{equation}
From Lemma \ref{lem:zy_dyn_bounds}, there exists a region $[-\tzz_{d},\tzz_{d}]\times[0,\eps_{d}]\times[-y_{d}, y_{d}]$
where the magnitudes of $f_{1,\eps}$, $f_{2,\eps}$, $g_{1,\eps}$, and $g_{2,\eps}$ are bounded by
$F_{1}$, $F_{2}$, $G_{1}$, and $G_{2}$ respectively.
\begin{equation}
|y^2\tzz f_{1}(\tzz, y)+y\tzz^{2}f_{2}(\tzz)|\leq F_{1}y^2\tzz + F_{2}y\tzz^{2}
\end{equation}
\begin{equation}
|y^{2}\tzz^{2}g_{1}(\tzz, y)|\leq G_{1}y^{2}\tzz^{2}
\end{equation}
\begin{equation}
|\tzz^{3}g_{2}(\tzz, y)|\leq G_{2}\tzz^{3}
\end{equation}
Define $\tzz_{c}$ and $y_{c}$ as
\begin{equation}
y_{c} = \min(A/F_{1}, B/2G_{1}, y_{d}),~\tzz_{c} = \min(A/2F_{2}, B/2G_{2}, y_{c})
\end{equation}
The desired bounds follow immediately.
\end{proof}

We consider an initialization $(\tzz_{0}, y_{0})$ such that $\tzz_{0}\leq \tzz_{c}$ and
$y_{0}\leq y_{c}$, and $y_{0}\leq \tzz_{0}^{2}$.
Armed with Lemma \ref{lem:hot_bounds}, we can analyze the dynamics. There are two phases;
in the first phase, $\tzz$ is increasing, and $y$ is decreasing. The first phase ends when
$y$ becomes negative for the first time - reaching a value of $O(\eps)$. In the second phase,
$\tzz$ is decreasing, and $y$ stays negative and $O(\eps)$.

\subsubsection{Phase one}

Let $t_{sm}$ be the time such that
for $t\leq t_{sm}$, $\tzz_{t}\leq 2\tzz_{0}$. (We will later show that
$\tzz_{t}\leq2\tzz_{0}$ over the whole dynamics.)
For $t\leq t_{sm}$,
using Lemma \ref{lem:hot_bounds},
the change in $\tzz$ can be bounded from below by
\begin{equation}
\tzz_{t+2}-\tzz_{t} \geq 2y_{t}\tzz_{t}(1-A)
\label{eq:z_phase_one}
\end{equation}
Therefore at initialization, $\tzz$ is increasing. It remains increasing until $y_{t}$ becomes
negative, or $\tzz_{t}\geq 2\tzz_{0}$. We want to show that $y_{t}$ becomes negative before
$\tzz_{t}\geq 2\tzz_{0}$.

For any $t\leq t_{sm}$,
Lemma \ref{lem:hot_bounds} gives the following upper bound on 
$y_{t+2}-y_{t}$:
\begin{equation}
y_{t+2}-y_{t} \leq -(8-B)y_t \tzz_{t}^{2}-(4-B)\eps\tzz_{t}^{2}
\end{equation}
Let $t_{-}$ be the first time that $y_{t}$ becomes negative. Since $\tzz_{t}$
is increasing for $t\leq t_{-}$, we have
\begin{equation}
y_{t+2}-y_{t} \leq -(8-B)y_t \tzz_{0}^{2}-(4-B)\eps\tzz_{0}^{2}
\label{eq:y_bound_phase_1}
\end{equation}
This gives us the following bound on $y_{t}$:
\begin{equation}
y_{t} \leq y_{0}e^{-(8-B) \tzz_{0}^{2} t}
\label{eq:exp_bound_phase_1}
\end{equation}
valid for $t\leq t_{-}$ and $t\leq t_{sm}$.

We will now show that $t_{-}< t_{sm}$. Suppose that $t_{sm} \leq t_{-}$. Then at
$t_{sm}+2$, $\tzz_{t_{sm}+2}>2\tzz_{0}$ for the first time. Summing the bound in
Equation \ref{eq:z_phase_one}, we have:
\begin{equation}
\tzz_{t_{sm}+2}-\tzz_{0} \leq \sum_{t=0}^{t_{sm}}2y_{t}\tzz_{t}(1+A)\leq 4\tzz_{0}(1+A)\sum_{t=0}^{t_{sm}}y_{t}
\end{equation}
where the second bound comes from the definition of $t_{sm}$.
Using our bound on $y_{t}$, we have:
\begin{equation}
\tzz_{t_{sm}+2}-\tzz_{0} \leq 4\tzz_{0}(1+A)\sum_{s=0}^{t_{sm}}y_{0}e^{-(8-B) \tzz_{0}^{2} s} \leq \frac{(1+A)}{2}\frac{y_{0}}{\tzz_{0}}
\end{equation}
Since $y_{0}\leq \tzz_{0}^{2}$, $\tzz_{t_{sm}+2}\leq 2\tzz_{0}$. However, by assumption
$\tzz_{t_{sm}+2} >2\tzz_{0}$. We arrive at a contradiction; $t_{sm}$ is not less
than or equal to $t_{-}$.

There are three possibilities: the first is that $t_{-}$ is well-defined, and $t_{-}< t_{sm}$.
Another possibility is that $t_{-}$ is not well-defined - that is, $y_{t}$ never becomes negative.
In this case the bounds we derived are valid for all $t$.
Therefore using Equation \ref{eq:exp_bound_phase_1}, there exists some time $t_{\eps}$ where
$y_{t_{\eps}}< (4-B)\eps\tzz_{0}^{2}$. Then, using Equation \ref{eq:y_bound_phase_1}
we have $y_{t_{\eps}+2}<0$. Therefore, we conclude that $t_{-}$ is finite and less than
$t_{sm}$.

Since the well defined value $t_{-}<t_{sm}$,, when
$y$ first becomes negative, $\tzz_{t_{-}}\leq 2\tzz_{0}$. This means that we can
continue to apply the bounds from
Lemma \ref{lem:hot_bounds} at the start of the next phase. At $t = t_{-}-2$,
applying Lemma \ref{lem:hot_bounds} and $\tzz_{t_{-}}\leq 2\tzz_{0}$, we have
\begin{equation}
y_{ t_{-}}-y_{ t_{-}-2} \geq -4(8+B)y_{ t_{-}-2} \tzz_{0}^{2}-4(4+B)\eps\tzz_{0}^{2}
\end{equation}
which gives us $y_{t_{-}}\geq -4(4+B)\eps\tzz_{0}^{2}$. This concludes the first phase.
To
summarize we have
\begin{equation}
-4(4+B)\eps\tzz_{0}^{2}< y_{t_{-}}\leq 0,~\tzz_{t_{-}} \leq 2\tzz_{0}
\end{equation}

\subsubsection{Phase two}

Now consider the second phase of the dynamics. We will show that $y$ remains negative and
$O(\eps)$, and $\tzz$ decreases to $0$. While $y$ is negative, $\tzz$ decreases. While 
$y\geq-y_{0}$, from Lemma \ref{lem:hot_bounds} we have
\begin{equation}
\tzz_{t+2}-\tzz_{t} \leq (1-A)2y_{t}\tzz_{t}
\end{equation}
Therefore as long as $-y_{0}\leq y<0$, $\tzz_{t}$ is decreasing. If this is true
for all subsequent $t$, $\tzz_{0}$ will converge to $0$.

We will now show that $y$ remains negative and $O(\eps)$, concluding the proof. Let
$y^* = -\frac{\eps}{2-(3/2)\eps+2\eps^2}$. We can re-write the dynamical equation
for $y$ as
\begin{equation}
y_{t+2}-y_{t} = -2(4-3\eps+4\eps^2) \tzz_{t}^{2}(y_{t}-y^*)
+y_{t}^{2}\tzz_{t}^{2}g_{1}(\tzz_{t}, y_{t})+\tzz_{t}^{3}g_{2}(\tzz_{t}, y_{t})
\end{equation}
Applying Lemma \ref{lem:hot_bounds} to the higher order terms, we have:
\begin{equation}
y_{t+2}-y_{t} \leq -2(4-3\eps+4\eps^2) \tzz_{t}^{2}(y_{t}-y^*)
+B(|y_{t}|+\eps)\tzz_{t}^{2}
\label{eq:y_final_upper}
\end{equation}
\begin{equation}
y_{t+2}-y_{t} \geq -2(4-3\eps+4\eps^2) \tzz_{t}^{2}(y_{t}-y^*)
-B(|y_{t}|+\eps)\tzz_{t}^{2}
\label{eq:y_final_lower}
\end{equation}
These inequalities are valid as long as $|y_{t}|< y_{c}$.

At $t_{-}$, $y^{*}< y_{t}<0$. When $y^{*}<y_{t}<0$, then $|y_{t}|\leq |y^*|$.
Note that $\eps<2 |y^*|$. From Equation \ref{eq:y_final_upper}, we have
\begin{equation}
y_{t+2}-y_{t} \leq -2(4-3\eps+4\eps^2) \tzz_{t}^{2}(y_{t}-y^*)
+B(-y_{t}+\eps)\tzz_{t}^{2}
\end{equation}
From this inequality we can conclude that
\begin{equation}
y_{t+2} \leq (1-2(4-3\eps+4\eps^2) \tzz_{t}^{2}-B)y_{t}+\tzz_{t}^{2}[2(4-3\eps+4\eps^2) y^{*}
+B\eps]
\end{equation}
If $B<1$, then both terms are negative. We can conclude that if $y^{*}< y_{t}<0$,
$y_{t+2}<0$. In fact, from the last term we can conclude that $y_{t+2}<-4\eps\tzz_{t}^{2}$.

Now we must show that when $y^*< y_{t}< 0$, $y_{t+2}$ does not become too negative
(namely, smaller than $-y_{c}$).
Using Equation \ref{eq:y_final_lower}, we have:
\begin{equation}
y_{t+2}> y^*(1+3B\tzz_{0}^{2}) ~\text{~if~}y_{t}>y^*
\end{equation}
This means that if $y_{t}$ starts larger than $y^*$, it will be at most
$3B\tzz_{0}^{2} y^*$ below $y^*$ at the next step. Since $B<1$, $y_{t+2}>-y_{c}$
if $y^*<y_{t}<0$.

Finally, we will show that if $y^*(1+3B/(8-B) ) <y_{t}<y^*$, $y^*(1+3B/(8-B))< y_{t+2}< 0$.
Since $y_{t_{-}+2}$ fits this condition, 
we can conclude that $y_{t}$ is negative for all $t>t_{-}$, with magnitude
bounded from below by $y^*(1+3B/(8-B) )$, and complete the proof.

We will first show that $y^*(1+3B/(8-B) )< y_{t}$ implies that
$y^*(1+3B/(8-B) )< y_{t+2}$. Let $y_{t} = (1+\dl_{t})y^*$,
for $\dl_{t}< 3B/(8-B)$. We will show that $\dl_{t+2}< 3B/(8-B)$.
Using Equation \ref{eq:y_final_lower}, we have:
\begin{equation}
y_{t+2}\geq(1+\dl_{t})y^* -8\tzz^{2}_{t}\dl_{t} y^*-B\tzz_{t}^{2}(\eps-(1+\dl_{t})y^*)
\end{equation}
Substituting $y_{t+2} = (1+\dl_{t+2})y^*$, and dividing both sides by $y^*$ we have
\begin{equation}
\dl_{t+2}-\dl_{t}\leq -(8-B)\tzz^{2}_{t}\dl_{t} +3B\tzz_{t}^{2}
\end{equation}
If $\dl_{t} < 3B/(8-B)$, then we have $\dl_{t+2} < 3B/(8-B)$ as desired.

Finally, we will show that $0< \dl_{t}< 3B/(8-B)$ implies that $\dl_{t+2}> -1$ - that is,
$[1+3B/(8-B)]y^*<y_{t}<y^*$ implies $[1+3B/(8-B)]y^*<y_{t+2}<0$.
Equation \ref{eq:y_final_upper} implies
\begin{equation}
y_{t+2} \leq (1+\dl_{t})y^{*}-8\tzz^{2}_{t}\dl_{t} y^*+B\tzz_{t}^{2}(\eps-(1+\dl_{t})y^*)
\end{equation}
which gives us
\begin{equation}
\dl_{t+2}-\dl_{t}\geq -(8-B)\tzz^{2}_{t}\dl_{t} -3B\tzz_{t}^{2}
\end{equation}
If $\dl_{t} > 0$ implies
\begin{equation}
\dl_{t+2}> -3B\tzz_{t}^{2}
\end{equation}
If $3B\tzz_{0}^{2} <1$, then $\dl_{t+2} >-1$. This means that $y_{t+2}<0$
if $[1+3B/(8-B)]y^*<y_{t}<y^*$.

Finally, we make some choices of $B$ and $\tzz_{0}$ to guarantee convergence.
Choose $\tzz_{0}^{2} < 3/7$, and choose $B< \frac{1}{2}$. Then in summary, what we have
shown for phase two is:
\begin{itemize}
    \item At the start of the phase (time $t_{-}$), $y^*<y_{t_{-}}<0$.
    \item If $y^*<y_{t}<0$, $t>t_{-}$, $y^*(1+3B\tzz_{0}^{2})< y_{t+2}<-4\eps\tzz_{t}^{2}$.
    \item If $[1+3B/(8-B)]y^*<y_{t}<y^*$, $t>t_{-}$, $[1+3B/(8-B)]y^*<y_{t+2}<0$.
\end{itemize}
Through our choices of $\tzz_{0}$ and $B$, we know that $[1+3B/(8-B)]y^*< y^*(1+3B\tzz_{0}^{2})$.
Therefore, the entire trajectory for $t>t_{-}$ is accounted for by these regions, and
$[1+3B/(8-B)]y^*<y_{t}<0$ for all $t>t_{-}$. Additionally, we know that at least
once every $2$ steps, $y_{t} < -4\eps\tzz_{t}^{2}$. This means that the dynamics of
$\tzz_{t}$ can be bounded from above by
\begin{equation}
\tzz_{t+2}-\tzz_{t} \leq -2\eps^{2}\tzz_{t}^{4}
\end{equation}
From this we can conclude that $\tzz_{t}$ converges to $0$.

Therefore, for any positive
initialization with $\tzz_{0}\leq \tzz_{c}$, $y_{0}\leq y_{c}$, and $y_{0}\leq \tzz_{0}^{2}$,
we have:
\begin{equation}
\lim_{t\to\infty}\tzz_{t}\to 0,~\lim_{t\to\infty}y = -y_{f}
\end{equation}
where $y_{f} = O(\eps)$.

Now we can prove the statement of Theorem \ref{thm:two_step_ode_approx}.
Given a model with
$\eps\leq\eps_{c}$, there is a continuous mapping between $\th-\lr$ space and $\tzz-y$ space.
Since there is some neighborhood in $\tzz-y$ space that displays edge-of-stability behavior
($\TT_{t}(0)$ converging to within $O(\eps)$ of $2$),
the inverse image of that neighborhood is a neighborhood in $\th-\lr$ space that displays
edge-of-stability behavior. This concludes the proof.

\subsection{Low order dynamics}

\label{app:low_order_dyn}

In order to predict the final value of $y$, and understand the convergence to the fixed point,
We can study the low order dynamics in $\tzz$ and $y$. The low order
dynamical equations are:
\begin{equation}
\tzz_{t+2}-\tzz_{t} = 2y_t\tzz_t
\label{eq:dz2_no_HOT}
\end{equation}
\begin{equation}
y_{t+2}-y_{t} = -2(4-3\eps+4\eps^2)y_t\tzz_{t}^{2}-4\eps\tzz_t^2
\label{eq:dy2_no_HOT}
\end{equation}
For these reduced dynamics, we can show the following:

\begin{theorem}
\label{thm:low_order_converge}
For the dynamics defined by Equations \ref{eq:dz2_no_HOT} and \ref{eq:dy2_no_HOT},
for $\eps\ll 1$,
for positive inititializations $\tzz_{0}\ll1$, $y_{0}\ll 1$ with the additional constraints
$-\eps\log(\eps)\ll 16\tzz_{0}^{2}$ and $y_{0}<2\tzz_{0}^{2}$, we have
\begin{equation}
\lim_{t\to\infty}\tzz_{t} = 0,~\lim_{t\to\infty}y_{t} = -\eps/2+O(\eps^2)
\end{equation}
\end{theorem}

\begin{proof}
The proof distinguishes two phases in the time evolution:
\begin{itemize}
    \item \underline{Phase 1}: $\tzz$ starts positive and increases, $y$ starts positive and decreases.
    At the end of the phase we want $\tzz_{t}\leq 2\tzz_{0}$ and $y$ to be negative but bounded
    by $-16\tzz_{0}^{2}\eps$.
    \item \underline{Phase 2}: $\tzz$ decreases slowly, and $y$ settles to the fixed point
    (relatively) quickly, up to error $O(\eps^2)$.
\end{itemize}

Let $\eps \ll 1$. Consider an initialization $(\tzz_{0}, y_{0})$ where both
variables are positive, such that $\tzz_{0} \ll1$,
$\eps\log(\eps) \ll \tzz_{0}^{2}$, and $y_{0}\ll\tzz_{0}^{2}$.
From Equations \ref{eq:dz2_no_HOT} and \ref{eq:dy2_no_HOT}, we see that
the dynamics of $y$ will depend on the balance of the two terms.

Initially $\tzz$ increases and $y$ decreases. We analyze the dynamics of $y$ assuming
that $\tzz$ is fixed, and then compute the corrections. 

\underline{Phase 1}. At initialization, the first
term in the dynamics dominates, since by assumption $\eps\tzz_{t}^{2}\ll y_{t}\tzz_{t}^{2}\ll$.
Since $\tzz_{0}^{2}\ll 1$, $y$ initially decreases exponentially with decay rate bounded from above by $8\tzz_{0}^{2}$.
Therefore within $\log(-\eps/y_{0})/8\tzz_{0}^{2}$ steps, $y<\eps$.

At this point, the rate of change of $y$ is at least $-4\eps\tzz_{0}^{2}$. Therefore, in no more
than $1/4\tzz_{0}^{2}$ additional steps, $y$ becomes negative. Let $t_{-}$ be the first time
that $y$ becomes negative. We note that $y_{t_{-}} \geq -4\eps\tzz_{0}^{2}$ under this analysis -
the first term in Equation \ref{eq:dy2_no_HOT} is less than $y_{t}$ in magnitude, so the smallest
value that $y_{t+2}$ can take if $y_{t}$ is positive is $-4\eps\tzz_{0}^{2}$.

We can now understand the corrections due to the change in $\tzz$. We note that
$e^{-8\tzz_{0}^{2} t}$ is an upper bound for $y$ - since $\tzz$ is increasing, and
the $-4\eps\tzz_{t}^{2}$ decreases $y$ faster than exponential decay from the first term.
Since $\tzz$ is increasing,
$y_{t}\geq e^{-8\tzz_{0}^{2} t}$ as long as $y$ remains positive ($t < t_{-}$).
Let $t_{sm}$ be a time such that $\tzz_{t_{sm}} < 2\tzz_{0}$. We can bound
the change in $\tzz_{t}$ for $t<t_{sm}$. We know that $y_{t}\geq y_{0}e^{-8\tzz_{0}^2t}$.
The change in $\tzz$ can be bounded by
\begin{equation}
\tzz_{t_{sm}}-\tzz_{0}\leq \sum_{t=0}^{t_{sm}}2z_{t}y_{t} \leq 4\tzz_{0}\sum_{t = 0}^{t_{sm}}y_{t} \leq 4\tzz_{0}y_{0}\sum_{t = 0}^{t_{sm}}e^{-8\tzz_{0}^{2} t} \leq \frac{1}{2}\cdot\frac{y_{0}}{\tzz_{0}}\,.
\label{eq:z_sm_bound}
\end{equation}
If $y_{0} < 2\tzz_{0}^{2}$, then the bound holds independent of the value of $t_{sm}$,
as long as the bound on $y$ is correct. We know that the bound on $y$ is correct until time
$t_{-}$; therefore, $t_{sm}\geq t_{-}$.

\underline{Phase 2}.
This proves that there exists a time $t_{-}$, such that $\tzz_{t_{-}}\leq 2\tzz$, and
$-16\tzz_{0}^{2}\eps\leq y_{t_{-}}\leq 0$. Now that $y$ is negative, it will stay negative,
and $\tzz$ will decrease until it reaches $0$. In order to understand the dynamics, we
will use a change of coordinates. Consider solving Equation \ref{eq:dy2_no_HOT} for
$y_{t+2}-y_{t} = 0$ for $\tzz_{t}\neq 0$. We have
\begin{equation}
y^* = -\frac{\eps}{2-3/2\eps+2\eps^2}
\end{equation}
Consider now the coordinate $\dl_{t}$ defined by the equation
\begin{equation}
y_{t} = -(1+\dl_{t})\frac{\eps}{2-3/2\eps+2\eps^2}
\end{equation}
The dynamics of $\dl_{t}$ are given by
\begin{equation}
\dl_{t+2} = (1-2(4-3\eps+4\eps^2)\tzz_{t}^{2})\dl_{t}
\end{equation}
Since $\tzz_{t}\ll1$, $\dl_{t}$ is strictly decreasing in magnitude.
We can bound $\dl_{t}$ from above by
\begin{equation}
|\dl_{t}|\leq \exp\left(-8\sum_{s=t_{-}}^{t}\tzz_{s}^{2}\right)|\dl_{t_{-}}|
\end{equation}
Since $\dl$ starts negative, and is decreasing in magnitude, we know that $y_{t} > -\frac{\eps}{2-3/2\eps+2\eps^2}$.
This means that we can bound $\tzz_{t}$ by 
\begin{equation}
\tzz_{t} \geq 2e^{-\eps t}\tzz_{0}
\end{equation}
Substitution gives us the following bound on $\dl_{t}$:
\begin{equation}
|\dl_{t}|\leq \exp\left(-8\sum_{s=t_{-}}^{t}4e^{-2\eps s}\tzz_{0}^{2}\right)|\dl_{t_{-}}|
\end{equation}
Using the integral approximation for the sum, the bound becomes
\begin{equation}
|\dl_{t}|\leq \exp\left(-32\tzz_{0}^{2}\int_{0}^{t}e^{-2\eps s}ds\right)\dl_{t_{-}} = \exp\left(-16\tzz_{0}^{2}/\eps(1-e^{-2\eps t})\right)|\dl_{t_{-}}|
\end{equation}
From our previous analysis, we know that $-1\leq\dl_{t-}\leq 0$. In the limit of large $t$
we have
\begin{equation}
\lim_{t\to\infty}|\dl_{t}|\leq \exp\left(-16\tzz_{0}^{2}/\eps\right)|\dl_{t_{-}}|
\end{equation}
If we have the condition
\begin{equation}
16\tzz_{0}^{2}/\eps\geq -\log(\eps)
\end{equation}
then $\lim_{t\to\infty}|\dl_{t}|\leq \eps^{2}$.

If we want $\lim_{t\to\infty}y_{t} = -\eps/2+O(\eps^2)$, then we need the condition
\begin{equation}
16\tzz_{0}^{2}\geq -\eps\log(\eps)
\end{equation}
or equivalently $-\eps\log(\eps)<16\tzz_{0}^{2}$. Under these conditions, 
$\lim_{t\to\infty}\tzz_{t} = 0$ and
$\lim_{t\to\infty}y_{t} = -\eps/2+O(\eps^2)$.
\end{proof}

This result can be confirmed numerically by running the dynamical equations from a variety of
initializations, computing the median eigenvalue (restricted to the range $[1.9, 2.0]$), and
plotting versus $\eps$ (Figure \ref{fig:y_eps_dep}).We note that since the dynamics is slow, the ODE
given by
\begin{equation}
\dot{\tzz} = 2y\tzz
\end{equation}
\begin{equation}
\dot{y} = -2(4-3\eps+4\eps^2)y\tzz^{2}-4\eps\tzz^2
\end{equation}
also obtains the same limit (Figure \ref{fig:y_eps_dep}). The ODE suggests that the concentration
relies on both the equal-orders in $\tzz$ of the $y^{0}$ and $y^{1}$ terms, as well as
a separation of timescales - $\tzz$ converges to $0$ at a rate of $\eps$, while
$y$ converges to the fixed point at a rate $\tzz_{t}^{2}$.
In both cases, the deviation from $-\eps/2$
scales as $O(\eps^2)$ (Figure \ref{fig:y_eps_dep}, right).

\begin{figure}[h]
    \centering
    \begin{tabular}{cc}
    \includegraphics[width=0.42\linewidth]{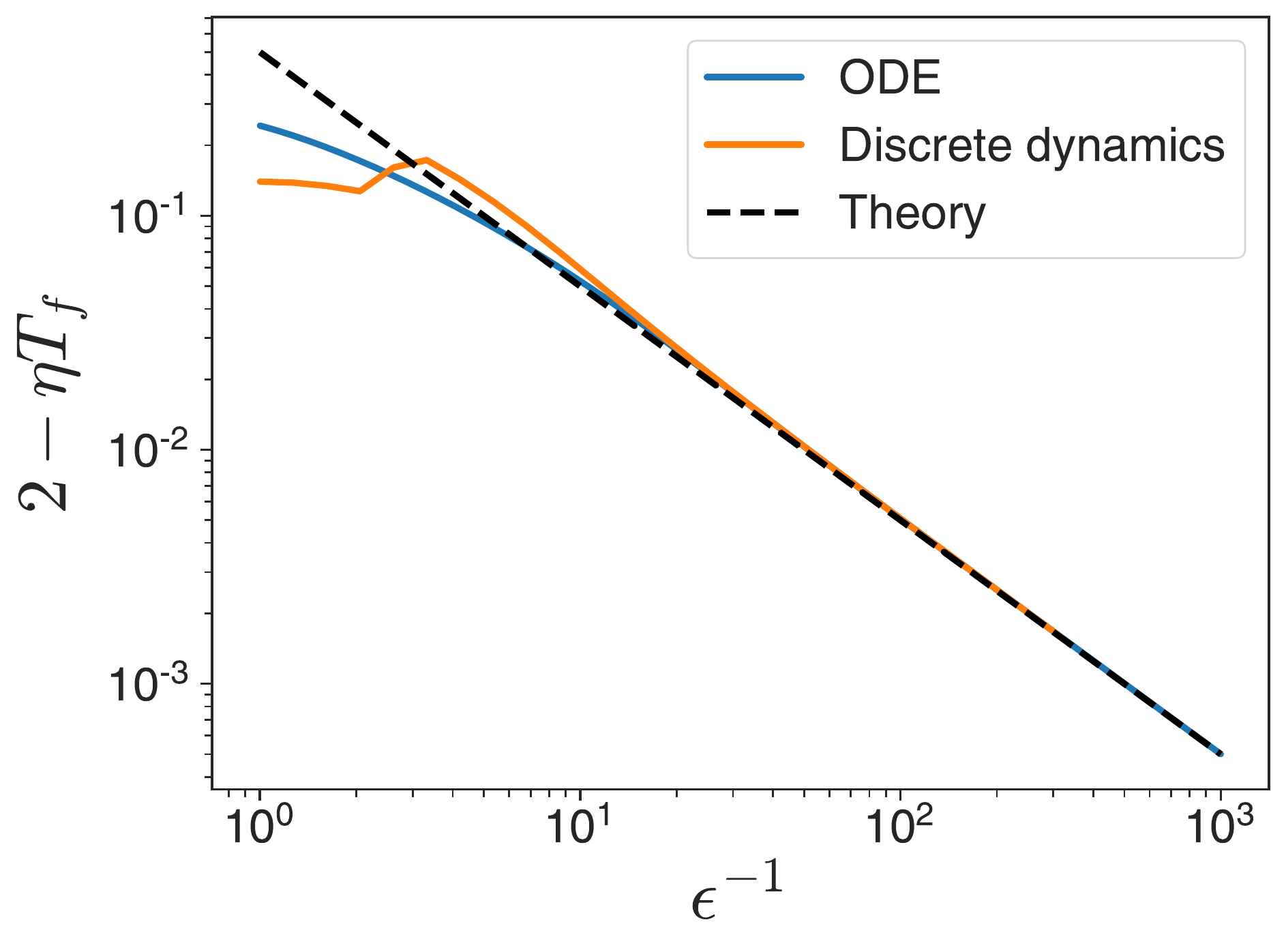} & \includegraphics[width=0.42\linewidth]{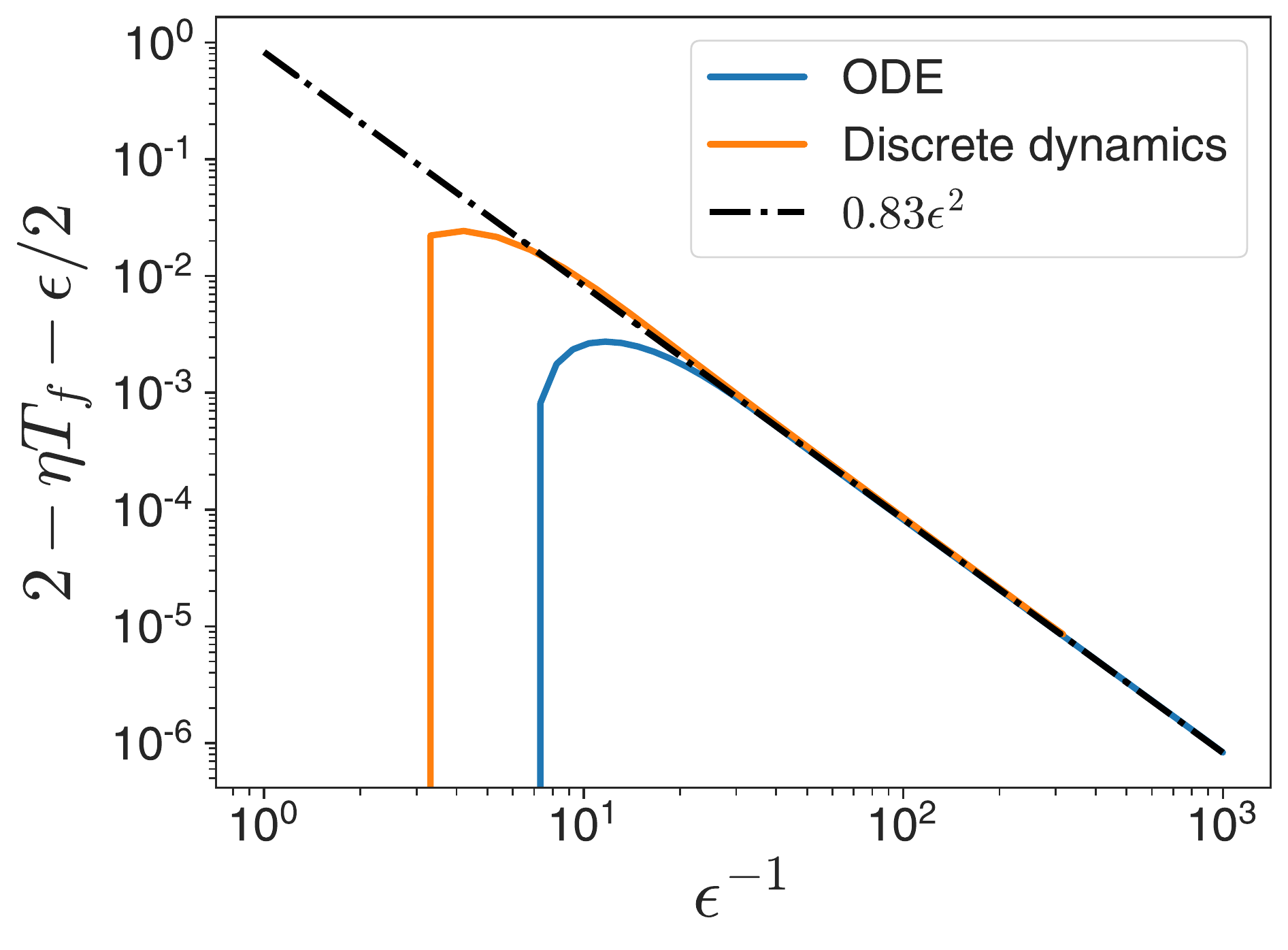}
    \end{tabular}
    \caption{Final values of $y$, normalized deviation from critical value $\TT(0) = 2$, for
    discrete dynamics and ODE approximation. Deviation is well approximated by $\eps/2$ over
    a large range (left). Deviations from $\eps/2$ are $O(\eps^2)$ (right).}
    \label{fig:y_eps_dep}
\end{figure}

\section{Quadratic regression model dynamics}

We use Einstein summation notation in this section - repeated indices on the right-hand-side
of equations are considered to be summed over, unless they show up on the left-hand-side.

\subsection{Proof of Theorem \ref{thm:ave_curv_deriv}}

\label{app:gf_dynamics}

Let $\z$, $\J$, and $\Q$ be initialized with i.i.d. random elements with $0$ mean
and variance $\sigma_{z}^{2}$, $\sigma_{J}^{2}$, and $1$ respectively. Furthermore,
Let the distributions be invariant to rotations in both data space and parameter
space, and have finite $4$th moment.

In order to understand the development of the curvature at early times, we consider coordinates
which convert  $\J$ into its singular value form.
In these coordinates, we can write:
\begin{equation}
\J_{\al i} = \begin{cases}
0&\text{~if~}\al\neq i\\
\sigma_{\al}&\text{~if~}\al = i
\end{cases}
\end{equation}
The singular values $\sigma_{\al}$ are the square roots of the singular values of the NTK matrix. We assume that
they are ordered from largest ($\sigma_{1}$) to smallest in magnitude.
By assumption, under this rotation the statistics of $\z$ and $\Q$ are left unchanged.

The time derivatives at $t = 0$ can be computed directly in the singular
value coordinates.
The first derivative is given by
\begin{equation}
\frac{d}{dt}\sigma^{2}_{\al} = 2\sigma_{\al}\dot{\sigma}_{\al}
\end{equation}
 Using the
diagonal coordinate system, we have
\begin{equation}
\expect\left[\frac{d}{dt}\sigma^{2}_{\al}\right] = \expect[\Q_{\al\bt j}\J_{\bt j}\z_{\bt}] = 0
\end{equation}

However, the average second derivative is positive. Calculating, we have:
\begin{equation}
\frac{d^2}{dt^2}\sigma^{2}_{\al} = 2(\dot{\sigma}_{\al}^2+\sigma_{\al}\ddot{\sigma}_{\al})
\end{equation}
We can compute the average at initialization. We have:
\begin{equation}
\expect[\dot{\sigma}^{2}_{\al}] = \expect[\Q_{\al\bt j}\J_{\bt j}\z_{\bt}\Q_{\al \dl k}\J_{\dl k}\z_{\dl}] = \expect[\dl_{\bt\dl}\dl_{jk}\J_{\bt j}\J_{\dl k}\z_{\bt}\z_{\dl}]
\end{equation}
\begin{equation}
\expect[\dot{\sigma}^{2}_{\al}] = \expect[\Q_{\al\bt j}\J_{\bt j}\z_{\bt}\Q_{\al \dl k}\J_{\dl k}\z_{\dl}] = \sum_{j}\expect[\J_{\bt j}^2\z_{\bt}^2] = \D\P\sigma_{J}^{2}\sigma_{z}^{2}
\end{equation}
To compute the second term, we compute $\ddot{\J}_{\al i}$:
\begin{equation}
\ddot{\J}_{\al i} = -\Q_{\al i j}(\J_{\bt j}\dot{\z}_{\bt}+\dot{\J}_{\bt j}\z_{\bt})
\end{equation}
Expanding, we have:
\begin{equation}
\ddot{\J}_{\al i} = \Q_{\al i j}(\J_{\bt j}\J_{\bt k}\J_{\dl k}\z_{\dl}+\Q_{\bt j k}\J_{\dl k}\z_{\dl}\z_{\bt})
\end{equation}
In the diagonal coordinates $\J_{\al\al} = \sigma_{\al}$. This gives us:
\begin{equation}
\expect[\sigma_{\al}\ddot{\sigma}_{\al}] = \expect[\sigma_{\al}\Q_{\al \al j} \Q_{\bt j k}\J_{\dl k}\z_{\dl}\z_{\bt}]
\end{equation}
Averaging over the $\Q$, we get:
\begin{equation}
\expect[\sigma_{\al}\ddot{\sigma}_{\al}] = \P\expect[\sigma_{\al}\delta_{\al \bt}\delta_{\al k}\J_{\dl k}\z_{\dl}\z_{\bt}] = \expect[\sigma_{\al}\z_{\al}\z_{\dl}\J_{\dl\al}]
\end{equation}
Which evaluates to:
\begin{equation}
\expect[\sigma_{\al}\ddot{\sigma}_{\al}] = \sigma_{z}^{2}\P\expect[\sigma_{\al}^2]
\end{equation}
In the limit of large $\D$ and $\P$, for fixed ratio $\D/\P$,
the statistics of the Marchenko-Pastur distribution allow us to compute
the derivative of the largest eigenmode as
\begin{equation}
\expect[\sigma_{0}\ddot{\sigma}_{0}] = \sigma_{z}^{2}\sigma_{J}^{2}\P^2\D(1+\sqrt{\D/\P})^2
\end{equation}
Taken together, this gives us
\begin{equation}
\expect\left[\frac{d^{2}\lambda_{max}}{dt^2}\right] = \sigma_{z}^{2}\sigma_{J}^{2}\D\P(\P(1+\sqrt{\D/\P})^2+1)
\end{equation}
We confirm the prediction numerically in Figure \ref{fig:ave_second_deriv}.

That is, the second derivative of the maximum curvature is positive on average. If we normalize
with respect to the eigenvalue scale, in the limit of large $\D$ and $\P$ we have:
\begin{equation}
\expect\left[\frac{d^{2}\lambda_{max}}{dt^2}\right]/\expect[\lambda_{max}] = \sigma_{z}^{2}
\end{equation}
Therefore, increasing $\sigma_{z}$ increases the relative curvature of the $\lambda_{max}$ trajectory. This gives us the proof of Theorem \ref{thm:ave_curv_deriv}.

\begin{figure}[h]
    \centering
    \begin{tabular}{cc}
    \includegraphics[width=0.4\linewidth]{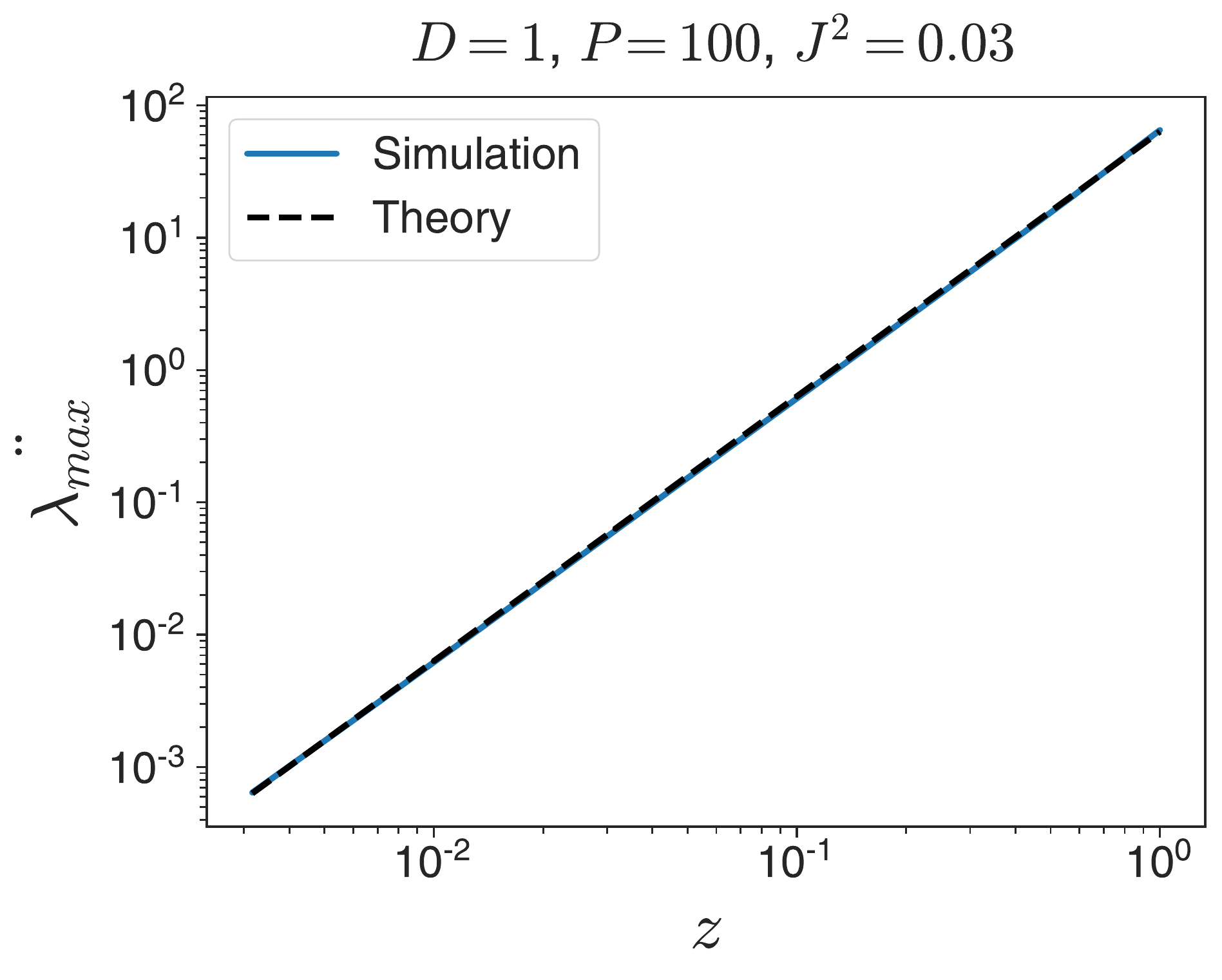} &
    \includegraphics[width=0.4\linewidth]{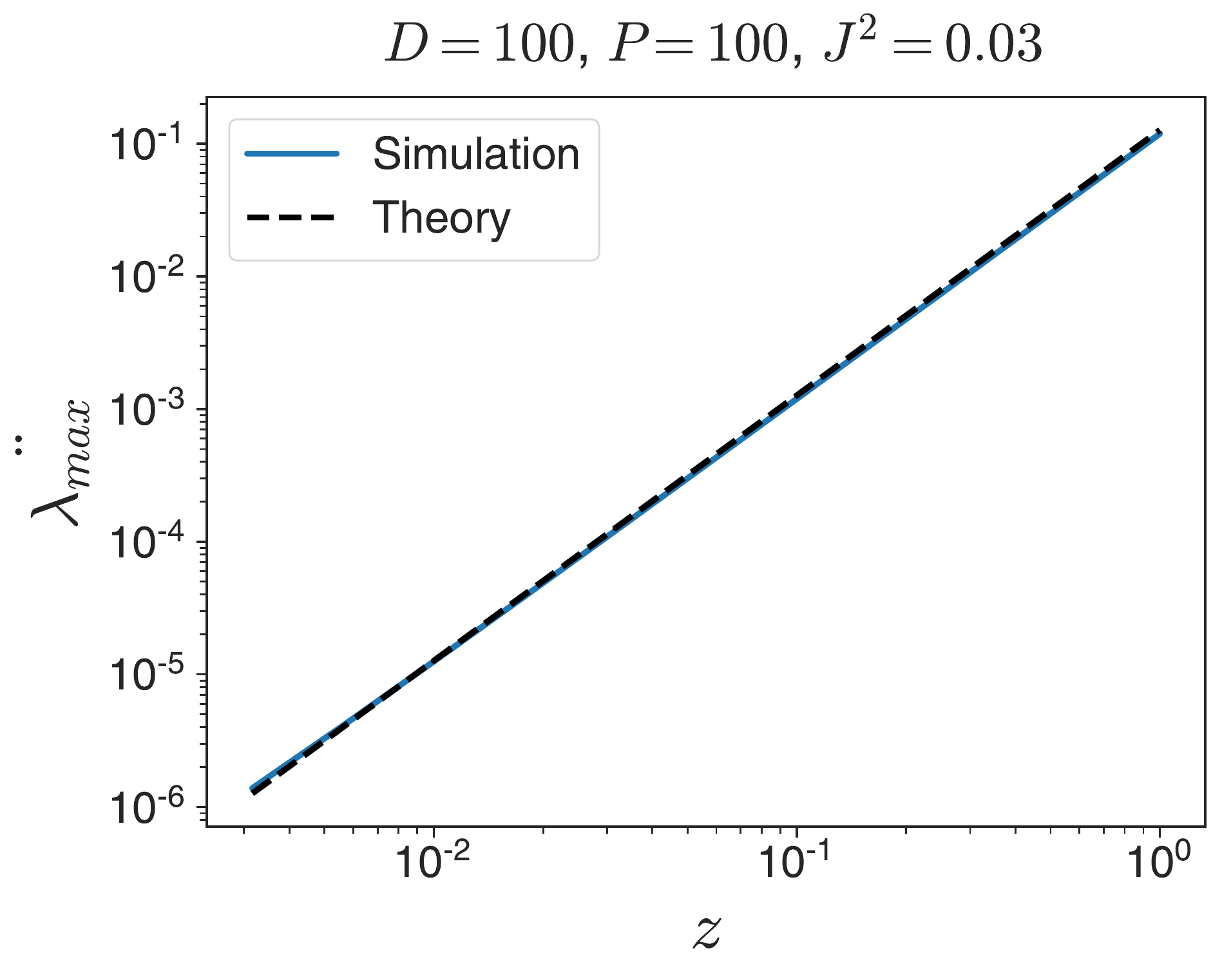}\\
    \includegraphics[width=0.4\linewidth]{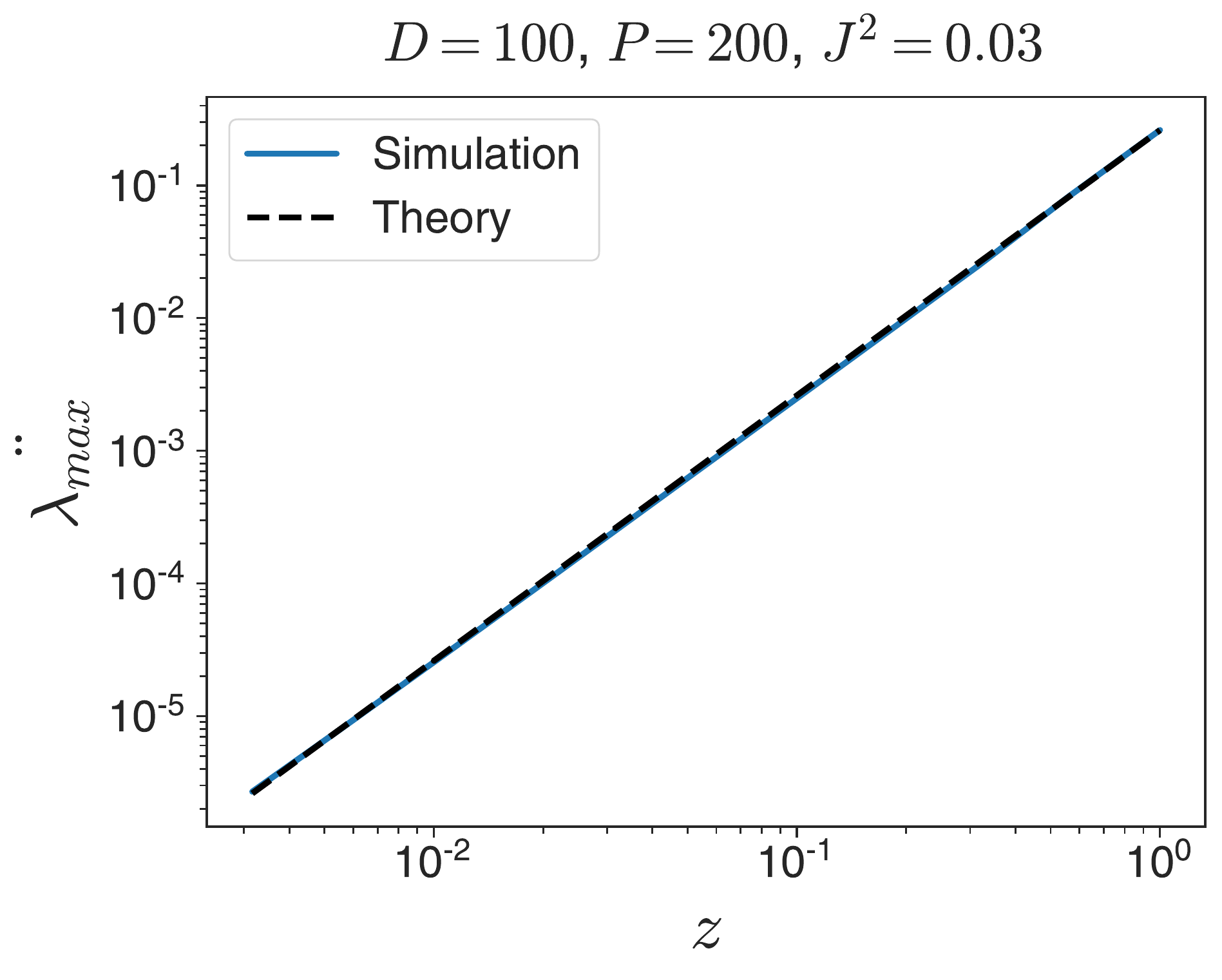} &
    \includegraphics[width=0.4\linewidth]{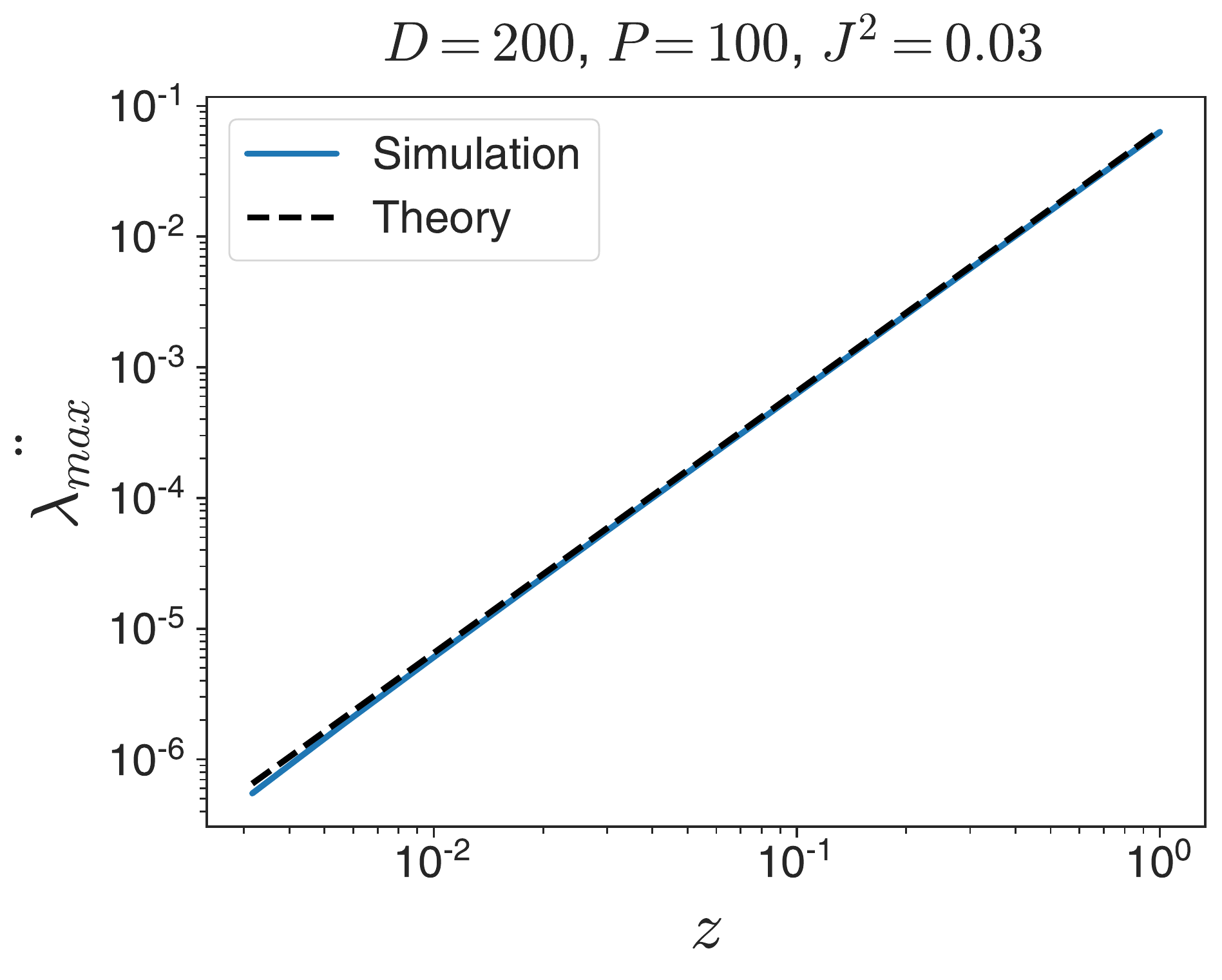}
    \end{tabular}
    \caption{Average $\ddot{\lam}_{\max}(0)$ versus $\sgz$, various $\D$ and $\P$
    (100 seeds).}
    \label{fig:ave_second_deriv}
\end{figure}

This result suggests that as $\sgz$ increases, so does the degree
of progressive sharpening. This can be confirmed by looking at GF trajectories
(Figure \ref{fig:gf_traj}). The trajectories with small $\sgz$ don't change their curvature
much, and the loss decays exponentially at some rate. However, when $\sgz$ is larger,
the curvature increases initially, and then stabilizes to a higher value, allowing for
faster convergence to the minimum of the loss.

\begin{figure}[h]
    \centering
    \begin{tabular}{cc}
    \includegraphics[width=0.48\linewidth]{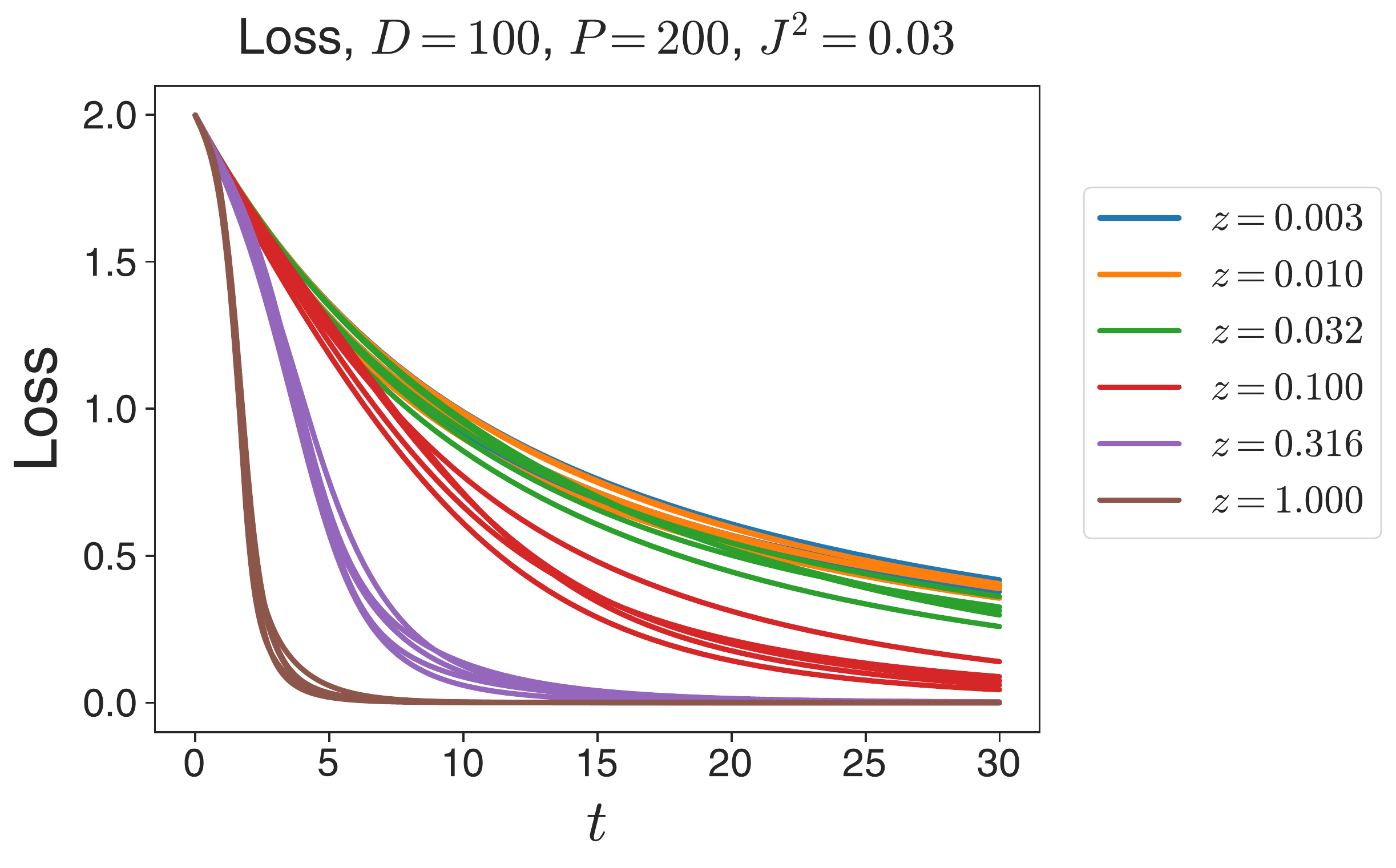} & \includegraphics[width=0.48\linewidth]{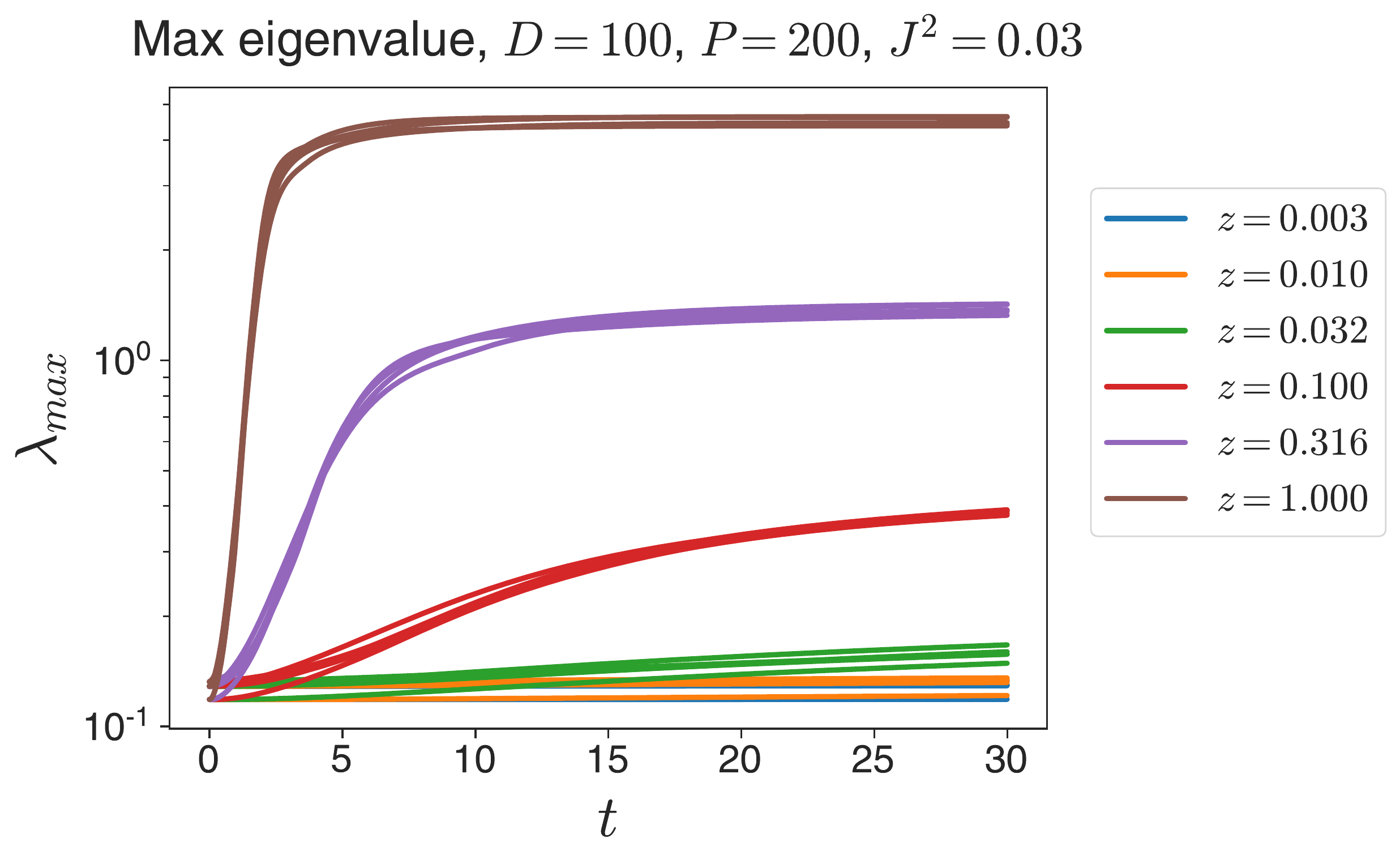}
    \end{tabular}
    \caption{Gradient flow trajectories of loss and max NTK eigenvalues for quadratic regression models 
    for varying $\sgz$. As $\sgz$ increases, $\lam_{\max}$ changes more quickly, and is
    generally increasing. Models with higher $\sgz$ converge faster in GF dynamics.}
    \label{fig:gf_traj}
\end{figure}

\subsection{Timescales for gradient descent}

\label{app:gd_timescales}

Consider a random initialization of $\z$, $\J$, and $\Q$, where the terms are i.i.d. with zero
mean variances $\sgz^2$, $\sgJ^2$, and $1$ respectively, and finite fourth moments.
Furthermore, suppose that
$\z$, $\J$, and $\Q$ are rotationally invariant in both input and output space. Under
these conditions, we hope to compute
\begin{equation}
r_{NL}^{2} \equiv \frac{\expect[|| \frac{1}{2}\lr^2 \Q_{\al i j}(\J_{\bt i})_{0}(\z_{\bt})_{0}(\J_{\dl j})_{0}(\z_{\dl})_{0}||_{2}^{2}]}{\expect[||\lr(\J_{\al i})_{0} (\J_{i\bt})_{0}(\z_{\bt})_{0} ||_{2}^{2}]} = \frac{1}{4}\lr^2\sgz^2\D^2
\end{equation}
at initialization, in the limit of large $\D$ and $\P$.

The denominator is given by:
\begin{equation}
\expect[\J_{\al i}\J_{\bt i}(\z_{\bt})\J_{\al j}\J_{\dl j}(\z_{\dl})] = \sgz^2\expect[\J_{\al i}\J_{\bt i}\J_{\al j}\J_{\dl j}\dl_{\bt\dl}] = \sgz^2\expect[\J_{\al i}\J_{\bt i}\J_{\al j}\J_{\bt j}]
\end{equation}
Evaluation gives us:
\begin{equation}
\expect[\J_{\al i}\J_{\bt i}(\z_{\bt})\J_{\al j}\J_{\dl j}(\z_{\dl})] = \sgz^2(\sgJ^4(\P(\P-1)\D)+C_{4}\D\P)
\end{equation}
where $C_{4}$ is the $4$th moment of $\J_{\al i}$. To lowest order in $\D$ and $\P$
\begin{equation}
\expect[\J_{\al i}\J_{\bt i}(\z_{\bt})\J_{\al j}\J_{\dl j}(\z_{\dl})] = \sgz^2\sgJ^4\D\P^{2}+O(\D\P)
\end{equation}

Evaluating the numerator, we have:
\begin{equation}
\expect[\Q_{\al i j}\J_{\bt i}\z_{\bt}\J_{\dl j}\z_{\dl}\Q_{\al mn}\J_{\gm m}\z_{\gm}\J_{\nu n}\z_{\nu}] = \expect[\J_{\bt i}\z_{\bt}\J_{\dl j}\z_{\dl}\J_{\gm m}\z_{\gm}\J_{\nu n}\z_{\nu}](\delta_{im}\delta_{jn}+(M_{4}-1)\delta_{ijmn})
\end{equation}
where $M_{4}$ is the $4th$ moment of $\Q_{\al i j}$. This gives us:
\begin{equation}
\begin{split}
\frac{1}{\D}\expect[\Q_{\al i j}\J_{\bt i}\z_{\bt}\J_{\dl j}\z_{\dl}\Q_{\al mn}\J_{\gm m}\z_{\gm}\J_{\nu n}\z_{\nu}] & = \expect[\J_{\bt i}\z_{\bt}\J_{\dl j}\z_{\dl}\J_{\gm i}\z_{\gm}\J_{\nu j}\z_{\nu}]+\\
&(M_{4}-1)\expect[\J_{\bt i}\z_{\bt}\J_{\dl i}\z_{\dl}\J_{\gm i}\z_{\gm}\J_{\nu i}\z_{\nu}]
\end{split}
\end{equation}
Next, we perform the $\z$ averages. We have
\begin{equation}
\begin{split}
\frac{1}{\D}\expect[\Q_{\al i j}\J_{\bt i}\z_{\bt}\J_{\dl j}\z_{\dl}\Q_{\al mn}\J_{\gm m}\z_{\gm}\J_{\nu n}\z_{\nu}] & = \sigma_{z}^{4}\expect[\J_{\bt i}\J_{\dl j}\J_{\gm i}\J_{\nu j}](\delta_{\bt\dl}\delta_{\gm\nu}+\delta_{\bt\gm}\delta_{\dl\nu}+\delta_{\bt\nu}\delta_{\dl\gm})\\
&+(C_{4}-\sigma^{4})\expect[\J_{\bt i}\J_{\dl j}\J_{\gm i}\J_{\nu j}]\delta_{\bt\dl\gm\nu}\\
& +(M_{4}-1)\sigma_{z}^{4}\expect[\J_{\bt i}\J_{\dl i}\J_{\gm i}\J_{\nu i}](\delta_{\bt\dl}\delta_{\gm\nu}+\delta_{\bt\gm}\delta_{\dl\nu}+\delta_{\bt\nu}\delta_{\dl\gm})\\
& +(M_{4}-1)(C_{4}-\sigma^{4})\expect[\J_{\bt i}\J_{\dl i}\J_{\gm i}\J_{\nu i}]\delta_{\bt\dl\gm\nu}
\end{split}
\end{equation}
where $C_{4}$ is the $4$th moment of $\z$. Simplification gives us:
\begin{equation}
\begin{split}
\frac{1}{\D}\expect[\Q_{\al i j}\J_{\bt i}\z_{\bt}\J_{\dl j}\z_{\dl}\Q_{\al mn}\J_{\gm m}\z_{\gm}\J_{\nu n}\z_{\nu}] & = \sigma_{z}^{4}(\expect[\J_{\bt i}\J_{\bt j}\J_{\dl i}\J_{\dl j}]+\expect[\J_{\bt i}\J_{\dl j}\J_{\bt i}\J_{\dl j}]+\expect[\J_{\bt i}\J_{\dl j}\J_{\dl i}\J_{\bt j}])\\
&+(C_{4}-\sigma_{z}^{4})\expect[\J_{\bt i}\J_{\bt j}\J_{\bt i}\J_{\bt j}]\\
& +(M_{4}-1)\sigma_{z}^{4}(\expect[\J_{\bt i}\J_{\bt i}\J_{\gm i}\J_{\gm i}]+\expect[\J_{\bt i}\J_{\dl i}\J_{\bt i}\J_{\dl i}]+\expect[\J_{\bt i}\J_{\dl i}\J_{\dl i}\J_{\bt i}])\\
& +(M_{4}-1)(C_{4}-\sigma^{4})\expect[\J_{\bt i}\J_{\bt i}\J_{\bt i}\J_{\bt i}]
\end{split}
\end{equation}
For large $\D$ and $\P$, the final three terms are asymptotically smaller than the first term. Evaluating the first
term, to leading order we have:
\begin{equation}
\frac{1}{\D}\expect[\Q_{\al i j}\J_{\bt i}\z_{\bt}\J_{\dl j}\z_{\dl}\Q_{\al mn}\J_{\gm m}\z_{\gm}\J_{\nu n}\z_{\nu}] =  \sigma_{z}^{4}\sigma_{J}^{4}(2\D\P^{2}+
2\D^2\P+\D^2\P^2) +O(\D^2\P+\D\P^2)
\end{equation}
\begin{equation}
\expect[\Q_{\al i j}\J_{\bt i}\z_{\bt}\J_{\dl j}\z_{\dl}\Q_{\al mn}\J_{\gm m}\z_{\gm}\J_{\nu n}\z_{\nu}] =  \sigma_{z}^{4}\sigma_{J}^{4}\D^3\P^2+O(\D^3\P+\D^2\P^2)
\end{equation}

This gives us:
\begin{equation}
r_{NL}^2 = \frac{1}{4}\frac{\sigma_{z}^{4}\sigma_{J}^{4}\D^3\P^2}{\sgz^2\sgJ^4\D\P^{2}} = \frac{1}{4}\sgz^2\D^2
\end{equation}
to leading order, in the limit of large $\D$ and $\P$.

\section{Analysis of real models}

\subsection{Dynamics of $y$ in CIFAR10 model}

\label{app:y_dyn}

The dynamics of $y$ in the CIFAR10 model analyzed in Section \ref{sec:real_world_model} are
more complicated than the $\zz_{1}$ dynamics. We see from Figure \ref{fig:cifar_eos} that
there is a $\zz_{1}$ and $y$-independent component of the two-step change in $y$. We can approximate
this change $b$ by computing the average value of $y_{t+2}-y_{t}$ for small $\zz_{1}$
(taking $\zz_{1}< 10^{-4}$ in this case). We can then subtract off $b$ from $y_{t+2}-y_{t}$,
and plot the remainder against $\zz_{t}^{2}$ (Figure \ref{fig:dy_approx} left). We see that
$y_{t+2}-y_{t}-b$ is negatively correlated with $\zz_{t}^{2}$, particularly for
large $\zz_{t}$. However, $y_{t+2}-y_{t}$ is clearly not simply function of $\zz_{1}$.

The two-step model dynamics could be written as $(ay+c)\tzz^{2}$. If we plot
$(y_{t+2}-y_{t}-b)/\zz_{1}^{2}$ versus $y_{t}$, we again don't have a single-valued function
(Figure \ref{fig:dy_approx}, right). Therefore, the functional form of $y_{t+2}-y_{t}$
is not given by $b+ay\zz_{1}^{2}+c\zz_{1}^{2}$.

\begin{figure}[h]
    \centering
    \begin{tabular}{cc}
    \includegraphics[height=0.28\linewidth]{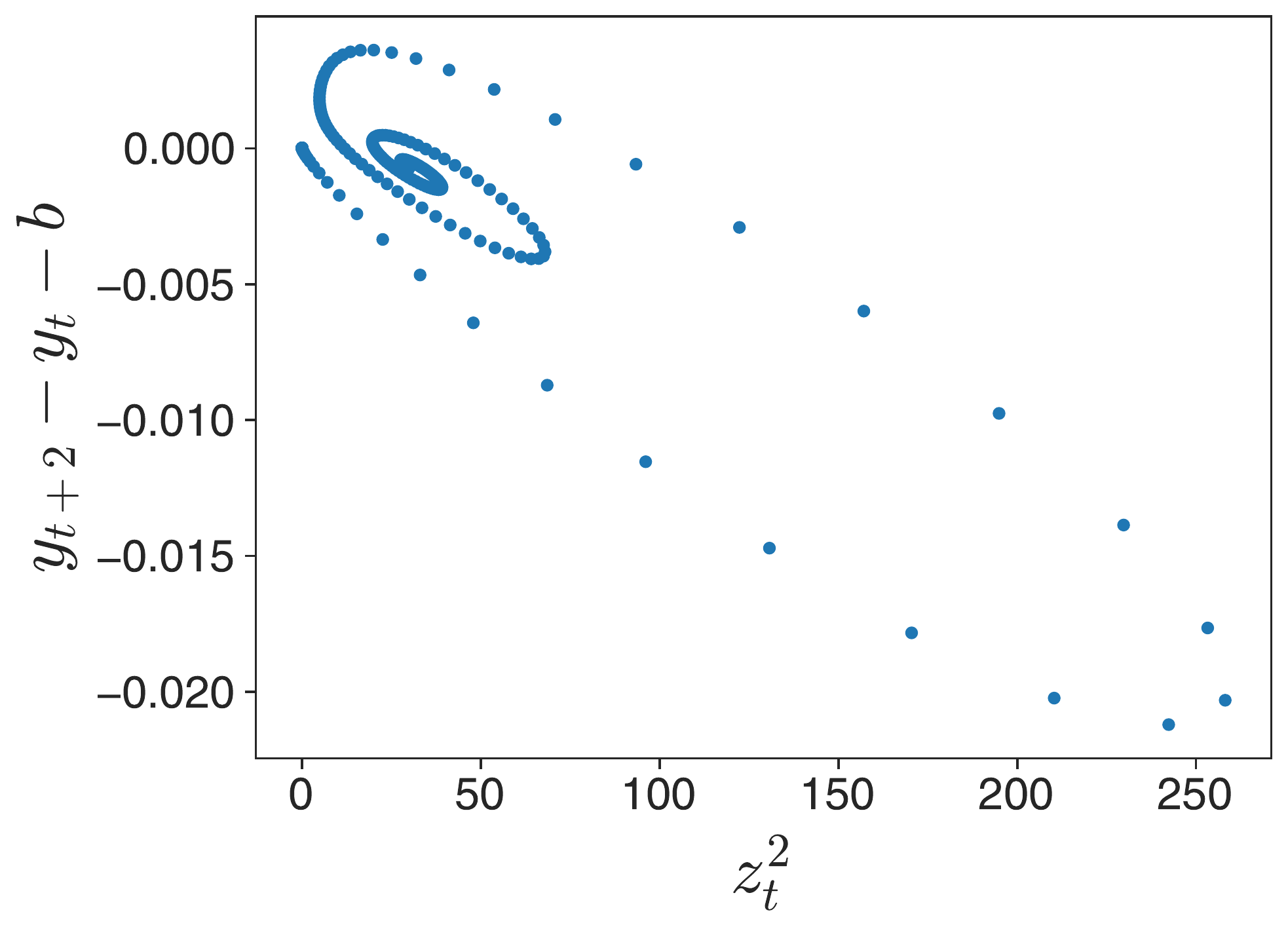} & \includegraphics[height=0.28\linewidth]{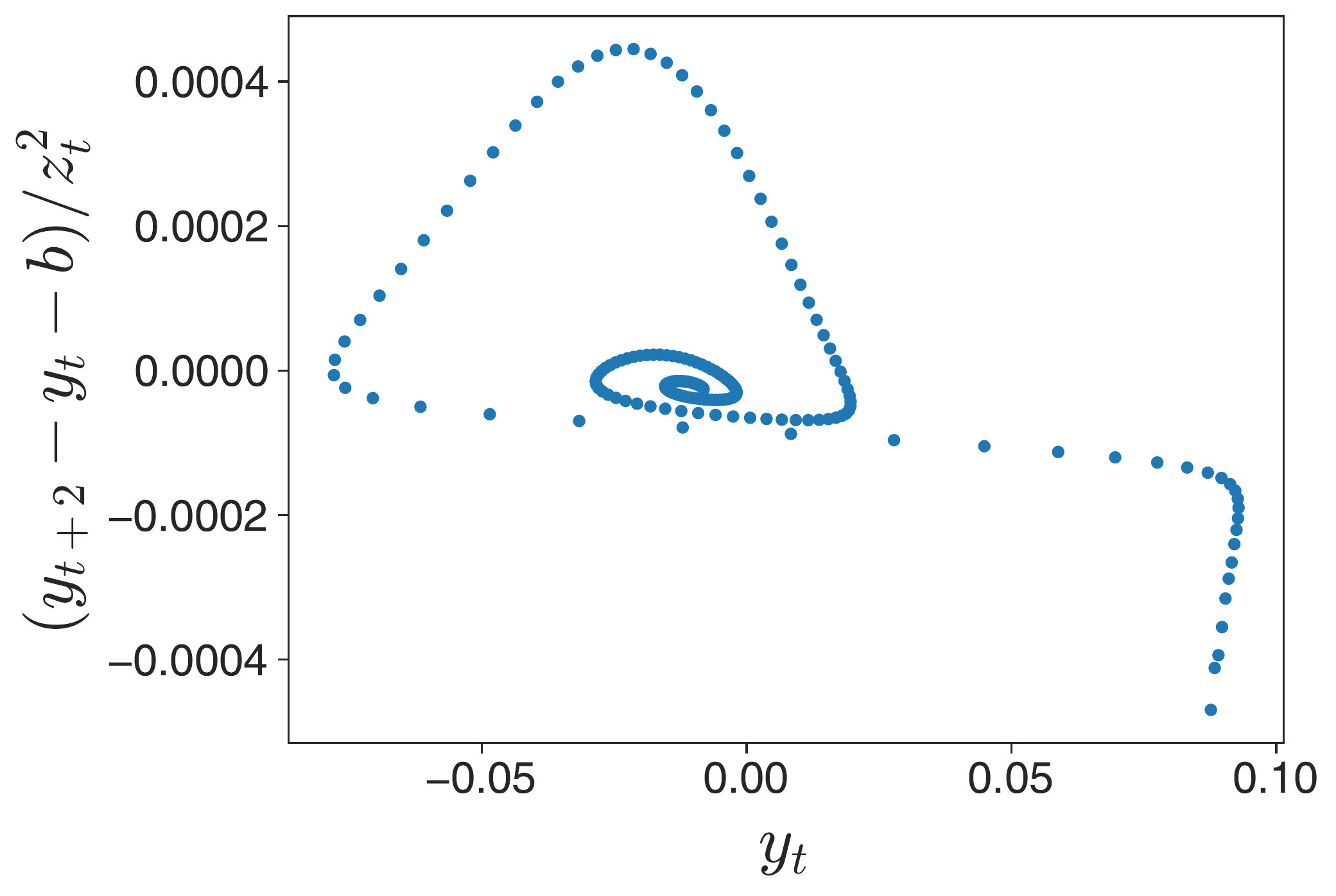}
    \end{tabular}
    \caption{Gradient flow trajectories of loss and max NTK eigenvalues for quadratic regression models 
    for varying $\sgz$. As $\sgz$ increases, $\lam_{\max}$ changes more quickly, and is
    generally increasing. Models with higher $\sgz$ converge faster in GF dynamics.}
    \label{fig:dy_approx}
\end{figure}

\subsection{Quadratic expansion of $2$-class CIFAR model}

\label{app:quad_expo}

We trained a CIFAR model using the first two classes only with $5000$ datapoints using the
Neural Tangents library \citep{novak_neural_2019} - which let us perform $2$nd order Taylor
expansions of the model at arbitrary parameters. The models were $2$-hidden layer fully-connected
networks, with hidden width $256$ and ${\rm Erf}$ non-linearities. Models were initialized with
the NTK parameterization, with weight variance $1$ and bias variance $0$.
The targets were scalar valued - $+1$ for the first class, $-1$ for the second class.
A learning rate of $0.003204$ was used in all experiments. All plots were made using
float-64 precision.

Taking a quadratic expansion at initialization, we see that the loss tracks the full
model for the first $1000$ steps in this setting (Figure \ref{fig:quad_expo_details}, left),
but misses the edge-of-stability behavior.
We use Neural Tangents to efficiently compute the NTK to get the top eigenvalue
$\lam_{1}$ (and consequently, $y$). We can also compute $\zz_{1}$ by computing the associated
eigenvector $\v_{1}$ and projecting residuals $\z$. If the quadratic expansion is taken
closer to the edge of stability, the dynamics of $\zz_{1}$ well approximates the true
$\zz_{1}$ dynamics, up to a shift associated with exponential growth of $\zz_{1}$ occurring
at different times (Figure \ref{fig:quad_expo_details}, middle).
We see that the shape of the first peak in $|\zz_{1}|$ is the same
for the full model and the quadratic model, but the subsequent oscillations are
faster and more quickly damped in the full model. This suggests that the initial EOS behavior may
be captured by the quadratic model, but the detailed dynamics require an understanding of
higher order terms. For example, the $3$rd order Taylor expansion improves the prediction
of the magnitude and period of the oscillations, but still misses key quantitative
features (Figure \ref{fig:quad_expo_details}, right).

\begin{figure}[h]
\begin{tabular}{ccc}
\includegraphics[height=0.23\linewidth]{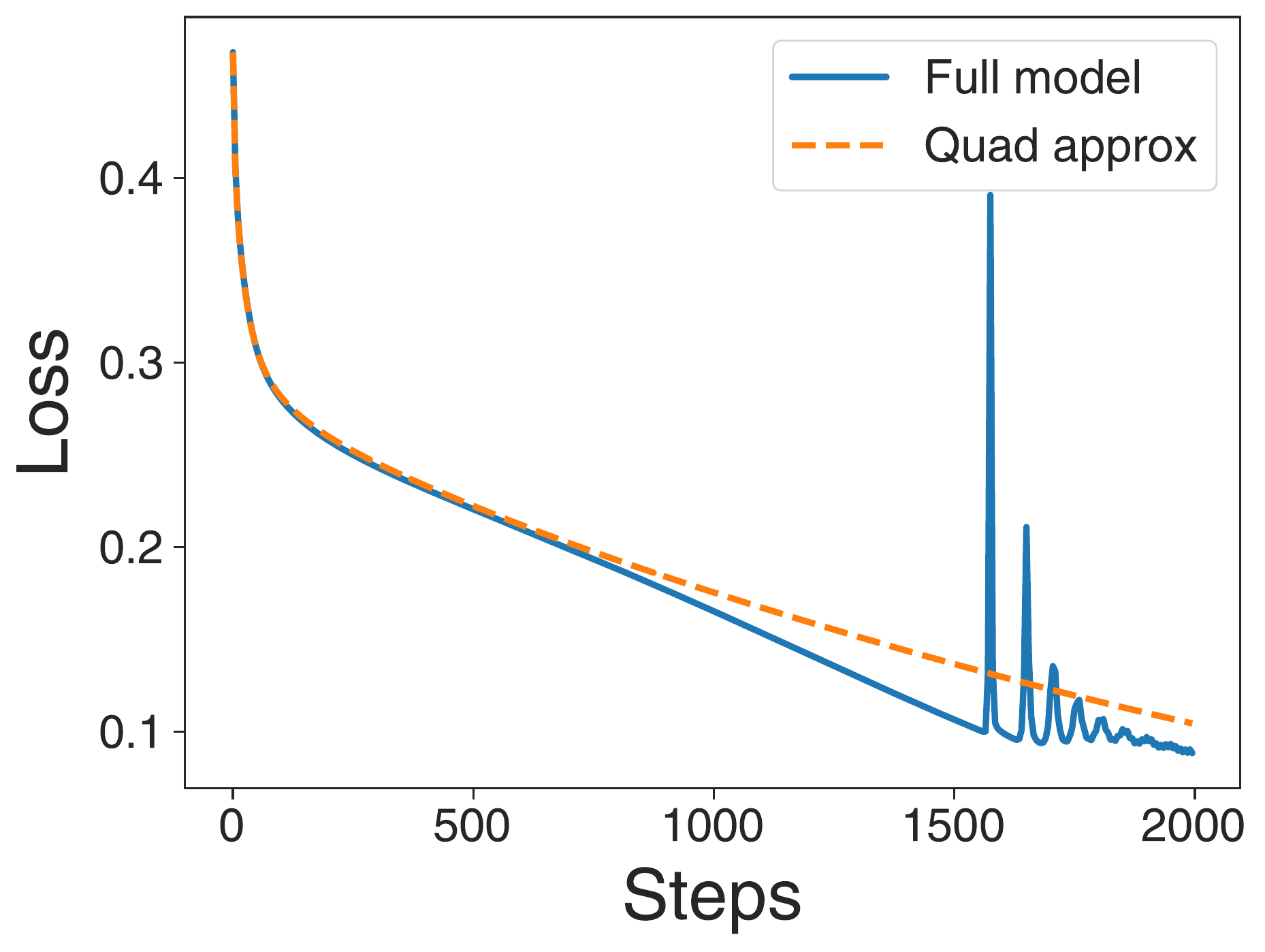} & \includegraphics[height=0.23\linewidth]{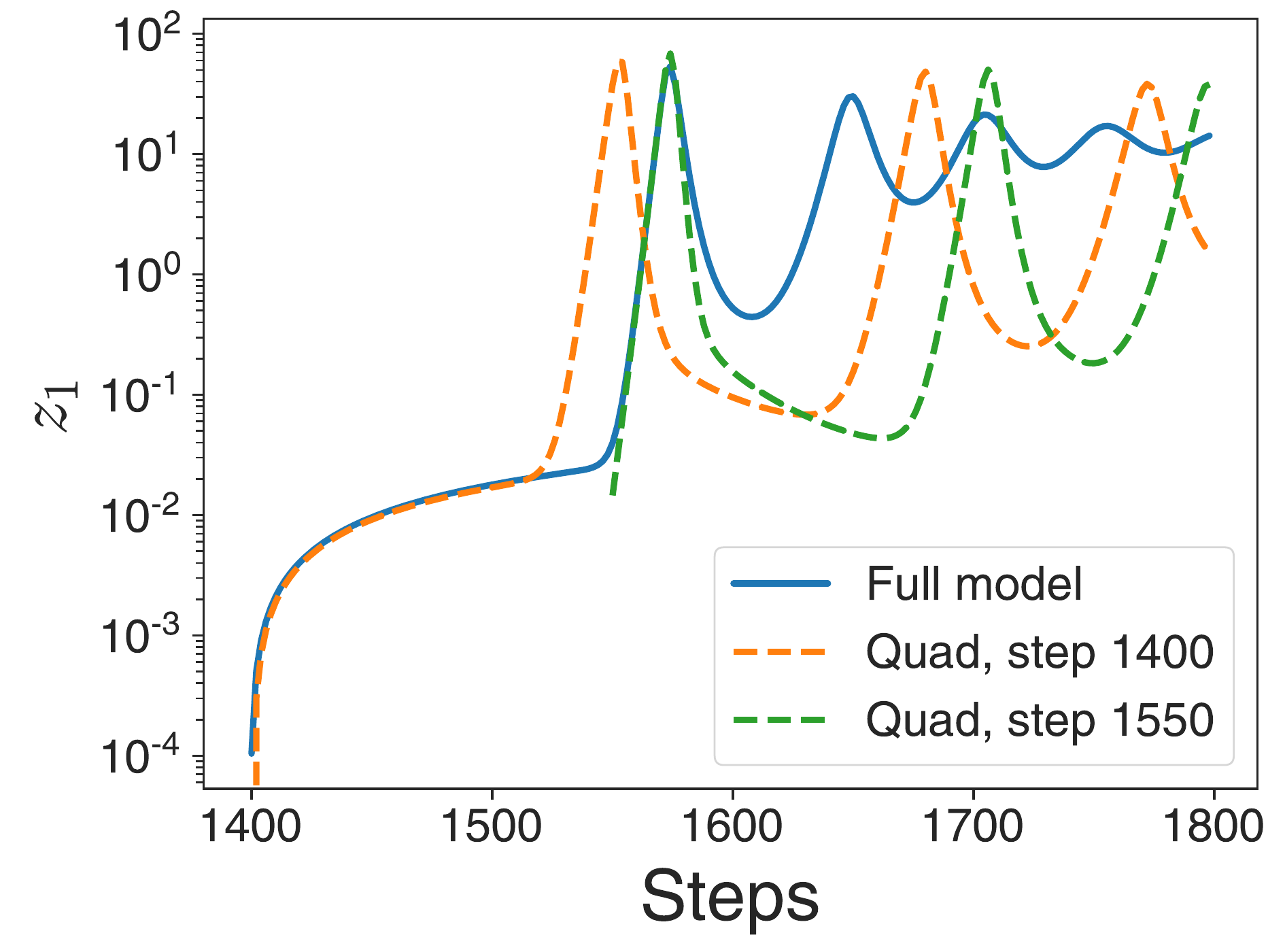} &
\includegraphics[height=0.23\linewidth]{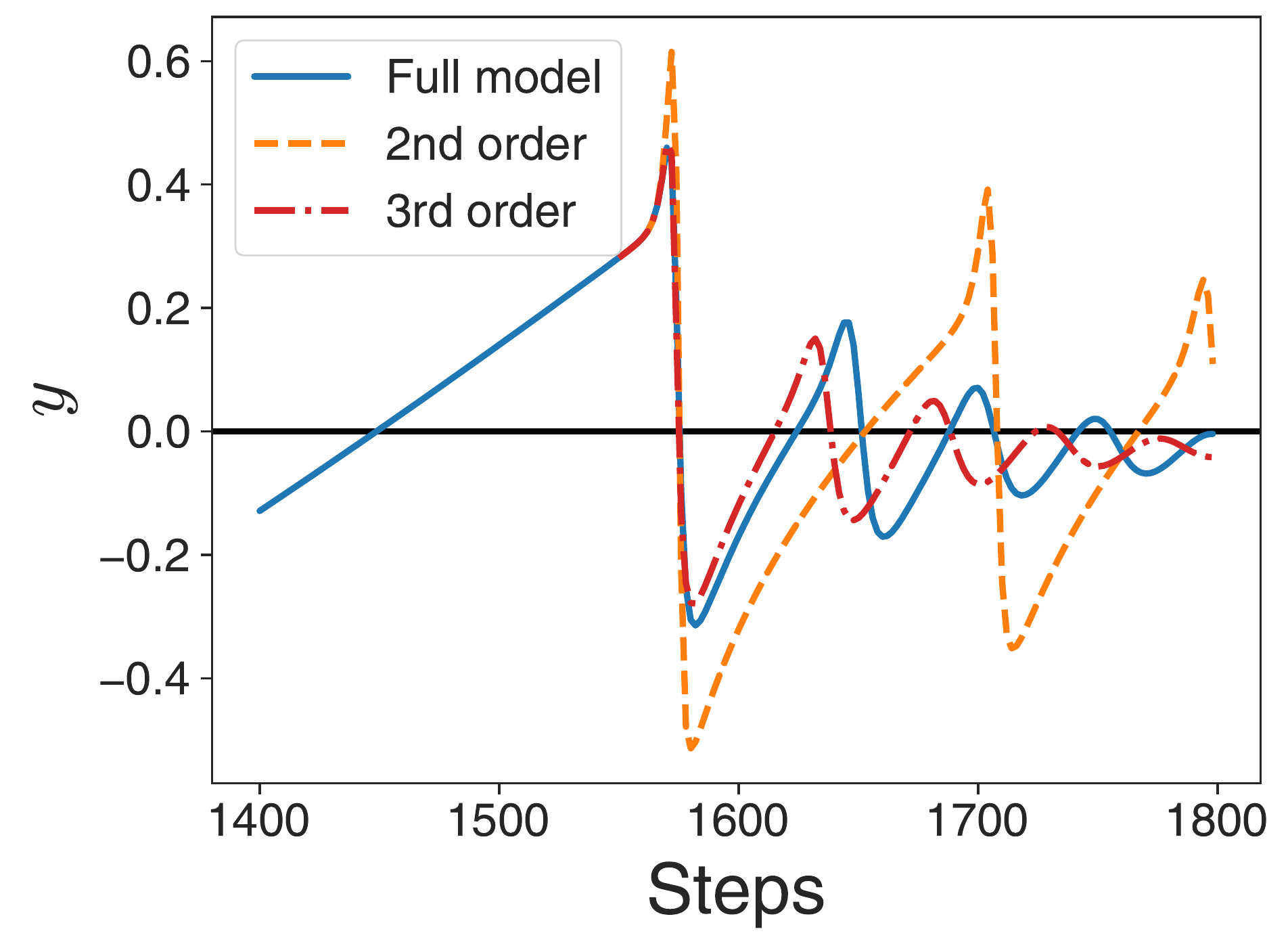}
\end{tabular}
\caption{Quadratic expansion of FCN model trained on two-class CIFAR. Expanding at initialization
gives good approximation to full model for $1000$ steps, after which EOS behavior occurs
in full model but not approximate one (left). When $\zz_{1}$ is small, quadratic model tracks
full model; however, initial exponential increase may happen earlier in approximate model (middle).
Magnitude of $\zz_{1}$ has larger oscillations in full model compared to approximate
model. Third-order Taylor expansion better captures magnitude and period of oscillations, but
still misses quantitative features (right).}
\label{fig:quad_expo_details}
\end{figure}

\end{document}